%% file: arxiv_paper.tex
\documentclass{article}
\usepackage{arxiv}

\usepackage{authblk}

\usepackage[numbers]{natbib} % has a nice set of citation styles and commands

\usepackage{booktabs}

\usepackage{tikz}
\usepackage{pgfplots}
\usetikzlibrary{arrows.meta}

\usepackage{subcaption}

\pgfplotsset{compat=1.18}
\usepgfplotslibrary{fillbetween}

\usepackage{mathtools} %amsmath with fixes and additions
\usepackage{amsthm, amssymb,amsmath}
\usepackage{thmtools, thm-restate}
\usepackage{ciphod}
\usepackage{mathbbol}
\usepackage{subcaption}
\usepackage{forest}

\usepackage{algorithm}
\usepackage{algpseudocode}
\usepackage{float}

\usepackage{xurl}
\usepackage{hyperref}

\newtheorem{theorem}{Theorem}
\newtheorem{proposition}{Proposition}
\newtheorem{lemma}{Lemma}
\newtheorem{corollary}{Corollary}
\newtheorem{example}{Example}
\newtheorem{definition}{Definition}
\newtheorem{assumption}{Assumption}

\DeclareSymbolFontAlphabet{\mathbb}{AMSb}
\DeclareSymbolFontAlphabet{\mathbbl}{bbold}

\title{Root cause analysis via difference graph discovery from linear time-series data}

\author[1]{Anouk Ruer}
\author[1]{Timothée Loranchet}
\author[1]{Daria Bystrova}
\author[1]{Charles K. Assaad}
\affil[1]{Sorbonne Université, INSERM, Institut Pierre Louis d’Epidémiologie et de Santé Publique, F75012, Paris, France}   

\date{}

\begin{document}
\maketitle

\input{main}

%\setcitestyle{numbers}
\bibliographystyle{plainnat}
\bibliography{references.bib}

\newpage
\appendix

\input{appendix}

\end{document}

%% file: main.tex
\begin{abstract}
Root cause analysis aims to identify the mechanisms responsible for anomalies in complex dynamical systems. In this paper, we study root cause analysis in linear time-series through the lens of difference graph discovery. We focus on effect-defying root causes, corresponding to variables whose causal coefficients change between a normal and an anomalous regime. We formalize this problem using linear discrete-time dynamic structural causal models and adapt several methods originally introduced for discovering difference graphs between two populations to the time-series setting, where the two populations are replaced by a normal and an anomalous regime. We first evaluate the proposed approaches on simulated data, and then demonstrate their practical relevance on real-world datasets from IT monitoring and intensive care monitoring. Our results show how difference graph discovery can help localize causal mechanisms responsible for anomalous behavior.

\keywords{difference graphs  \and root cause analysis \and time-series.}
\end{abstract}

\section{Introduction}

%\textcolor{red}{Dasha: Can you add more citations?}

Root cause analysis is a central task in many scientific and industrial domains, where the goal is to identify the mechanisms responsible for observed changes in a system. In climate science, for instance, understanding the drivers of extreme events or long-term shifts is crucial for prediction and policy-making~\cite{Camps_Valls_2025}. In IT monitoring, root cause analysis supports the diagnosis of system failures and performance degradation, enabling faster and more reliable interventions~\cite{Assaad_AISTATS_2023}. Similarly, in healthcare, identifying the causes of adverse events or changes in patient trajectories can have important clinical implications, for example in intensive care or nephrology settings~\cite{KDIGOAKI2012,connell_2015acute,mahmood2013_frequency,Assaad_HDR_2026}.

In parallel, there has been growing interest in methods that learn differences between two  environments~\cite{Wang_Neurips_2018,Chen_Neurips_2023,Malik_UAI_2024,Bystrova_Workshop_UAI_2024,Bystrova_Arxiv_2026}. These approaches, often referred to as difference graph discovery methods, aim to identify changes in dependency or causal structures across datasets. Although they have been studied in several settings, their potential for root cause analysis remains largely unexplored.

In this paper, we investigate how difference graph discovery can be leveraged for root cause analysis in a linear setting. A key limitation of existing methods is that they are mostly designed for static data, whereas many real-world systems evolve over time. We therefore extend difference graph discovery methods to the time-series setting, allowing them to capture dynamic changes in causal or dependency structures between a normal and an anomalous regime. More specifically, we extend DCI~\cite{Wang_Neurips_2018} and LDiffPC~\cite{Bystrova_Workshop_UAI_2024,Bystrova_Arxiv_2026}, and provide theoretical guarantees for their correctness in the time-series setting. %We also adapt iSCAN~\cite{Chen_Neurips_2023} and the method introduced in~\cite{Malik_UAI_2024}, which we evaluate empirically.

We highlight that these methods are naturally suited to detecting certain types of root causes, in particular effect-defying root causes, corresponding to variables whose causal coefficients change between a normal and an anomalous regime~\cite{Assaad_AISTATS_2023,Assaad_HDR_2026}. We then empirically evaluate the proposed extensions on both simulated and real-world datasets. For the real-world evaluation, we consider a dataset from IT monitoring as well as a dataset from intensive care monitoring. These experiments assess the practical relevance of difference graph discovery for root cause analysis and illustrate its potential across different application domains.

The remainder of the paper is organized as follows. Section~\ref{sec:Background} introduces the necessary preliminaries. Section~\ref{sec:RelatedWork} reviews related work. Section~\ref{sec:Main_Difference_graph_discovery} presents the time-series extensions of the considered difference graph discovery algorithms. Section~\ref{sec:Main_Root_causes} discusses how these algorithms can detect effect-defying root causes. Section~\ref{sec:Experiments} reports experiments on simulated data and apply the proposed methods to IT monitoring and intensive care monitoring data, respectively. Finally, Section~\ref{sec:Discussion} discusses the results and concludes the paper.
All proofs are provided in Appendix. The implementation of the extended time-series difference graph methods, together with the code used for the experiments is available at \url{https://github.com/CIPHOD/pyCIPHOD/tree/main/reproducibility/ecmlpkddcaesar2026}.

\section{Background}
\label{sec:Background}

We consider multivariate time-series $\mathbb{V}$ evolving according to an unknown linear discrete-time dynamic structure causal model (DT-DSCM)~\cite{Pearl_2000,Assaad_AISTATS_2023}, denoted by $\mathcal{M}$, such that
$
\forall Y_t \in \mathbb{V},  Y_t
=
\sum_{X_{t-\ell}\in \mathbb{V}}
\alpha_{{X_{t-\ell}},Y_t} X_{t-\ell}
+
\varepsilon_{Y_t},
$
where $\alpha_{{X_{t-\ell}},Y_t}\neq 0$
only if $X_{t-\ell}$ is a direct cause of $Y_t$
and $\varepsilon_{Y_t}$ is an exogenous noise term. Variables with a temporal subscript, such as $Y_t$ or $X_{t-\ell}$, denote time-indexed endogenous variables. When referring to an entire time-series, we omit the temporal subscript and write $Y$ or $X$.
The DT-DSCM governs the temporal and causal dependencies among the variables in $\mathbb{V}$. The model respects temporal ordering, meaning that a variable $X_t$ cannot cause $Y_{t-\ell}$ for any $\ell > 0$, while allowing instantaneous relations, so that $X_t$ can directly cause $Y_t$. In this paper we assume that there is no unmeasured confounding, an assumption known as causal sufficiency.

\begin{assumption}[Causal sufficiency] 
\label{assumption:sufficiency}
All noise terms are mutually independent and each noise term can affect only one observed variable.
\end{assumption}

We consider the presence of two regimes $\mathcal{N}$ and $\mathcal{\bar N}$, corresponding to two different conditions of the system, namely, $\mathcal{N}$ represents the normal regime and $\mathcal{\bar N}$ represents the anomalous regime. Throughout the paper, we assume that the observations have been assigned to either the normal or the anomalous regime before applying the difference graph discovery methods. This assumption is standard in both the difference graph discovery literature, where the generating regime is assumed to be known, and the root cause analysis literature, where anomalous observations are assumed to be identified beforehand, as is the case for all methods reviewed in Section~\ref{sec:RelatedWork}.
In the following, we use superscripts to indicate the regime associated with a given object. For example, $\mathcal{M}^{\bar{\mathcal{N}}}$, $\alpha_{X_{t-\ell},Y_t}^{\bar{\mathcal{N}}}$, and $\varepsilon_{Y_t}^{\bar{\mathcal{N}}}$ denote, respectively, the DT-DSCM in the anomalous regime, the causal coefficient from $X_{t-\ell}$ to $Y_t$ in the anomalous regime, and the exogenous noise of $Y_t$ in the anomalous regime. Within each regime the underlying DT-DSCM is assumed to be causally stationary.

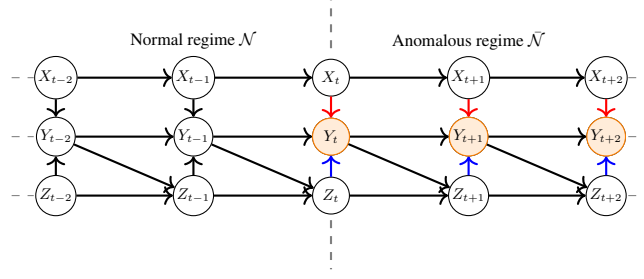
\begin{figure}[t]
\centering

\begin{tikzpicture}[
    scale=0.65,
    transform shape,
    node/.style={
        circle,
        draw,
        minimum size=0.75cm,
        inner sep=1pt,
        font=\small
    },
    edge/.style={->, thick},
    changededge/.style={->, thick, red},
    changededgeb/.style={->, thick, blue},
    dashedline/.style={dashed, gray},
    regime/.style={dashed, red, thick}
]

%------------------------
% Nodes
%------------------------

\foreach \x/\s/\lab in {
    0/m2/{t-2},
    2.8/m1/{t-1},
    5.6/zero/{t},
    8.4/p1/{t+1},
    11.2/p2/{t+2}
}{
    \node[node] (X\s) at (\x,2.4) {$X_{\lab}$};
    \node[node] (Y\s) at (\x,1.2) {$Y_{\lab}$};
    \node[node] (Z\s) at (\x,0) {$Z_{\lab}$};

    % \node at (\x,-0.9) {$\lab$};
}

%------------------------
% Autoregressive edges
%------------------------

\foreach \a/\b in {m2/m1,m1/zero,zero/p1,p1/p2}{
    \draw[edge] (X\a) -- (X\b);
    \draw[edge] (Y\a) -- (Y\b);
    \draw[edge] (Z\a) -- (Z\b);
}

%------------------------
% Cross-lagged edges
%------------------------

\foreach \a/\b in {m2/m1,m1/zero,zero/p1,p1/p2}{
    \draw[edge] (Y\a) -- (Z\b);
    %\draw[edge] (Z\a) -- (Y\b);
}

%------------------------
% Contemporaneous edges
%------------------------

\foreach \s in {m2,m1}{
    \draw[edge, left=25] (Z\s) to (Y\s);
}

\foreach \s in {m2,m1}{
    \draw[edge, left=25] (X\s) to (Y\s);
}

% changed contemporaneous edge at t
\draw[changededge, left=25] (Xzero) to (Yzero);

\foreach \s in {p1,p2}{
    \draw[changededge, left=25] (X\s) to (Y\s);
}

\draw[changededgeb, left=25] (Zzero) to (Yzero);

\foreach \s in {p1,p2}{
    \draw[changededgeb, left=25] (Z\s) to (Y\s);
}

%------------------------
% Regime change marker
%------------------------

\draw[regime, gray] (5.6,3.0) -- (5.6,4.0);
\draw[regime, gray] (5.6,-0.5) -- (5.6,-1.5);
\node[black] at (8.4,3.15) {Anomalous regime $\bar{\mathcal{N}}$};
\node[black] at (2.8,3.15) {Normal regime $\mathcal{N}$};

%------------------------
% Continuation lines
%------------------------

\foreach \y in {2.4,1.2,0}{
    \draw[dashedline] (-0.9,\y) -- (-0.45,\y);
    \draw[dashedline] (11.65,\y) -- (12.1,\y);
}
% Y_t
\node[node, draw=orange, text=black, fill=orange!15,]
(Yzero) at (5.6,1.2) {$Y_t$};

% Y_{t+1}
\node[node, draw=orange, text=black, fill=orange!15,]
(Yp1) at (8.4,1.2) {$Y_{t+1}$};

% Y_{t+2}
\node[node, draw=orange, text=black, fill=orange!15,]
(Yp2) at (11.2,1.2) {$Y_{t+2}$};

% % Z_{t+1}
% \node[node, draw=red, text=black, fill=orange!15]
% (Xp1) at (8.4,0) {$Z_{t+1}$};

% % Z_{t+2}
% \node[node, draw=red, text=black, fill=orange!15]
% (Xp2) at (11.2,0) {$Z_{t+2}$};
\end{tikzpicture}

\caption{
FT-DAG  with a regime change at time $t$ and $\ell_{max}=1$. The vertices highlighted in orange are affected by a change in its causal mechanism; in this example, the time-series $Y$ is an effect-defying root cause. %The vertices highlighted in orange correspond to downstream effects of this change. 
We consider two settings depending on how the causal mechanism of $Y$ changes: (i) Setting R, where only the red edges are changed, and (ii) Setting RB, where both the red and blue edges are changed. Note that the causal mechanism  can also be changed through a lagged mechanism, such as $Y_{t-1}\rightarrow Z_t$, where $Z$ would be the root cause.}
\label{fig:full_time_regime_change}
\end{figure}

\begin{assumption}[Causal stationarity]
\label{assumption:stationarity}
% Within each regime, for all time points \(t\), the functional form,
% the set of parents, and the structural coefficients of each variable remain
% constant over time. That is, for each variable $Y_t$,
% the causal coefficients $\alpha_{{X_{t-\ell}},Y_t}$ do not depend on $t$ and the noise $\varepsilon_{Y_t}$ is identically distributed over time within the same regime.
Within each regime, for each variable $Y_t$,
the coefficients $\alpha_{{X_{t-\ell}},Y_t}$ do not depend on $t$ and the noise $\varepsilon_{Y_t}$ is identically distributed over time.
\end{assumption}

In addition, we assume the existence of a finite maximal lag between causes and their effects in the system, denoted by $\ell_{\max}$, which is much smaller than the total number of observed time points. In practice, we assume that this maximal lag is provided by domain knowledge. More generally, the domain expert may specify an upper bound on the true maximal lag, so that the chosen $\ell_{\max}$ is greater than or equal to the actual maximal lag of the system.

Within each regime, the DT-DSCM induces a full-time directed acyclic graph (FT-DAG), denoted by $\mathcal{G}^{\mathcal{N}} = (\mathbb{V}, \mathbb{E}^{\mathcal{N}})$ and $\mathcal{G}^{\bar{\mathcal{N}}} = (\mathbb{V}, \mathbb{E}^{\bar{\mathcal{N}}})$, where vertices correspond to time-indexed variables and edges encode direct causal relationships both within a time step and across successive time steps. In this paper, we use several standard graphical notions. For instance, a vertex $X_{t-\ell}$ is a parent of $Y_t$ in $\mathcal{G}$ if $X_{t-\ell}\to Y_t$ in $\mathcal{G}$, equivalently if $\alpha_{X_{t-\ell},Y_t}\neq 0$. The set of parents of $Y_t$ is denoted by $\mathrm{Pa}(Y_t,\mathcal{G})$. Two vertices are adjacent if they are connected by an edge in $\mathcal{G}$. The skeleton of $\mathcal{G}$, denoted by $\mathrm{Skel}(\mathcal{G})$, is the undirected graph obtained from $\mathcal{G}$ by replacing every directed edge with an undirected edge while preserving the same set of adjacencies.

For illustration, Figure~\ref{fig:full_time_regime_change} presents FT-DAGs corresponding to a normal and an anomalous regime. %that share the same topological ordering. 
In the normal regime, the system is governed by a fixed stationary DT-DSCM, whereas in the anomalous regime it is governed by another stationary DT-DSCM in which changes, relative to the normal regime, occur in the direct effects on a specific variable $Y_t$.
Two settings are considered in this figure. In the first one, referred to as \emph{Setting R}, only the red edge from $X_t$ to $Y_t$ is modified across regimes. In the second one, referred to as \emph{Setting RB}, both the red and blue edges, from $X_t$ to $Y_t$ and from $Z_t$ to $Y_t$, are modified.
These settings will serve as running examples throughout the paper to clarify the behavior of the proposed methods.

Under Assumption~\ref{assumption:stationarity}, the FT-DAG of each regime is fully characterized by the relationships between variables within a window of size $\ell_{\max}+1$, as it repeats this structure throughout the regime. Henceforth, $\mathbb{V}$ refers solely to variables within this window; see Appendix %~\ref{appendix:window} 
for more details.

The main objective of this paper is to identify root causes. To relate observed anomalies to changes in the  causal system, we make the following assumption.
\begin{assumption}[Anomalies]
\label{assumption:anomalies}
All anomalies observed in the anomalous regime are induced by changes in the underlying DT-DSCM between the normal regime $\mathcal{N}$ and the anomalous regime $\bar{\mathcal{N}}$.
\end{assumption}
Under the above assumption, root causes are defined as variables whose changes trigger and propagate anomalies throughout the system.
At the finest level, such changes occur at the micro, or time-indexed, level of the system, for instance through a change in the mechanism of a variable $Y_t$. However, under stationarity, a change affecting $Y_t$ at the beginning of the anomalous regime may propagate throughout the rest of the regime. Therefore, for simplicity and practical interpretability, we define root causes at the macro, or time-series, level. That is, we say that the time-series $Y$ is a root cause, rather than referring to a specific time point $Y_t$. This is also consistent with many applications, where domain experts are primarily interested in identifying which process is causing the malfunctioning, rather than the exact time at which the malfunctioning first appeared.

In DT-DSCMs, two types of root causes are commonly considered~\cite{Assaad_HDR_2026}. A macro vertex $Y$ is called an \emph{effect-defying root cause} if there exist $Y_t\in Y$ and $X_{t-\ell}\in \mathrm{Pa}(Y_t, \mathcal{G}^{\mathcal{N}})$ such that the corresponding causal coefficient changes across regimes, namely $ \alpha^{\mathcal{N}}_{X_{t-\ell},Y_t} \neq \alpha^{\bar{\mathcal{N}}}_{X_{t-\ell},Y_t}. $ Similarly, $Y$ is called a \emph{noise-defying root cause} if there exists $Y_t\in Y$ such that the noise distribution of $Y_t$ differs between regimes, that is, $ \varepsilon^{\mathcal{N}}_{Y_t} \neq_d \varepsilon^{\bar{\mathcal{N}}}_{Y_t}$.

In this paper, we mainly focus on effect-defying root causes. Difference graphs provide a natural tool for this task, as they aim to identify direct relationships that differ between variables across two distributions or environments. While difference graph methods have received increasing attention in the independent and identical distributed (i.i.d.) setting, mostly to characterize differences between populations~\cite{Wang_Neurips_2018,Chen_Neurips_2023,Malik_UAI_2024,Bystrova_Workshop_UAI_2024,Assaad_CLEAR_2025} rather than for root cause analysis, their role in time-series root cause analysis remains largely unexplored. Here, we investigate how such methods can be used to detect effect-defying root causes in linear dynamical systems by comparing normal and anomalous regimes. In the following, we introduce difference graphs adapted to time-series and linear DT-DSCMs.

\begin{definition}[Full-time Difference Graph] Consider two DT-DSCMs $\mathcal{M}^{\mathcal{N}}$ and $\mathcal{M}^{\mathcal{\bar N}}$.
    A Full-Time Difference Graph (FT-DFG) $\mathcal{D}=(\mathbb{V}, \mathbb{E}^{|\mathcal{N}-\mathcal{\bar N}|})$ between $\mathcal{M}^{\mathcal{N}}$ and $\mathcal{M}^{\mathcal{\bar N}}$, is a directed graph such that the set of vertices is identical to the set of endogeneous variables in $\mathcal{M}^{\mathcal{N}}$, $\mathcal{M}^{\mathcal{\bar N}}$ and the set of edges is defined as:
     $\mathbb{E}^{|\mathcal{N}-\mathcal{\bar N}|} = \{ X_{t-\ell} \rightarrow Y_t |  \forall (X_{t-\ell}, Y_t)  \in \mathbb{V}^2  
    \text{ such that } 
    \alpha^{\mathcal{N}}_{X_{t-\ell},Y_t} \neq \alpha^{\bar{\mathcal{N}}}_{X_{t-\ell},Y_t} \}.$
    % \begin{align*}    
    % \mathbb{E}^{|\mathcal{N}-\mathcal{\bar N}|} = \{ X_{t-\ell} \rightarrow Y_t |  &\forall (X_{t-\ell}, Y_t)  \in \mathbb{V}^2  
    % \text{ such that } 
    % \alpha^{\mathcal{N}}_{X_{t-\ell},Y_t} \neq \alpha^{\bar{\mathcal{N}}}_{X_{t-\ell},Y_t} \}.
    % \end{align*}
\end{definition}

In general, FT-DFGs may contain cycles~\cite{Assaad_CLEAR_2025}. However, to simplify the problem, we assume in this paper that the FT-DFG is acyclic. This assumption is commonly adopted by algorithms~\cite{Wang_Neurips_2018,Chen_Neurips_2023,Malik_UAI_2024,Bystrova_Workshop_UAI_2024} that aim to fully or partially recover FT-DFGs from data.

\begin{figure}[t]
\begin{subfigure}{0.33\textwidth}
\centering
\begin{tikzpicture}[
    scale=0.65,
    transform shape,
    node/.style={
        circle,
        draw,
        minimum size=0.75cm,
        inner sep=1pt,
        font=\small
    },
    edge/.style={->, thick},
    changededge/.style={->, thick, red},
    dashedline/.style={dashed, gray},
    regime/.style={dashed, red, thick}
]

%------------------------
% Nodes
%------------------------

\foreach \x/\s/\lab in {
    % 0/m2/{t-2},
    % 2.8/m1/{t-1},
    % 5.6/zero/{t-2},
    8.4/p1/{t-1},
    11.2/p2/{t}
}{
    \node[node] (X\s) at (\x,2.4) {$X_{\lab}$};
    \node[node] (Y\s) at (\x,1.2) {$Y_{\lab}$};
    \node[node] (Z\s) at (\x,0) {$Z_{\lab}$};

    % \node at (\x,-0.9) {$\lab$};
}

% changed contemporaneous edge at t
% \draw[edge, left=25] (Xzero) to (Yzero);

\foreach \s in {p1,p2}{
    \draw[edge, left=25] (X\s) to (Y\s);
}
\end{tikzpicture}
\end{subfigure}
\hfill 
\vrule
\hfill 
\begin{subfigure}{0.33\textwidth}
\begin{tikzpicture}[
    scale=0.65,
    transform shape,
    node/.style={
        circle,
        draw,
        minimum size=0.75cm,
        inner sep=1pt,
        font=\small
    },
    edge/.style={->, thick},
    changededge/.style={->, thick, red},
    dashedline/.style={dashed, gray},
    regime/.style={dashed, red, thick}
]

%------------------------
% Nodes
%------------------------

\foreach \x/\s/\lab in {
    % 0/m2/{t-2},
    % 2.8/m1/{t-1},
    % 5.6/zero/{t-2},
    8.4/p1/{t-1},
    11.2/p2/{t}
}{
    \node[node] (X\s) at (\x,2.4) {$X_{\lab}$};
    \node[node] (Y\s) at (\x,1.2) {$Y_{\lab}$};
    \node[node] (Z\s) at (\x,0) {$Z_{\lab}$};

    % \node at (\x,-0.9) {$\lab$};
}

% changed contemporaneous edge at t
% \draw[edge, left=25] (Xzero) to (Yzero);

\foreach \s in {p1,p2}{
    \draw[edge, left=25] (X\s) to (Y\s);
}

% \draw[edge, left=25] (Zzero) to (Yzero);

\foreach \s in {p1,p2}{
    \draw[edge, left=25] (Z\s) to (Y\s);
}

\end{tikzpicture}
\end{subfigure}
\caption{
FT-DFGs corresponding to the FT-DAG in Figure~\ref{fig:full_time_regime_change}. 
The left panel shows the FT-DFG for Setting R, where only the red edges in Figure~\ref{fig:full_time_regime_change} are changed. 
The right panel shows the FT-DFG for Setting RB, where both the red and blue edges in Figure~\ref{fig:full_time_regime_change} are changed. 
}
\label{fig:difference_graph_example}

\end{figure}
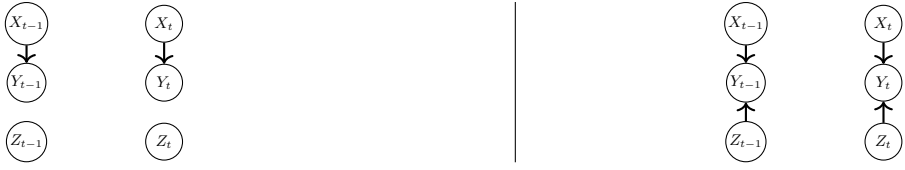

\begin{assumption}[Common topological ordering]
\label{assumption:order}
The graphs $\mathcal{G}^{\mathcal{N}}$ and $\mathcal{G}^{\bar{\mathcal{N}}}$
share a common topological ordering. That is, there exists a mapping
$\pi:\mathbb{V}\to\{1,\dots,|\mathbb{V}|\}$ such that $\pi$ is a topological
ordering of both $\mathcal{G}^{\mathcal{N}}$ and
$\mathcal{G}^{\bar{\mathcal{N}}}$, \ie, for every directed edge $X_{t}\to Y_t$ in
either graph, $\pi(X_{t})<\pi(Y_t)$.
\end{assumption}

% \textcolor{red}{
% In general it cannot be fully identified from observational data alone. Different temporal causal strucutures may induce the same set of conditional independencies and statistical invariances across environments. Consequently, only a partially identifiable representation of the summary causal difference graph can be recovered. 
% [under conditions etc]
% The output of temporal difference causal discovery algorithms is therefore a partially directed summary causal difference graph, containing both directed and undirected edges. 
% }

Figure~\ref{fig:difference_graph_example} illustrates the FT-DFGs associated with the two settings, R and RB, from Figure~\ref{fig:full_time_regime_change}.

%The corresponding FT-DFGs are shown in Figure~\ref{fig:difference_graph_example}. In Setting R, the difference graph contains a single changed edge, whereas in Setting RB it contains two changed edges targeting the same variable.

\section{Related Work}
\label{sec:RelatedWork}

Existing root-cause analysis methods either rely on a predefined causal graph~\cite{budhathoki2021,Li_2022,Assaad_AISTATS_2023,Assaad_HDR_2026} or first estimate causal relations from data~\cite{Spirtes_2000,Runge_2019_Nature,Runge_2020_UAI,Assaad_2022,Reiter_2026}. Among the latter, MicroCause~\cite{Meng_2020} combines PCMCI~\cite{Runge_2019_Nature} with a random-walk procedure, RCD~\cite{Ikram_2022} uses a regime indicator, and T-RCA~\cite{Zan_CIKM_2024} combines a learned graph with anomaly timing. There exists also  approaches that focus on explaining one-point anomalies or outliers~\cite{budhathoki2022outliers,orchard2026,schkoda2026} bur we consider those to be beyond the scope of this paper. Here we focus on collective anomalies and assume having access to  samples from normal and anomalous regimes. 

Our work builds on difference-graph discovery, which directly estimates how causal mechanisms vary across environments. In particular,  we  focus on methods that use equality tests to infer the FT-DFG.

\begin{itemize}
    \item  LDiffPC~\cite{Bystrova_Workshop_UAI_2024,Bystrova_Arxiv_2026} is an algorithm for discovering a partially oriented graph in linear non-dynamic structural causal models. It recovers the skeleton of the graph by testing equality of regression coefficients across environments and then orients edges using collider detection and Meek \cite{Meek1995CausalIA} propagation rules.

    \item DCI~\cite{Wang_Neurips_2018} is also designed to discover a partially oriented graph in linear  non-dynamic structural causal models. It recovers the skeleton using the same procedure as LDiffPC. However, its orientation step differs: instead of relying on separation sets and orientation rules, DCI orients edges by exploiting changes in residual variances across environments.
\end{itemize}

There exists also other approaches that do not use equality tests. For instance, 
MBGH~\cite{Malik_UAI_2024} relies on covariance differences, while iSCAN~\cite{Chen_Neurips_2023} identifies changed mechanisms in nonlinear additive-noise models through distributional invariance.
More details on related works is provided in Appendix. %~\ref{appendix:RelatedWork}.

%LDiffPC~\cite{Bystrova_Workshop_UAI_2024,Bystrova_Arxiv_2026} and DCI~\cite{Wang_Neurips_2018} use equality tests on regression coefficients in linear non-dynamic models, but differ in their orientation procedures: LDiffPC uses collider detection and Meek’s rules~\cite{Meek1995CausalIA}, whereas DCI exploits changes in residual variances. MBGH~\cite{Malik_UAI_2024} relies on covariance differences, while iSCAN~\cite{Chen_Neurips_2023} identifies changed mechanisms in nonlinear additive-noise models through distributional invariance. A detailed comparison is provided in Appendix~\ref{appendix:RelatedWork}.

\section{Difference graph discovery: from static to time-series}
\label{sec:Main_Difference_graph_discovery}

In this section, we adapt existing difference graph discovery methods to the time-series setting. The methods considered here do not return a fully oriented FT-DFG, but rather a partially oriented graph. We denote the output of such methods by $\widehat{\mathcal{D}}=(\mathbb{V},\widehat{\mathbb{E}},\widehat{\mathbb{E}}^{-})$, where $\mathbb{V}$ is the set of vertices, $\widehat{\mathbb{E}}$ is the set of directed edges, and $\widehat{\mathbb{E}}^{-}$ is the set of undirected edges.

In the following, $\sigma_{Y_t}$ denotes the standard deviation of the noise term $\varepsilon_{Y_t}$. 
For any conditioning set $\mathbb{S}$, 
$\beta^{\mathcal{N}}_{X_{t-\ell},Y_t \mid \mathbb{S}}$ denotes the coefficient of $X_{t-\ell}$, and 
$\sigma^{\mathcal{N}}_{X_{t-\ell},Y_t \mid \mathbb{S}}$ denotes the residual standard deviation, in the regression of $Y_t$ on $\{X_{t-\ell}\}\cup\mathbb{S}$ in the normal regime. 
Similarly, 
$\beta^{\bar{\mathcal{N}}}_{X_{t-\ell},Y_t \mid \mathbb{S}}$ and 
$\sigma^{\bar{\mathcal{N}}}_{X_{t-\ell},Y_t \mid \mathbb{S}}$ denote the corresponding coefficient and residual standard deviation in the anomalous regime.
 
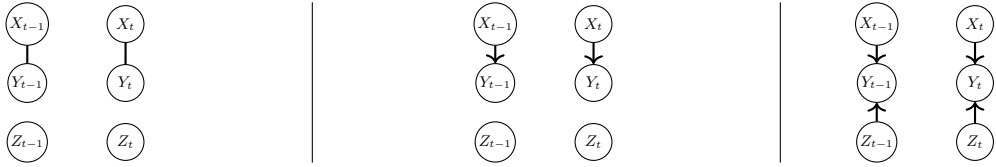
\begin{figure}[t]
\begin{subfigure}{0.25\textwidth}
\centering
\begin{tikzpicture}[
    scale=0.65,
    transform shape,
    node/.style={
        circle,
        draw,
        minimum size=0.75cm,
        inner sep=1pt,
        font=\small
    },
    edge/.style={-, thick},
    changededge/.style={->, thick, red},
    dashedline/.style={dashed, gray},
    regime/.style={dashed, red, thick}
]

%------------------------
% Nodes
%------------------------

\foreach \x/\s/\lab in {
    % 0/m2/{t-2},
    % 2.8/m1/{t-1},
    % 5./zero/{t-2},
    7./p1/{t-1},
    9./p2/{t}
}{
    \node[node] (X\s) at (\x,2.4) {$X_{\lab}$};
    \node[node] (Y\s) at (\x,1.2) {$Y_{\lab}$};
    \node[node] (Z\s) at (\x,0) {$Z_{\lab}$};

    % \node at (\x,-0.9) {$\lab$};
}

% changed contemporaneous edge at t
% \draw[edge, left=25] (Xzero) to (Yzero);

\foreach \s in {p1,p2}{
    \draw[edge, left=25] (X\s) to (Y\s);
}
\end{tikzpicture}
\end{subfigure}
\hfill 
\vrule
\hfill
\begin{subfigure}{0.25\textwidth}
\centering
\begin{tikzpicture}[
    scale=0.65,
    transform shape,
    node/.style={
        circle,
        draw,
        minimum size=0.75cm,
        inner sep=1pt,
        font=\small
    },
    edge/.style={->, thick},
    changededge/.style={->, thick, red},
    dashedline/.style={dashed, gray},
    regime/.style={dashed, red, thick}
]

%------------------------
% Nodes
%------------------------

\foreach \x/\s/\lab in {
    % 0/m2/{t-2},
    % 2.8/m1/{t-1},
    % 5./zero/{t-2},
    7./p1/{t-1},
    9./p2/{t}
}{
    \node[node] (X\s) at (\x,2.4) {$X_{\lab}$};
    \node[node] (Y\s) at (\x,1.2) {$Y_{\lab}$};
    \node[node] (Z\s) at (\x,0) {$Z_{\lab}$};

    % \node at (\x,-0.9) {$\lab$};
}

% changed contemporaneous edge at t
% \draw[edge, left=25] (Xzero) to (Yzero);

\foreach \s in {p1,p2}{
    \draw[edge, left=25] (X\s) to (Y\s);
}
\end{tikzpicture}
\end{subfigure}
\hfill 
\vrule
\hfill
\begin{subfigure}{0.25\textwidth}
\begin{tikzpicture}[
    scale=0.65,
    transform shape,
    node/.style={
        circle,
        draw,
        minimum size=0.75cm,
        inner sep=1pt,
        font=\small
    },
    edge/.style={->, thick},
    changededge/.style={->, thick, red},
    dashedline/.style={dashed, gray},
    regime/.style={dashed, red, thick}
]

%------------------------
% Nodes
%------------------------

\foreach \x/\s/\lab in {
    % 0/m2/{t-2},
    % 2.8/m1/{t-1},
    % 5./zero/{t-2},
    7./p1/{t-1},
    9./p2/{t}
}{
    \node[node] (X\s) at (\x,2.4) {$X_{\lab}$};
    \node[node] (Y\s) at (\x,1.2) {$Y_{\lab}$};
    \node[node] (Z\s) at (\x,0) {$Z_{\lab}$};

    % \node at (\x,-0.9) {$\lab$};
}

% changed contemporaneous edge at t
% \draw[edge, left=25] (Xzero) to (Yzero);

\foreach \s in {p1,p2}{
    \draw[edge, left=25] (X\s) to (Y\s);
}

% \draw[edge, left=25] (Zzero) to (Yzero);

\foreach \s in {p1,p2}{
    \draw[edge, left=25] (Z\s) to (Y\s);
}

\end{tikzpicture}
\end{subfigure}
\caption{
FT-DFGs recovered by the considered algorithms. The left panel shows the output of tsLDiffPC in Setting R. The middle panel shows the output of tsLDiffPC$^2$ and tsDCIPC in Setting R. Under the variance-invariance condition required by DCI, tsDCI recovers the same graph. The right panel shows the output of tsLDiffPC, tsLDiffPC$^2$, and tsDCIPC in Setting RB; under the same condition, tsDCI recovers an identical graph.}
\label{fig:partially_oriented_difference_graph_example}

\end{figure}

\subsection{The tsLDiffPC algorithm}

We begin with the extension of LDiffPC, which we denote by tsLDiffPC.
tsLDiffPC proceeds in two steps. First, it estimates the FT-DFG skeleton by testing equality of regression coefficients across regimes. %By stationarity, all changed mechanisms can be captured by considering only candidate edges involving at least one present-time variable. 
More specifically, tsLDiffPC starts from a complete undirected graph over $\mathbb{V}$ and iteratively removes edges between pairs of vertices $X_{t-\ell}$ and $Y_t$ whenever it finds a conditioning set $\mathbb{S}\subseteq\mathbb{V}\setminus\{X_{t-\ell},Y_t\}$ such that $\beta_{X_{t-\ell}, Y_t\mid \mathbb{S}}^{\mathcal{N}}=\beta_{X_{t-\ell}, Y_t\mid \mathbb{S}}^{\mathcal{\bar N}}$. The size of the conditioning set is increased progressively, following the PC strategy~\cite{Spirtes_2000}. For each removed edge, the corresponding separating set $\mathbb{S}$ is stored. 
Second, tsLDiffPC starts the orientation phase by orienting lagged edges according to the temporal order. It then orients instantaneous edges using the standard LDiffPC collider-detection rule: for any  triple $X_{t-\ell} - Z_t - Y_t$ such that $X_{t-\ell}$ and $Y_t$ are not adjacent, if $Z_t$ does not belong to the separating set stored for $X_{t-\ell}$ and $Y_t$, the triple is oriented as
$X_{t-\ell} \to Z_t \leftarrow Y_t$.
Finally, tsLDiffPC further orients edges by iteratively applying Meek's orientation rules until no additional orientation is possible. The pseudocode is given in Appendix. %\ref{appendix:pseudocode}.

To ensure correct detection, tsDiffPC relies on the following two assumptions.
\begin{assumption}[Diff-adjacency-faithfulness]
\label{assumption:diff_adjacency_faithfulness}
Let $\mathcal{D}=(\mathbb{V}, \mathbb{E})$ denote an FT-DFG. If $X_{t-\ell}\rightarrow Y_t$ in $\mathbb{E}$ then for any subset $\mathbb{S}\subseteq \mathbb{V}\backslash\{X_{t-\ell}, Y_t\}$, $
\beta^{\mathcal{N}}_{X_{t- \ell},Y_t\mid \mathbb{S}}
\neq
\beta^{\bar{\mathcal{N}}}_{X_{t- \ell},Y_t\mid \mathbb{S}}$.
\end{assumption}

\begin{assumption}[Diff-orientation-faithfulness]
\label{assumption:diff_orient_faithfulness}
Let $\mathcal{D}=(\mathbb{V}, \mathbb{E})$ denote an FT-DFG. Let $(X_{t-\ell}, Z_t, Y_t)$ be a triplet of vertices such that $X_{t-\ell}$ and $Z_t$ are adjacent, $Z_{t}$ and $Y_t$ are adjacent, and $X_{t-\ell}$ and $Y_t$ are not adjacent. The following holds:
\begin{itemize}
    \item If $X_{t-\ell}\rightarrow Z_t\leftarrow Y_t$ in $\mathcal{D}$ then for any subset $\mathbb{S}\subseteq \mathbb{V}\backslash\{X_{t-\ell}, Y_t\}$ that contains $Z_t$, $
\beta^{\mathcal{N}}_{X_{t- \ell},Y_t\mid \mathbb{S}}
\neq
\beta^{\bar{\mathcal{N}}}_{X_{t- \ell},Y_t\mid \mathbb{S}}$.
    \item Otherwise,   for any subset $\mathbb{S}\subseteq \mathbb{V}\backslash\{X_{t-\ell}, Y_t\}$ that does not contain $Z_t$, $
\beta^{\mathcal{N}}_{X_{t- \ell},Y_t\mid \mathbb{S}}
\neq
\beta^{\bar{\mathcal{N}}}_{X_{t- \ell},Y_t\mid \mathbb{S}}$.
\end{itemize}
\end{assumption}

The two assumptions introduced above differ from those originally used in the work introducing LDiffPC. In~\cite{Bystrova_Arxiv_2026}, the analysis relies on the stronger assumption of diff-faithfulness, which establishes an exact correspondence between the equality of regression coefficients and diff-separation (a tool similar to d-separation~\cite{Pearl_2000} but adapted to difference graphs). The analysis also requires an additional assumption ensuring the symmetry of regression invariance. It was shown that diff-adjacency-faithfulness and diff-orientation-faithfulness are jointly implied by diff-faithfulness and the symmetry of regression invariance~\cite{Bystrova_Arxiv_2026}.
%In other words, diff-faithfulness is a strictly stronger condition that guarantees the validity of these two assumptions. 
Here, we argue that diff-adjacency-faithfulness and diff-orientation-faithfulness are sufficient on their own to establish the correctness of the tsLDiffPC algorithm, as well as its extension to the time series setting. Further details on how Assumptions~\ref{assumption:diff_adjacency_faithfulness} and \ref{assumption:diff_orient_faithfulness} relate to the assumptions used in~\cite{Bystrova_Arxiv_2026} are provided in Appendix. %~\ref{appendix:diff-faithfulnes}.

Theorem~\ref{theorem:correctness_tsLDiffPC} establishes the correctness of tsLDiffPC under diff-adjacency-faithfulness and the diff-orientation-faithfulness introduced above.

\begin{theorem}
\label{theorem:correctness_tsLDiffPC}
Let $\mathcal{D}=(\mathbb{V}, \mathbb{E})$ be the true FT-DFG and $\widehat{\mathcal{D}}=(\mathbb{V}, \widehat{\mathbb{E}}, \widehat{\mathbb{E}}^{-})$ be the output of tsLDiffPC.
 If Assumptions~\ref{assumption:sufficiency},\ref{assumption:stationarity},\ref{assumption:anomalies},\ref{assumption:order}, \ref{assumption:diff_adjacency_faithfulness} and \ref{assumption:diff_orient_faithfulness}  are satisfied  and we are given perfect conditional equality information about all pairs of variables  then $\mathrm{Skel}(\widehat{\mathcal{D}})  = \mathrm{Skel}(\mathcal{D})$  and $\widehat{\mathbb{E}}\subseteq\mathbb{E}$.
\end{theorem}

\begin{example}
In Setting R, only a single edge entering $Y_t$ is modified across regimes. Consequently, the corresponding FT-DFG contains only one changed edge, $X_t \rightarrow Y_t$, as shown in Figure~\ref{fig:difference_graph_example} (left). While LDiffPC correctly identifies the adjacency, it lacks sufficient information to orient the edge, as neither temporal ordering, collider detection, nor Meek’s orientation rules apply. As a result, the edge remains unoriented in the output, shown in Figure~\ref{fig:partially_oriented_difference_graph_example} (left).
In contrast, in Setting RB, two incoming edges into $Y_t$ are modified across regimes. The resulting FT-DFG contains an unshielded collider structure of the form $X_t \rightarrow Y_t \leftarrow Z_t$, where $Y_t$ is the common endpoint of the changed edges. In this case, LDiffPC exploits its collider orientation rule to orient both edges toward $Y_t$ (Figure \ref{fig:partially_oriented_difference_graph_example} right).

%\textcolor{red}{ It is important to note that tsLDiffPC cannot orient instantaneous edges when no unshielded collider is present, leaving them undirected. An example is shown in the left panel of Figure~\ref{fig:partially_oriented_difference_graph_example}. Setting R illustrates a limitation of tsLDiffPC. Since the difference graph contains only one changed edge, no unshielded collider is available and the edge cannot be oriented from the difference graph alone.

\end{example}

\subsection{The tsDCI algorithm}
%\textcolor{red}{TODO: Dasha: does DCI orientation phase requires iid? }

%\textcolor{red}{tentative de response Anouk : Oui pour la F-distribution, mais donc la preuve du theoreme 4.4  reste le même sans cette assumption, ce livre parle des test pour ts New Introduction to Multiple time-series Analysis}

The second algorithm, denoted tsDCI is a time-series extension of DCI.
The tsDCI algorithm proceeds in two steps. First, it estimates the FT-DFG skeleton using the same procedure as tsLDiffPC. Lagged edges are oriented directly according to temporal order, while instantaneous edges are oriented by testing the invariance of residual variances across regimes. More precisely, tsDCI searches for a conditioning set $\mathbb{S}$  such that $\sigma^{\mathcal N}_{X_t\mid \mathbb S} = \sigma^{\bar{\mathcal N}}_{X_t\mid \mathbb S}$. If such a set $\mathbb S$ exists and $Y_t\notin\mathbb S$, the edge is oriented as $X_t\to Y_t$; if $Y_t\in\mathbb S$, it is oriented as $Y_t\to X_t$. The procedure is repeated until no further orientation is possible.
The pseudocode is given in Appendix. %\ref{appendix:pseudocode}.

To ensure correctness, tsDCI requires, similarly to tsLDiffPC, the diff-adjacency-faithfulness assumption. But for orientation, it relies on the following assumption on residual variances, which is the temporal analogue of Assumption~4.2 in~\cite{Wang_Neurips_2018}.

\begin{assumption}[Var-orientation-faithfulness]
\label{assumption:diff_var_orient_faithfulness}
    (temporal analogue of Assumption 4.2~\cite{Wang_Neurips_2018}) For any node $X_{t- \ell}, Y_t \in \mathbb{V}$ it holds that:
    \begin{itemize}
        \item If $\alpha^{\mathcal N}_{{X_{t- \ell}, }Y_t}   \neq \alpha^{\bar{\mathcal N}}_{X_{t- \ell}, Y_t}$ then for every $\mathbb{S} \subseteq \mathbb{V} \setminus \{Y_t\}$, $\sigma^{\mathcal N}_{Y_t | \mathbb{S}} \neq \sigma^{\bar{\mathcal N}}_{Y_t | \mathbb{S}}$ and  $\sigma^{\mathcal N}_{X_{t- \ell} | \mathbb{S} \cup \{ Y_t \}} \neq \sigma^{\bar{\mathcal N}}_{X_{t- \ell} | \mathbb{S} \cup \{ Y_t \}}$.

        \item If $\sigma^{\mathcal N}_{Y_t} \neq \sigma^{\bar{\mathcal N}}_{Y_t}$, then $\sigma^{\mathcal N}_{Y_t | \mathbb{S}} \neq \sigma^{\bar{\mathcal N}}_{Y_t | \mathbb{S}}$ for every $\mathbb{S} \subseteq \mathbb{V} \setminus \{ Y_t\}$.
    \end{itemize}
\end{assumption}

Theorem~\ref{theorem:correctness_tsdci} establishes the correctness of tsDCI under diff-adjacency-faithfulness and the residual-variance assumption introduced above.

\begin{theorem}
\label{theorem:correctness_tsdci}
Let $\mathcal{D}=(\mathbb{V}, \mathbb{E})$ be the true FT-DFG and $\widehat{\mathcal{D}}=(\mathbb{V}, \widehat{\mathbb{E}}, \widehat{\mathbb{E}}^{-})$ be the output of tsDCI.
 If Assumptions~\ref{assumption:sufficiency},\ref{assumption:stationarity},\ref{assumption:anomalies},\ref{assumption:order},\ref{assumption:diff_adjacency_faithfulness} and \ref{assumption:diff_var_orient_faithfulness} are satisfied  and we are given perfect conditional equality information about all pairs of variables then $\mathrm{Skel}(\widehat{\mathcal{D}})  = \mathrm{Skel}(\mathcal{D})$  and $\widehat{\mathbb{E}}\subseteq\mathbb{E}$.
\end{theorem}

%The variance-based orientation criterion of tsDCI can orient certain edges that remain undirected under tsLDiffPC. 
%The variance-based orientation criterion of tsDCI can orient certain edges that remain undirected under tsLDiffPC. Conversely, there exist difference graphs that are fully oriented by tsLDiffPC but only partially oriented by tsDCI.
%\textcolor{red}{But also the opposite is true}

\begin{example}
Similarly to tsLDiffPC, tsDCI recovers the adjacencies in both settings R and RB. However, unlike tsLDiffPC, tsDCI can orient certain isolated changed edges. For example in Setting R, if $\sigma_{Y_t}^{\mathcal N}=\sigma_{Y_t}^{\bar{\mathcal N}}$, then Assumption~\ref{assumption:diff_var_orient_faithfulness} guarantees the existence of a conditioning set witnessing residual-variance invariance for $Y_t$. Consequently, tsDCI orients the edge toward $Y_t$, thereby recovering the true FT-DFG as shown in the middel panel of Figure~\ref{fig:partially_oriented_difference_graph_example}.
\end{example}

\subsection{The tsLDiffPC$^2$ and the tsDCIPC algorithms}

We also consider two augmented variants, tsLDiffPC$^2$ and tsDCIPC. 
The main idea behind these algorithms is to first apply a causal discovery method to the normal data, and then leverage the inferred orientations by projecting them onto the FT-DFG. 
%By construction, and under Assumption~\ref{assumption:order}, any orientation in the difference graph cannot contradict an orientation in the FT-DAG. However, 
To enable this step, we require the following additional assumption.

\begin{assumption}[Anomalous subgraph]
\label{assumption:graph_in_one_regime_is_subgraph_of_another}
Let $\mathcal{G}^{\mathcal{N}} = (\mathbb{V}, \mathbb{E}^{\mathcal{N}})$ and $\mathcal{G}^{\bar{\mathcal{N}}} = (\mathbb{V}, \mathbb{E}^{\bar{\mathcal{N}}})$ denote the FT-DAGs associated with the normal and anomalous regimes, respectively. It is assumed that
$
\mathbb{E}^{\bar{\mathcal{N}}} \subseteq \mathbb{E}^{\mathcal{N}}
% \quad \text{or} \quad
% \mathbb{E}^{\bar{\mathcal{N}}} \subseteq \mathbb{E}^{\mathcal{N}}.
$.
\end{assumption}

%This assumption has been considered in previous root cause analysis methods. For instance, EasyRCA and SGRCA assume that the anomalous graph is a subgraph of the normal graph. This assumption could be replaced by an alternative one. In particular, one could instead assume that the normal graph is a subgraph of the anomalous graph. In that case, the PC algorithm should be applied to the anomalous data rather than to the normal data. We adopt the former assumption for two practical reasons. First, anomalies are less likely to introduce entirely new causal relationships in a DT-DSCM describing the normal behavior of the system; rather, they are more likely to modify existing mechanisms. Second, the normal regime typically contains more observations than the anomalous regime, making causal discovery more reliable in the normal regime. A more general strategy would be to assume only that one graph is a subgraph of the other. This would require running PC on both the normal and anomalous regimes and combining the resulting graphs to refine the orientation of the output of tsLDiffPC. Although theoretically possible, this strategy is computationally more demanding and potentially less stable.

This assumption is common in the root cause analysis literature and is notably adopted by EasyRCA and SGRCA. An alternative would be to assume that the normal graph is a subgraph of the anomalous graph and to apply tPC to the anomalous regime instead. We focus on the former setting because anomalies are more likely to modify existing mechanisms than to create entirely new ones, and because the normal regime typically contains more observations, leading to more reliable causal discovery.

The tsLDiffPC$^2$ and tsDCIPC algorithms first run tsLDiffPC and tsDCI, respectively, to obtain a partially oriented FT-DFG. Then apply tPC~\cite{Spirtes_2000}  to the normal regime and estimate a partially oriented FT-DAG. As in the standard PC algorithm~\cite{Spirtes_2000}, tPC removes edges by testing conditional independences and subsequently orients the remaining edges using temporal order, the collider-detection rule, and Meek rules. Finally, each undirected edge in $\widehat{\mathcal D}$ is oriented whenever tPC provides an orientation in $\widehat{\mathcal G}^{\mathcal N}$. Since both methods rely on tPC, they additionally require the assumptions underlying tPC.

\begin{assumption}[Faithfulness in the normal regime]
\label{assumption:faithfulness}
Let $\mathcal{M}^{\mathcal N}$ be the DT-DSCM corresponding to the normal regime, with causal graph $\mathcal{G}^{\mathcal N}$ and distribution $\mathbb{P}^{\mathcal N}$. All conditional independences in $\mathbb{P}^{\mathcal N}$ are implied by the causal Markov property associated with $\mathcal{G}^{\mathcal N}$, which states that each variable is independent of its non-descendants given its parents. 
\end{assumption}
Table~\ref{tab:assumptions_algorithms} summarizes all assumptions required by each of the algorithms introduced in this paper.
The soundness of tsLDiffPC$^2$ and tsDCIPC trivially follows from the soundness of  tsLDiffPC, tsDCI and tPC.

\begin{corollary}
\label{cor:pc_augmented}
Let $\mathcal{D}=(\mathbb{V}, \mathbb{E})$ be the true FT-DFG, $\mathcal{A}\in\{\text{tsLDiffPC}^2,\text{tsDCIPC}\}$ and
$\widehat{\mathcal D}=(\mathbb V,\widehat{\mathbb E},\widehat{\mathbb E}^{-})$
be the output of $\mathcal A$.
Under the assumptions required by $\mathcal A$ in Table~\ref{tab:assumptions_algorithms}, and given perfect conditional equality information for all pairs of variables,
$\mathrm{Skel}(\mathcal D)=\mathrm{Skel}(\widehat{\mathcal D})$
and $\widehat{\mathbb E}\subseteq\mathbb E.$
\end{corollary}

\begin{example}
%In Setting R, although tsLDiffPC cannot orient the edge in the FT-DFG, tsLDiffPC$^2$ leverages orientations learned from the normal regime. In particular, in the FT-DAG, the structure forms a collider with non-adjacent extremities, allowing the collider detection rule to be applied by tPC. This orientation is then transferred to the difference graph, enabling tsLDiffPC$^2$ to correctly orient the edge. In Setting RB, tsLDiffPC$^2$ yields the same orientations as tsLDiffPC. Similarly, tsDCIPC combines residual-variance-based orientation with orientations obtained from tPC. As a result, it is able to orient edges in both settings, even in cases where residual-variance conditions alone would be insufficient. The outputs of both algorithms are shown in Figure~\ref{fig:partially_oriented_difference_graph_example} (middle) for Setting R and in Figure~\ref{fig:partially_oriented_difference_graph_example} (right) for Setting RB.

In Setting R, tsLDiffPC cannot orient the edge in the FT-DFG. However, tsLDiffPC$^2$ leverages orientations learned by tPC in the normal regime. Since the corresponding FT-DAG contains a collider with non-adjacent extremities, tPC correctly orients the edge, and this orientation is subsequently transferred to the difference graph. In Setting RB, tsLDiffPC$^2$ produces the same output as tsLDiffPC. Similarly, tsDCIPC combines residual-variance-based orientations with orientations obtained from tPC. Consequently, it correctly orients the edge in both settings. The outputs of both algorithms are shown in Figure~\ref{fig:partially_oriented_difference_graph_example} (middle) for Setting R and in Figure~\ref{fig:partially_oriented_difference_graph_example} (right) for Setting RB.
\end{example}

%\subsection{The tsMBGH and tsiSCAN algorithm}

% Unlike tsLDiffPC and tsDCI, these adaptations remain heuristic, as the identifiability guarantees of MBGH \cite{Malik_UAI_2024} and iSCAN \cite{Chen_Neurips_2023} are established for i.i.d. data generated by static structural equation models and do not directly extend to the temporal setting. We therefore include tsMBGH and tsiSCAN as empirical baselines without theoretical guarantees.

\begin{table}[t]
\centering
\small
\begin{tabular}{lcccccc}
\toprule
\textbf{Assumption} & \textbf{tsLDiffPC} & \textbf{tsLDiffPC$^2$} & \textbf{tsDCI} & \textbf{tsDCIPC}   \\
\midrule
Linearity & \checkmark & \checkmark & \checkmark & \checkmark \\
Sufficiency (As.~\ref{assumption:sufficiency}) & \checkmark & \checkmark & \checkmark & \checkmark \\
Stationarity (As.~\ref{assumption:stationarity}) & \checkmark & \checkmark & \checkmark & \checkmark \\
Anomalies (As.~\ref{assumption:anomalies}) & \checkmark & \checkmark & \checkmark & \checkmark \\
Order (As.~\ref{assumption:order}) & \checkmark & \checkmark & \checkmark & \checkmark \\
D. a. faithful. (As.~\ref{assumption:diff_adjacency_faithfulness}) & \checkmark & \checkmark & \checkmark & \checkmark \\
D. o. faithful. (Ass.~\ref{assumption:diff_orient_faithfulness}) & \checkmark & \checkmark &  &  \\
D. v. o. faithful. (As.~\ref{assumption:diff_var_orient_faithfulness}) &  &  & \checkmark & \checkmark \\
Subgraph (As.~\ref{assumption:graph_in_one_regime_is_subgraph_of_another}) &  & \checkmark &  & \checkmark \\
Faithful. (As.~\ref{assumption:faithfulness}) &  & \checkmark &  & \checkmark \\
\bottomrule
\end{tabular}
\caption{Assumptions required by each algorithm for soundness.}
\label{tab:assumptions_algorithms}
\end{table}

\section{Difference graph discovery for detecting root causes}
\label{sec:Main_Root_causes}

% Under the assumptions of the previous section, every directed edge returned by tsLDiffPC or tsDCI corresponds to a true changed causal effect. As a consequence, any variable identified from these directed edges is a true effect-defying root cause. The algorithms may however fail to recover some root causes whenever changed edges remain unoriented. 

% \begin{corollary}
% Let $\mathcal{C}$ be the set of effect-defying root causes and let
% $\widehat{\mathcal{C}}$ be the set of root causes detected by an algorithm
% $$\mathcal{A}\in\{\textnormal{tsLDiffPC},\textnormal{tsLDiffPC}^2,\textnormal{tsDCI},\textnormal{tsDCIPC}\}.$$
% If the assumptions required by $\mathcal{A}$, as specified in
% Table~\ref{tab:assumptions_algorithms}, are satisfied, and if we are given
% perfect conditional equality information for all pairs of variables, then
% $\widehat{\mathcal{C}}\subseteq \mathcal{C}$.
% \end{corollary}

Up to this point, we have not directly addressed the detection of root causes. Once the FT-DFG has been estimated, root-cause detection follows naturally.

%Instead, we have focused on recovering a partially directed version of the DFG from time-series data. Once this graph has been estimated, root-cause detection follows naturally.

%Under the assumptions established in Table~\ref{tab:assumptions_algorithms}, every directed edge returned by an algorithm $\mathcal{A} \in \{\text{tsLDiffPC},\, \text{tsLDiffPC}^2,\, \text{tsDCI},\, \text{tsDCIPC}\}$ is a true changed edge in the underlying DFG. Consequently, any variable identified as a target of such a directed edge is a confirmed effect-defying root cause. However, some edges may remain undirected, thereby leaving certain root causes uncertain. When a changed edge $X_{t} - Y_t$ remains unoriented in the output, it is known that at least one of $X$ or $Y$ is a root cause, yet $\mathcal{A}$ cannot determine which one with certainty; we refer to $X$ and $Y$ in this case as \emph{potential root causes}.

Under the assumptions of Table~\ref{tab:assumptions_algorithms}, every directed edge returned by an algorithm
$\mathcal{A}$
corresponds to a true changed edge in the FT-DFG. Hence, any variable targeted by such an edge is a confirmed effect-defying root cause. In contrast, an undirected edge $X_t - Y_t$ only implies that at least one of $X$ or $Y$ is a root cause; we therefore refer to $(X,Y)$ as a pair of \emph{potential root causes}.

To formalize these recovery guarantees, define the set of certain root causes and the set of pairs of potential root causes as
$
\widehat{\mathcal{C}} := \bigl\{ Y : \exists\, X_{t-\ell} \in \mathbb{V},\ 
X_{t-\ell} \to Y_t\in \widehat{\mathcal D}\bigr\},
$ and
$
\widehat{\mathcal{C}}_P := \bigl\{ (X, Y) : 
X_{t} - Y_t \in \widehat{\mathcal D} \bigr\}.
$
Corollary~\ref{cor:recovery} establishes the recovery guarantees for both sets.

\begin{corollary}
\label{cor:recovery}
Let $\mathcal{C}$ denote the set of effect-defying root causes, and let 
$\widehat{\mathcal{C}}$ and $\widehat{\mathcal{C}}_P$ be as defined above, 
for an algorithm
$
\mathcal{A} \in \{\text{tsLDiffPC},\, \text{tsLDiffPC}^2,\, 
\text{tsDCI},\, \text{tsDCIPC}\}.
$
Suppose that the assumptions required by $\mathcal{A}$, as specified in 
Table~\ref{tab:assumptions_algorithms}, are satisfied, and we are given perfect conditional equality information for all pairs of variables. Then: $\widehat{\mathcal{C}} \subseteq \mathcal{C}$ and $\forall\, (X, Y) \in \widehat{\mathcal{C}}_P$,  $[X \in \mathcal{C}] \lor [Y \in \mathcal{C}]$.
%\begin{itemize}
%    \item $\widehat{\mathcal{C}} \subseteq \mathcal{C}$,
%    \item $\forall\, (X, Y) \in \widehat{\mathcal{C}}_P, \quad [X \in 
%    \mathcal{C}] \lor [Y \in \mathcal{C}]$,
%\end{itemize}
Furthermore among all the possible ways of selecting one variable from each pair in $\widehat{\mathcal{C}}_P$, there exists a selection $\widehat{\mathcal{C}}_S$ such that
$
\widehat{\mathcal{C}} \cup \widehat{\mathcal{C}}_S = \mathcal{C}.
$
\end{corollary}

%The first item of the corollary shows that, in theory, all detected root causes $\widehat{\mathcal{C}}$ are indeed true root causes, i.e., $\widehat{\mathcal{C}} \subseteq \mathcal{C}$. However, it does not guarantee that all true root causes are recovered, i.e., $\widehat{\mathcal{C}} = \mathcal{C}$. The reason is that the presence of an undirected edge in the inferred DFG indicates that at least one of the two endpoints is a root cause, but it is not possible, based on the graph alone, to determine which one. This is reflected in the second item of the corollary. Finally, and importantly, the corollary shows that the missing root causes, i.e., those not included in $\widehat{\mathcal{C}}$, can be identified within the set $\widehat{\mathcal{C}}_P$.
The corollary implies that every detected root cause is a true root cause, i.e., $\widehat{\mathcal{C}} \subseteq \mathcal{C}$. However, not all root causes are necessarily identified with certainty, since an undirected edge only indicates that only one of its endpoints is a root cause. The second item guarantees that all such unresolved root causes are contained in the set of potential root causes $\widehat{\mathcal{C}}_P$. 

\begin{example}  
In both settings R and RB, $\mathcal{C} = \{Y\}$. On the left panel, tsLDiffPC only identifies \textit{potential root causes} : $\widehat{\mathcal{C}} = \emptyset$ and $\widehat{\mathcal{C}}_P = \{(X,Y)\}$ meaning that $X$ or $Y$ is the root cause. In the middle panel tsLDiffPC$^2$ and tsDCIPC directly identifies the root cause, i.e. $\widehat{\mathcal C}_P=\emptyset$ so $\widehat{\mathcal C}=\{Y\}=\mathcal C$. For setting RB, all four methods correctly identify the root cause, i.e. $\widehat{\mathcal C}_P=\emptyset$ and $\widehat{\mathcal{C}} = \{Y\} = \mathcal{C}$ for tsLDiffPC, tsLDiffPC$^2$, and for tsDCIPC and tsDCI provided that $\sigma_Y^{\mathcal{N}} = \sigma_Y^{\bar{\mathcal{N}}}$.

\end{example}

\section{Experiments}
\label{sec:Experiments}

%In this section, we evaluate our algorithms on  simulated data as well as two real-world datasets. In particular, we compare the algorithms introduced in this paper with MicroCause and RCD. We also consider naive extensions of existing graph discovery algorithms, namely tsMBGH and tsiSCAN, derived from MBGH and iSCAN.  In addition, we include a simple baseline approach, denoted by \textnormal{tPC}$^2$, which consists of applying a causal discovery algorithm (tPC) separately to the normal and anomalous regimes, and identifying as root causes those nodes whose set of parents differs between the two regimes.

In this section, we evaluate the proposed algorithms on both simulated and real-world datasets.
%\footnote{The complete implementation of the proposed algorithm, together with all code required to reproduce the experiments, is available at \url{https://github.com/CIPHOD/pyCIPHOD/tree/main/reproducibility/ecmlpkddcaesar2026}.}
We compare tsLDiffPC, tsLDiffPC$^2$, tsDCI, and tsDCIPC against two root cause analysis methods, namely MicroCause and RCD. We also consider naive extensions of existing graph discovery algorithms, namely tsMBGH and tsiSCAN, derived from MBGH, iSCAN, and \textnormal{tPCUnion}, which estimates a temporal graph separately in each regime and compares direct effects of the normal and anomalous regimes that are identified using the union graph.
%This baseline applies tPC independently to the normal and anomalous regimes and identifies as root causes the variables whose estimated parent sets differ across the two regimes.

For tsLDiffPC, tsLDiffPC$^2$, tsDCI, and tsDCIPC, the skeletons are estimated using a linear regression coefficient equality test across regimes. Edge orientations in tsDCI and tsDCIPC are obtained through an $F$-test for equality of residual variances across regimes. The tPC procedure used by tsLDiffPC$^2$ and tsDCIPC employs Fisher's $Z$ test for conditional independence testing. All other methods are run with their default hyperparameters. All statistical tests are performed at significance level $\alpha=0.05$, and in all algorithms, $\ell_{max}$ is set to one. A sensitivity analysis with respect to the significance level is provided in Appendix. %~\ref{appendix:sensitivity_analysis}.
 
% Additonal experiments are provided in Appendix \textcolor{orange}{TODO}.

%\textcolor{red}{Practical considerations: All algorithms are run with a maximum lag of $1$. But in appendix, estimation, test, hyperprameters}

% \textcolor{red}{
% We also evaluate two heuristic temporal adaptations of existing difference-graph methods: tsMBGH and tsiSCAN. Both methods operate on a lag-augmented representation that encodes temporal dependencies within a finite lag window as a static graph. tsMBGH applies the original MBGH peeling procedure to this augmented representation while restricting terminal vertices to present-time variables. tsiSCAN identifies shifted variables on the augmented representation and reconstructs temporal difference edges from the estimated causal ordering. 
% }

% \textcolor{red}{DFG-by-CD: causal discovery in the normal regime and causal discovery in the anomalous regine and then compare the graphs}

% \textcolor{red}{EasyRCA$^*$}

% \textcolor{red}{MicroCause}

% \textcolor{red}{RCD}

\subsection{Simulated data}
\label{sec:Experiments_sim}

We evaluate the algorithms on simulated multivariate time-series data generated from random finite-lag causal graphs. First, a random FT-DAG is sampled with a predefined number of time-series ($d \in \{3,5,7, 9\}$), edge probability $0.3$ and $\ell_{\max}=1$. Edge coefficients are sampled uniformly from $[0.2,0.8]$ with random sign. To introduce anomalies, we consider three settings. In all cases, a target vertex is selected as an effect-defying root cause, and its incoming edge coefficients are modified while keeping the graph structure unchanged.
% \begin{itemize}
    % \item 
    In the first setting, only one incoming edge to the target vertex is modified; 
    % \item 
    in the second setting, all incoming edges of the target vertex are modified;
    % \item 
    in the third setting, all incoming edges to the target vertex are modified, and we additionally impose that the target vertex has at least two parents. This last setting allows us to focus on cases where tsLDiffPC has sufficient information to orient edges.
% \end{itemize}

Time series are then simulated recursively from the corresponding linear DT-DSCM with Gaussian noise. For each setting, we simulate 10 datasets. To evaluate the different algorithms, we use the F1-score defined with respect to the set of true root causes $\mathcal{C}$ and a set of predicted positives $\mathcal{P}$, \ie, $\text{F1}(\mathcal{C}, \mathcal{P}) = \frac{2 \, |\mathcal{C} \cap \mathcal{P}|}{|\mathcal{C}| + |\mathcal{P}|}, $ where two choices are considered for the set of predicted positives $\mathcal{P}$:
% \begin{itemize}
    % \item 
    $\mathcal{P} = \widehat{\mathcal{C}}$, where only confirmed root causes are considered; 
    % \item 
    $\mathcal{P} = \widehat{\mathcal{C}} \cup \widehat{\mathcal{C}}_P$ where both confirmed and potential root causes are considered. 
% \end{itemize}

Figure~\ref{fig:res_simulated_data} shows that, when only confirmed root causes are considered, \textnormal{tsDCI} and \textnormal{tsDCIPC} achieve the best overall performance in the first and second settings. In contrast, \textnormal{tsLDiffPC} often fails to orient the detected changed edges, although \textnormal{tsLDiffPC}$^2$ improves its performance. In the third setting, \textnormal{tsLDiffPC} and especially \textnormal{tsLDiffPC}$^2$ become competitive with the DCI-based methods. This setting is particularly favorable to \textnormal{tsLDiffPC}, as multiple changed parents are more likely to form unshielded colliders that facilitate edge orientation.
When potential root causes are also taken into account, the performance of \textnormal{tsLDiffPC} improves in all settings, confirming that its main limitation lies in the fact that it discovers a partially oriented graph.
Among the remaining methods, RCD is the most competitive baseline, whereas \textnormal{tsMBGH}, MicroCause, tPCUnion, and especially \textnormal{tsiSCAN} generally obtain lower F1-scores.

\input{plot_simulated_data}

\subsection{IT monitoring}
\label{sec:IT_monitoring}
We evaluate the algorithms on IT monitoring data, namely, on  the Ingestion dataset, introduced in
\cite{Assaad_AISTATS_2023}. The dataset was collected  with a one-minute sampling rate and contains eight time-series measured
over normal and anomalous operating periods. The time-series describe different
activities of the ingestion pipeline, including message preprocessing (PMDB), message routing (MDB), metric extraction (CMB), database insertion (MB), updates of recent metric values (LMB), merging with check-message information (RTMB), insertion of historical status information (GSIB), and writing data to Elasticsearch (ESB). Each value is computed by multiplying the number of messages processed by the corresponding component during a $10$-minute window by the average execution latency in that window, and then normalizing by the window length.
A visualization of the time series is provided in Appendix. %~\ref{appendix:real_data_viz}.

The dataset contains $100$ anomalous observations and $1000$ normal observations
preceding the anomalous period. It was highlighted in \cite{Assaad_AISTATS_2023} that according to system experts two time-series were identified as root causes of the anomalies. The first corresponds to PMDB; however, based on the expert knowledge discussed in \cite{Assaad_AISTATS_2023}, this time-series is not caused by any other observed time-series, making it unlikely to be an effect-defying root cause. The second corresponds to ESB, which is more plausibly an effect-defying root cause, as it was detected as such by EasyRCA using data and the available background knowledge.

%\textcolor{blue}{Timothée : je trouve bizarre d'utiliser l'output de EasyRCA comme ground truth pour la vraie root cause, on n'a pas de ground truth donnée par l'expert ?}

%We ran the algorithms with a maximum lag of $1$.

On this dataset, all tsLDiffPC-based and tsDCI-based algorithms inferred the same graph, shown in Appendix. %~\ref{appendix:real_data_graphs}.
From this graph we can deduce that all these algorithms detected the effect-defying root cause ESB, along with three false positives. 
%Interestingly, these results are comparable to those obtained with EasyRCA$^*$~\cite{Assaad_AISTATS_2023}. 
By contrast, tsMBGH did not detect any root cause, while tPCUnion and tsiSCAN identified all time-series as root causes (except RTMB for tsiSCAN). Among the other baseline methods, RCD identified LMB, MB and PMDB, whereas MicroCause detected CMB and PMDB. Overall, the tsLDiffPC-based and tsDCI-based methods are the only approaches that consistently recover ESB.

%ANouk MicroCause : check message bolt,message dispatcher bolt = CMB, PMDB}

%\textcolor{orange}{ANouk RCD : capacity last metric bolt = ESB}

%\textcolor{orange}{EasyRCA$^*$ : detected 5 root causes but among those labeled only RTMB and ESB as effect-defying root causes.}

% pre_Message_dispatcher_bolt (PMDB); check_message_bolt (CMB); 
% group_status_information_bolt (GSIB); message_dispatcher_bolt; (MDB) 
% metric_bolt (MB); 
% capacity_last_metric_bolt (LMB)
% capacity_elastic_search_bolt (ESB); 

\subsection{Intensive care monitoring}
\label{sec:Intensive_care_monitoring}

We illustrate the algorithms on a patient record from the MIMIC-IV Waveform Database \cite{mimic4}, a publicly available critical care database of high-resolution physiological signals from ICU patients. Nine variables are retained: mean arterial blood pressure (ABPm), mean pulmonary arterial pressure (PAPm), mean umbilical arterial pressure (UAPm), heart rate (HR), ST-segment deviation on leads V and III (ST-V, ST-III), central venous pressure (CVPm), respiratory rate (RR), and oxygen saturation (SpO2). Since the raw signals are irregularly sampled with missing values, all variables are first aligned onto a common time grid via linear interpolation. Regime change detection is performed on ST-V using the PELT algorithm \cite{killick2014changepoint} with an MBIC penalty, yielding a single changepoint. A symmetric window of $10\,000$ observations around this changepoint defines two temporal segments. For consistency with our notation, we refer to the earlier segment as the normal regime and to the later segment as the anomalous regime; these labels indicate temporal order only, not clinical status. A visualization of the time series is provided in Appendix. %~\ref{appendix:real_data_viz}.

%\textcolor{red}{yielding a single changepoint $t^*$. A symmetric window of $10\,000$ observations around $t^*$ defines a pre-change regime $\mathcal{R}_1 = \{t < t^*\}$ and a post-change regime $\mathcal{R}_2 = \{t \geq t^*\}$} (see Appendix~\ref{appendix:icu}). 

The estimated FT-DFGs are shown in Appendix. %~\ref{appendix:real_data_graphs}. 
The tsLDiffPC-based methods identify the same set of candidate root causes, namely RR, HR, PAPm, ABPm and CVPm. The tsDCI-based methods yield a similar result, but leave an ambiguity between PAPm and CVPm due to an unresolved edge orientation. The tsiSCAN algorithm identifies RR as the sole root cause. In contrast, tPCUnion detects all variables as root causes except SpO2. MicroCause identifies PAPm and CVPm as root causes, partially overlapping with the variables detected by the difference graph methods. And RCD identifies ST-V, which is not selected by any other method. Although no ground truth is available, there is an agreement among most algorithms on PAPm, CVP, and RR.

%\textcolor{orange}{MicroCause : PAPm, CVPm}
%\textcolor{orange}{RCD : ST.V}

%The estimated difference graphs are shown in Figure~\ref{fig:mimic}. tsLDiffPC$^2$ and tsDCI return the same three root causes: RR, HR, and PAPm, and additionally report an altered dependency between ST-V and ST-III; however, since this edge is undirected in the difference graph, it is not possible to determine which of the two variables drives the change. tsISCAN identifies RR alone, a solution nested within the former. 

%\textcolor{red}{à modifier car les res ont changé, refaire les graphs et text}

\section{Conclusion}
\label{sec:Discussion}
In this paper, we investigated how difference-graph discovery methods can be leveraged for root-cause analysis in time-series data. Our results demonstrate the potential of these methods for identifying structural changes associated with observed anomalies. The proposed approaches nevertheless rely on several assumptions regarding the underlying causal structure and data-generating processes. Moreover, their computational cost increases rapidly with both the number of variables and the maximum time lag considered, which may limit their applicability to large-scale systems.

\subsubsection*{Acknowledgment}
This work was supported by the CIPHOD project (ANR-23-CPJ1-0212-01).

%% file: plot_simulated_data.tex
\begin{figure}[t]
\centering
\begin{subfigure}{0.31\textwidth}
\centering
\begin{tikzpicture}
\begin{axis}[
    width=\textwidth,
    height=4cm,
    xlabel={Number of time-series},
    ylabel={F1 ($\mathcal{C}$, $\widehat{\mathcal{C}}$)},
    title={One parent},
    ymin=0, ymax=1,
legend style={
    at={(2.2,1.3)},
    anchor=south,
    legend columns=6,
    font=\scriptsize
}]
% -----------------------------
%%%% f1 incoming shifted one parent
% -----------------------------
\addplot+[
    thick,
    mark=*, red, mark options={fill=red, draw=red}, solid,
    error bars/.cd,
        y dir=both,
        y explicit
] coordinates {
    (3,0.17) +- (0,0.36)
    (5,0.0) +- (0,0.0)
    (7,0.03) +- (0,0.08)
    (9,0.0) +- (0,0.0)
};
\addlegendentry{tsLDiffPC}

\addplot+[
    thick,
    mark=*, orange, mark options={fill=orange, draw=orange}, solid,
    error bars/.cd,
        y dir=both,
        y explicit
] coordinates {
    (3.1,0.53) +- (0,0.39)
    (5.1,0.5) +- (0,0.45)
    (7.1,0.56) +- (0,0.36)
    (9.1,0.47) +- (0,0.42)
};
\addlegendentry{tsLDiffPC$^2$}

\addplot+[
    thick,
    mark=*, blue, mark options={fill=blue, draw=blue}, solid,
    error bars/.cd,
        y dir=both,
        y explicit
] coordinates {
    (3.2,0.83) +- (0,0.32)
    (5.2,0.6) +- (0,0.52)
    (7.2,0.73) +- (0,0.45)
    (9.2,0.8) +- (0,0.42)
};
\addlegendentry{tsDCI}

\addplot+[
    thick,
    mark=*, cyan, mark options={fill=cyan, draw=cyan}, solid,
    error bars/.cd,
        y dir=both,
        y explicit
] coordinates {
    (3.3,0.83) +- (0,0.32)
    (5.3,0.6) +- (0,0.52)
    (7.3,0.73) +- (0,0.45)
    (9.3,0.8) +- (0,0.42)
};
\addlegendentry{tsDCIPC}

\addplot+[
    thick,
    mark=*, olive, mark options={fill=olive, draw=olive},solid, 
    error bars/.cd,
        y dir=both,
        y explicit
] coordinates {
    (2.9,0.43) +- (0,0.47)
    (4.9,0.47) +- (0,0.42)
    (6.9,0.28) +- (0,0.39)
    (8.9,0.32) +- (0,0.29)
};
\addlegendentry{tsMBGH}

\addplot+[
    thick,
    mark=*, yellow, mark options={fill=yellow, draw=yellow}, solid,
    error bars/.cd,
        y dir=both,
        y explicit
] coordinates {
    (2.8,0.0) +- (0,0.0)
    (4.8,0.24) +- (0,0.43)
    (6.8,0.13) +- (0,0.32)
    (8.8,0.0) +- (0,0.0)
};
\addlegendentry{tsiSCAN}

\addplot+[
    thick,
    mark=*, teal, mark options={fill=teal, draw=teal}, solid,
    error bars/.cd,
        y dir=both,
        y explicit
] coordinates {
    (2.6,0.6) +- (0,0.21)
    (4.6,0.13) +- (0,0.28)
    (6.6,0.27) +- (0,0.34)
    (8.6,0.13) +- (0,0.28)
};
\addlegendentry{MicroCause}

\addplot+[
    thick,
    mark=*, pink, mark options={fill=pink, draw=pink}, solid,
    error bars/.cd,
        y dir=both,
        y explicit
] coordinates {
    (2.7,0.8) +- (0,0.42)
    (4.7,0.5) +- (0,0.53)
    (6.7,0.3) +- (0,0.48)
    (8.7,0.4) +- (0,0.52)
};
\addlegendentry{RCD}

\addplot+[
    thick,
    mark=*, gray, mark options={fill=gray, draw=gray}, solid,
    error bars/.cd,
        y dir=both,
        y explicit
] coordinates {
    (3.,0.63) +- (0,0.28)
    (5,0.39) +- (0,0.23)
    (7,0.21) +- (0,0.19)
    (9,0.31) +- (0,0.17)
};
\addlegendentry{tPCUnion}

\end{axis}
\end{tikzpicture}
% \caption{Setting 1}
\end{subfigure}
\hfill
\begin{subfigure}{0.31\textwidth}
\centering
\begin{tikzpicture}
\begin{axis}[
    width=\textwidth,
    height=4cm,
    xlabel={Number of time-series},
    title={All parents},
    ymin=0, ymax=1
]
% -----------------------------
%%%% f1 incoming shifted all parents
% -----------------------------
\addplot+[
    thick,
    mark=*, red, mark options={fill=red, draw=red}, solid,
    error bars/.cd,
        y dir=both,
        y explicit
] coordinates {
    (3,0.17) +- (0,0.37)
    (5,0.17) +- (0,0.37)
    (7,0.33) +- (0,0.47)
    (9,0.5) +- (0,0.53)
}; %tsLDiffPC

\addplot+[
    thick,
    mark=*, orange, mark options={fill=orange, draw=orange}, solid,
    error bars/.cd,
        y dir=both,
        y explicit
] coordinates {
    (3.1,0.53) +- (0,0.39)
    (5.1,0.57) +- (0,0.42)
    (7.1,0.79) +- (0,0.25)
    (9.1,0.9) +- (0,0.16)
};%tsLDiffPC2

\addplot+[
    thick,
    mark=*, blue, mark options={fill=blue, draw=blue}, solid,
    error bars/.cd,
        y dir=both,
        y explicit
] coordinates {
    (3.2,0.83) +- (0,0.32)
    (5.2,0.67) +- (0,0.47)
    (7.2,0.79) +- (0,0.37)
    (9.2,0.9) +- (0,0.32)
}; %tsDCI

\addplot+[
    thick,
    mark=*, cyan, mark options={fill=cyan, draw=cyan}, solid,
    error bars/.cd,
        y dir=both,
        y explicit
] coordinates {
    (3.3,0.83) +- (0,0.32)
    (5.3,0.67) +- (0,0.47)
    (7.3,0.79) +- (0,0.37)
    (9.3,0.9) +- (0,0.32)
};%tsDCIPC

\addplot+[
    thick,
    mark=*, olive, mark options={fill=olive, draw=olive}, solid,
    error bars/.cd,
        y dir=both,
        y explicit
] coordinates {
    (2.9,0.43) +- (0,0.47)
    (4.9,0.42) +- (0,0.41)
    (6.9,0.48) +- (0,0.45)
    (8.9,0.49) +- (0,0.39)
}; %tsMalik

\addplot+[
    thick,
    mark=*, yellow, mark options={fill=yellow, draw=yellow}, solid,
    error bars/.cd,
        y dir=both,
        y explicit
] coordinates {
    (2.8,0.0) +- (0,0.0)
    (4.8,0.24) +- (0,0.42)
    (6.8,0.03) +- (0,0.09)
    (8.8,0.0) +- (0,0.0)
}; %tsiSCAN

\addplot+[
    thick,
    mark=*, teal, mark options={fill=teal, draw=teal}, solid,
    error bars/.cd,
        y dir=both,
        y explicit
] coordinates {
    (2.6,0.6) +- (0,0.21)
    (4.6,0.2) +- (0,0.32)
    (6.6,0.4) +- (0,0.34)
    (8.6,0.13) +- (0,0.28)
}; %microcause

\addplot+[
    thick,
    mark=*, pink, mark options={fill=pink, draw=pink}, solid,
    error bars/.cd,
        y dir=both,
        y explicit
] coordinates {
    (2.7,0.8) +- (0,0.42)
    (4.7,0.5) +- (0,0.53)
    (6.7,0.5) +- (0,0.53)
    (8.7,0.7) +- (0,0.48)
}; %rcd

\addplot+[
    thick,
    mark=*, gray, mark options={fill=gray, draw=gray}, solid,
    error bars/.cd,
        y dir=both,
        y explicit
] coordinates {
    (3.,0.6) +- (0,0.25)
    (5,0.41) +- (0,0.24)
    (7,0.25) +- (0,0.21)
    (9,0.27) +- (0,0.16)
};%pcunion

\end{axis}
\end{tikzpicture}
% \caption{Setting 1}
\end{subfigure}
\hfill
\begin{subfigure}{0.31\textwidth}
\centering
\begin{tikzpicture}
\begin{axis}[
    width=\textwidth,
    height=4cm,
    xlabel={Number of time-series},
    title={At least 2 parents},
    ymin=0, ymax=1,
legend style={
    at={(1.2,1.3)},
    anchor=south,
    legend columns=3,
    font=\scriptsize
}]
% -----------------------------
%%%% f1 incoming shifted, all parents with minimum 2 parents
% -----------------------------

\addplot+[
    thick,
    mark=*, red, mark options={fill=red, draw=red}, solid,
    error bars/.cd,
        y dir=both,
        y explicit
] coordinates {
    (3,0.72) +- (0,0.42) 
    (5,0.77) +- (0,0.42) 
    (7,0.83) +- (0,0.37)
    (9,0.7) +- (0,0.48)
}; %tsLDiffPC

\addplot+[
    thick,
    mark=*, orange, mark options={fill=orange, draw=orange}, solid,
    error bars/.cd,
        y dir=both,
        y explicit
] coordinates {
    (3.1,0.83) +- (0,0.22) 
    (5.1,0.83) +- (0,0.32) 
    (7.1,0.89) +- (0,0.25)
    (9.1,0.83) +- (0,0.32)
};%tsLDiffPC2

\addplot+[
    thick,
    mark=*, blue, mark options={fill=blue, draw=blue}, solid,
    error bars/.cd,
        y dir=both,
        y explicit
] coordinates {
    (3.2,0.65) +- (0,0.28)
    (5.2,0.87) +- (0,0.32)
    (7.2,0.69) +- (0,0.44)
    (9.2,0.77) +- (0,0.42)
};%tsDCI

\addplot+[
    thick,
    mark=*, cyan, mark options={fill=cyan, draw=cyan}, solid,
    error bars/.cd,
        y dir=both,
        y explicit
] coordinates {
    (3.3,0.63) +- (0,0.28)
    (5.3,0.87) +- (0,0.32)
    (7.3,0.69) +- (0,0.44)
    (9.3,0.77) +- (0,0.42)
};%ok

\addplot+[
    thick,
    mark=*, olive, mark options={fill=olive, draw=olive}, solid,
    error bars/.cd,
        y dir=both,
        y explicit
] coordinates {
    (2.9,0.53) +- (0,0.39)
    (4.9,0.51) +- (0,0.42)
    (6.9,0.51) +- (0,0.35)
    (8.9,0.39) +- (0,0.37)
}; %malik

\addplot+[
    thick,
    mark=*, yellow, mark options={fill=yellow, draw=yellow}, solid, 
    error bars/.cd,
        y dir=both,
        y explicit
] coordinates {
    (2.8,0.13) +- (0,0.28)
    (4.8,0.0) +- (0,0.0)
    (6.8,0.03) +- (0,0.09)
    (8.8,0.1) +- (0,0.32)
};% iscan ok

\addplot+[
    thick,
    mark=*, teal, mark options={fill=teal, draw=teal}, solid,
    error bars/.cd,
        y dir=both,
        y explicit
] coordinates {
    (2.6,0.53) +- (0,0.28)
    (4.6,0.27) +- (0,0.34)
    (6.6,0.13) +- (0,0.28)
    (8.6,0.07) +- (0,0.22)
}; %microcause

\addplot+[
    thick,
    mark=*, pink, mark options={fill=pink, draw=pink}, solid,
    error bars/.cd,
        y dir=both,
        y explicit
] coordinates {
    (2.7,0.4) +- (0,0.52)
    (4.7,0.7) +- (0,0.48)
    (6.7,0.7) +- (0,0.48)
    (8.7,0.6) +- (0,0.52)
}; %rcd

\addplot+[
    thick,
    mark=*, gray, mark options={fill=gray, draw=gray}, solid,
    error bars/.cd,
        y dir=both,
        y explicit
] coordinates {
    (3.,0.47) +- (0,0.38)
    (5,0.43) +- (0,0.34)
    (7,0.29) +- (0,0.11)
    (9,0.3) +- (0,0.14)
};%pcunion

\end{axis}
\end{tikzpicture}

% \caption{Setting 2}
\end{subfigure}

\vspace{0.4cm}

% -----------------------------
% -----------------------------
% NODE RElAXED
% -----------------------------
% -----------------------------

\begin{subfigure}{0.31\textwidth}
\centering
\begin{tikzpicture}
\begin{axis}[
    width=\textwidth,
    height=4cm,
    xlabel={Number of time-series},
    ylabel={F1 ($\mathcal{C}$, $\widehat{\mathcal{C}}\cup\widehat{\mathcal{C}}_P$)},    % title={Setting 3},
    ymin=0, ymax=1
]
% -----------------------------
%%% f1 node relaxed one parent
% -----------------------------
\addplot+[
    thick,
    mark=*, red, mark options={fill=red, draw=red}, solid,
    error bars/.cd,
        y dir=both,
        y explicit
] coordinates {
    (3,0.63) +- (0,0.25)
    (5,0.64) +- (0,0.08)
    (7,0.63) +- (0,0.13)
    (9,0.67) +- (0,0.0)
};%tsLDiffPC

\addplot+[
    thick,
    mark=*, orange, mark options={fill=orange, draw=orange}, solid,
    error bars/.cd,
        y dir=both,
        y explicit
] coordinates {
    (3.1,0.67) +- (0,0.27)
    (5.1,0.74) +- (0,0.20)
    (7.1,0.63) +- (0,0.30)
    (9.1,0.73) +- (0,0.14)
}; %tsLDiffPC2

\addplot+[
    thick,
    mark=*, blue, mark options={fill=blue, draw=blue}, solid,
    error bars/.cd,
        y dir=both,
        y explicit
] coordinates {
    (3.2,0.83) +- (0,0.32)
    (5.2,0.64) +- (0,0.48)
    (7.2,0.73) +- (0,0.45)
    (9.2,0.8) +- (0,0.42)
}; %tsDCI

\addplot+[
    thick,
    mark=*, cyan, mark options={fill=cyan, draw=cyan}, solid,
    error bars/.cd,
        y dir=both,
        y explicit
] coordinates {
    (3.3,0.83) +- (0,0.32)
    (5.3,0.64) +- (0,0.48)
    (7.3,0.73) +- (0,0.45)
    (9.3,0.8) +- (0,0.42)
}; %tsDCIPC

\addplot+[
    thick,
    mark=*, olive, mark options={fill=olive, draw=olive}, solid,
    error bars/.cd,
        y dir=both,
        y explicit
] coordinates {
    (2.9,0.43) +- (0,0.47)
    (4.9,0.47) +- (0,0.42)
    (6.9,0.28) +- (0,0.39)
    (8.9,0.32) +- (0,0.29)
}; %tsMalik

\addplot+[
    thick,
    mark=*, yellow, mark options={fill=yellow, draw=yellow}, solid,
    error bars/.cd,
        y dir=both,
        y explicit
] coordinates {
    (2.8,0.0) +- (0,0.0)
    (4.8,0.24) +- (0,0.42)
    (6.8,0.13) +- (0,0.32)
    (8.8,0.0) +- (0,0.0)
};%tsiSCAN

\addplot+[
    thick,
    mark=*, teal, mark options={fill=teal, draw=teal}, solid,
    error bars/.cd,
        y dir=both,
        y explicit
] coordinates {
    (2.6,0.6) +- (0,0.21)
    (4.6,0.13) +- (0,0.28)
    (6.6,0.27) +- (0,0.34)
    (8.6,0.13) +- (0,0.28)
}; %microcause

\addplot+[
    thick,
    mark=*, pink, mark options={fill=pink, draw=pink}, solid,
    error bars/.cd,
        y dir=both,
        y explicit
] coordinates {
    (2.7,0.8) +- (0,0.42)
    (4.7,0.5) +- (0,0.53)
    (6.7,0.3) +- (0,0.48)
    (8.7,0.4) +- (0,0.52)
}; %rcd

\addplot+[
    thick,
    mark=*, gray, mark options={fill=gray, draw=gray}, solid,
    error bars/.cd,
        y dir=both,
        y explicit
] coordinates {
    (3.,0.63) +- (0,0.28)
     (5,0.39) +- (0,0.23)
    (7,0.21) +- (0,0.19)
    (9,0.31) +- (0,0.17)
};%pcunion

\end{axis}
\end{tikzpicture}
% \caption{Setting 3}
\end{subfigure}
\hfill
\begin{subfigure}{0.31\textwidth}
\centering
\begin{tikzpicture}
\begin{axis}[
    width=\textwidth,
    height=4cm,
    xlabel={Number of time-series},
    ymin=0, ymax=1
]
% -----------------------------
%%% f1 node relaxed all parents
% -----------------------------
\addplot+[
    thick,
    mark=*, red, mark options={fill=red, draw=red}, solid,
    error bars/.cd,
        y dir=both,
        y explicit
] coordinates {
    (3,0.63) +- (0,0.25)
    (5,0.67) +- (0,0.14)
    (7,0.73) +- (0,0.23)
    (9,0.83) +- (0,0.18)
}; %tsLDiffPC

\addplot+[
    thick,
    mark=*, orange, mark options={fill=orange, draw=orange}, solid,
    error bars/.cd,
        y dir=both,
        y explicit
] coordinates {
    (3.1,0.67) +- (0,0.27)
    (5.1,0.74) +- (0,0.20)
    (7.1,0.79) +- (0,0.25)
    (9.1,0.9) +- (0,0.16)
}; %tsLDiffPC2

\addplot+[
    thick,
    mark=*, blue, mark options={fill=blue, draw=blue}, solid,
    error bars/.cd,
        y dir=both,
        y explicit
] coordinates {
    (3.2,0.83) +- (0,0.32)
    (5.2,0.71) +- (0,0.42)
    (7.2,0.79) +- (0,0.37)
    (9.2,0.9) +- (0,0.32)
}; %tsDCI

\addplot+[
    thick,
    mark=*, cyan, mark options={fill=cyan, draw=cyan}, solid,
    error bars/.cd,
        y dir=both,
        y explicit
] coordinates {
    (3.3,0.83) +- (0,0.32)
    (5.3,0.71) +- (0,0.42)
    (7.3,0.79) +- (0,0.37)
    (9.3,0.9) +- (0,0.32)
}; %tsDCIPC

\addplot+[
    thick,
    mark=*, olive, mark options={fill=olive, draw=olive}, solid,
    error bars/.cd,
        y dir=both,
        y explicit
] coordinates {
    (2.9,0.43) +- (0,0.47)
    (4.9,0.42) +- (0,0.41)
    (6.9,0.48) +- (0,0.45)
    (8.9,0.49) +- (0,0.39)
}; %tsMalik

\addplot+[
    thick,
    mark=*, yellow, mark options={fill=yellow, draw=yellow}, solid, 
    error bars/.cd,
        y dir=both,
        y explicit
] coordinates {
    (2.8,0.0) +- (0,0.0)
    (4.8,0.24) +- (0,0.42)
    (6.8,0.03) +- (0,0.09)
    (8.8,0.0) +- (0,0.0)
}; %tsiSCAN

\addplot+[
    thick,
    mark=*, teal, mark options={fill=teal, draw=teal}, solid,
    error bars/.cd,
        y dir=both,
        y explicit
] coordinates {
    (2.6,0.6) +- (0,0.21)
    (4.6,0.2) +- (0,0.32)
    (6.6,0.4) +- (0,0.34)
    (8.6,0.13) +- (0,0.28)
}; %microcause

\addplot+[
    thick,
    mark=*, pink, mark options={fill=pink, draw=pink}, solid,
    error bars/.cd,
        y dir=both,
        y explicit
] coordinates {
    (2.7,0.8) +- (0,0.42)
    (4.7,0.5) +- (0,0.53)
    (6.7,0.5) +- (0,0.53)
    (8.7,0.7) +- (0,0.48)
}; %rcd

\addplot+[
    thick,
    mark=*, gray, mark options={fill=gray, draw=gray}, solid,
    error bars/.cd,
        y dir=both,
        y explicit
] coordinates {
    (3.,0.6) +- (0,0.251)
    (5,0.41) +- (0,0.237)
    (7,0.25) +- (0,0.21)
    (9,0.27) +- (0,0.16)
};%pcunion

\end{axis}
\end{tikzpicture}
% \caption{Setting 1}
\end{subfigure}
\hfill
\begin{subfigure}{0.31\textwidth}
\centering
\begin{tikzpicture}
\begin{axis}[
    width=\textwidth,
    height=4cm,
    xlabel={Number of time-series},
    ymin=0, ymax=1,
legend style={
    at={(1.2,1.3)},
    anchor=south,
    legend columns=3,
    font=\scriptsize
}]

% -----------------------------
%%%% node relaxed, all parents, with minimum 2 parents
% -----------------------------
\addplot+[
    thick,
    mark=*, red, mark options={fill=red, draw=red}, solid,
    error bars/.cd,
        y dir=both,
        y explicit
] coordinates {
    (3,0.82) +- (0,0.24)
    (5,0.9) +- (0,0.16)
    (7,0.89) +- (0,0.25)
    (9,0.83) +- (0,0.32)
}; %tslindiffpc

\addplot+[
    thick,
    mark=*, orange, mark options={fill=orange, draw=orange}, solid,
    error bars/.cd,
        y dir=both,
        y explicit
] coordinates {
    (3.1,0.83) +- (0,0.22)
    (5.1,0.88) +- (0,0.32)
    (7.1,0.89) +- (0,0.25)
    (9.1,0.83) +- (0,0.32)
}; %tslindiffpc2

\addplot+[
    thick,
    mark=*, blue, mark options={fill=blue, draw=blue}, solid,
    error bars/.cd,
        y dir=both,
        y explicit
] coordinates {
    (3.2,0.63) +- (0,0.28)
    (5.2,0.87) +- (0,0.32)
    (7.2,0.69) +- (0,0.44)
    (9.2,0.77) +- (0,0.42)
}; %tsdci 

\addplot+[
    thick,
    mark=*, cyan, mark options={fill=cyan, draw=cyan}, solid,
    error bars/.cd,
        y dir=both,
        y explicit
] coordinates {
    (3.3,0.63) +- (0,0.28)
    (5.3,0.87) +- (0,0.32)
    (7.3,0.69) +- (0,0.44)
    (9.3,0.77) +- (0,0.42)
};% tsdcipc

\addplot+[
    thick,
    mark=*, olive, mark options={fill=olive, draw=olive}, solid,
    error bars/.cd,
        y dir=both,
        y explicit
] coordinates {
    (2.9,0.53) +- (0,0.39)
    (4.9,0.51) +- (0,0.42)
    (6.9,0.51) +- (0,0.35)
    (8.9,0.39) +- (0,0.37)
}; %malik ok

\addplot+[
    thick,
    mark=*, yellow, mark options={fill=yellow, draw=yellow}, solid, 
    error bars/.cd,
        y dir=both,
        y explicit
] coordinates {
    (2.8,0.13) +- (0,0.28)
    (4.8,0.0) +- (0,0.0)
    (6.8,0.03) +- (0,0.09)
    (8.8,0.1) +- (0,0.32)
};% iscan ok

\addplot+[
    thick,
    mark=*, teal, mark options={fill=teal, draw=teal}, solid,
    error bars/.cd,
        y dir=both,
        y explicit
] coordinates {
    (2.6,0.53) +- (0,0.28)
    (4.6,0.27) +- (0,0.34)
    (6.6,0.13) +- (0,0.28)
    (8.6,0.07) +- (0,0.22)
}; %microcause

\addplot+[
    thick,
    mark=*, pink, mark options={fill=pink, draw=pink}, solid,
    error bars/.cd,
        y dir=both,
        y explicit
] coordinates {
    (2.7,0.4) +- (0,0.52)
    (4.7,0.7) +- (0,0.48)
    (6.7,0.7) +- (0,0.48)
    (8.7,0.6) +- (0,0.52)
}; %rcd

\addplot+[
    thick,
    mark=*, gray, mark options={fill=gray, draw=gray}, solid,
    error bars/.cd,
        y dir=both,
        y explicit
] coordinates {
    (3.,0.467) +- (0,0.375)
    (5,0.43) +- (0,0.342)
    (7,0.29) +- (0,0.11)
    (9,0.3) +- (0,0.14)
};%pcunion

\end{axis}
\end{tikzpicture}
% \caption{Setting 4}
\end{subfigure}

\caption{Mean F1-score for nine methods as a function of the number of vertices across three settings.}
\label{fig:res_simulated_data}
\end{figure}
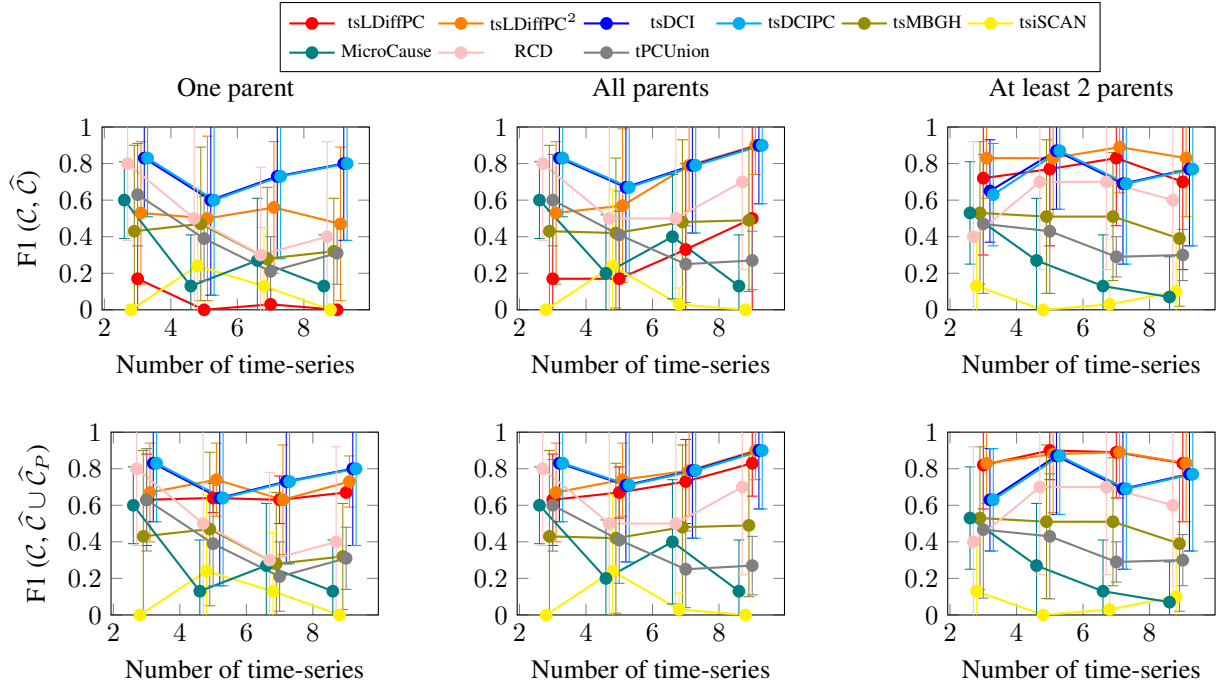

%% file: appendix.tex
\section{Details on window graphs}
\label{appendix:window}

Given causal stationarity (Assumption~\ref{assumption:stationarity}) and the 
existence of a maximum lag $\ell_{\max}$, the (potentially infinite) FT-DAG can be compressed into a finite representation called a \emph{window 
graph}~\cite{Assaad_2022}. For each regime, it suffices to represent only the variables 
between time $t$ and time $t - \ell_{\max}$, since the FT-DAG of the 
regime is obtained by replicating this window in a block-wise manner throughout 
the entire regime.

Figure~\ref{fig:window} shows the window graph corresponding to the full-time 
DAG of Figure~\ref{fig:full_time_regime_change}. It contains only time steps 
$t$ and $t-1$, since $\ell_{\max} = 1$ in this case. Here, the window graph 
is identical across both regimes as the FT-DAG is the same, but it may 
in general differ between regimes.
\begin{figure}[h]
    \centering
    \begin{tikzpicture}[
        scale=1,
        transform shape,
        node/.style={
            circle,
            draw,
            minimum size=0.75cm,
            inner sep=1pt,
            font=\small
        },
        edge/.style={->, thick},
        changededge/.style={->, thick, red},
        changededgeb/.style={->, thick, blue},
        dashedline/.style={dashed, gray},
        regime/.style={dashed, red, thick}
    ]
    %------------------------
    % Nodes
    %------------------------
    \foreach \x/\s/\lab in {
        2.8/m1/{t-1},
        5.6/zero/{t}
    }{
        \node[node] (X\s) at (\x,3) {$X_{\lab}$};
        \node[node] (Y\s) at (\x,1.5) {$Y_{\lab}$};
        \node[node] (Z\s) at (\x,0) {$Z_{\lab}$};
    }
    %------------------------
    % Autoregressive edges
    %------------------------
    \foreach \a/\b in {m1/zero}{
        \draw[edge] (X\a) -- (X\b);
        \draw[edge] (Y\a) -- (Y\b);
        \draw[edge] (Z\a) -- (Z\b);
    }
    %------------------------
    % Cross-lagged edges
    %------------------------
    \foreach \a/\b in {m1/zero}{
        \draw[edge] (Y\a) -- (Z\b);
    }
    %------------------------
    % Contemporaneous edges
    %------------------------
    \foreach \s in {m1, zero}{
        \draw[edge, left=25] (Z\s) to (Y\s);
    }
    \foreach \s in {m1, zero}{
        \draw[edge, left=25] (X\s) to (Y\s);
    }
    \end{tikzpicture}
    \caption{window graph corresponding to the FT-DAG of 
    Figure~\ref{fig:full_time_regime_change}, with $\ell_{\max} = 1$. 
    Only time steps $t-1$ and $t$ are represented. In this case, the window 
    graph is identical across both regimes.}
    \label{fig:window}
\end{figure}
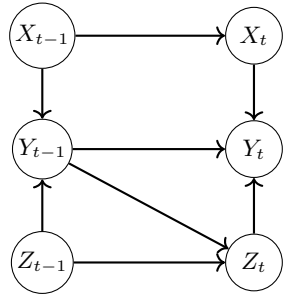

\section{Proofs}
\subsection{Diff-faithfulness vs diff-adjacency-faithfulness and diff-orientation-faithfulness}
\label{appendix:diff-faithfulnes}

This subsection assumes that the reader is familiar with the concepts introduced in~\cite{Bystrova_Arxiv_2026}.

% \begin{assumption}[Diff-faithfulness, \cite{Bystrova_Arxiv_2026}]
% Suppose  $\mathcal{G}^1$ and $\mathcal{G}^2$ the corresponding causal graphs and $\mathbb{P}^1$ and $\mathbb{P}^2$ the associate distributions. For every admissible ordered pair $(X_{t-\ell},Y_t)$ it is assumed that
% $ {\beta_{X_{t-\ell},Y_t\mid\mathbb{S}}}^{\mathcal{N}} ={\beta_{X_{t-\ell},Y_t\mid\mathbb{S}}}^{\mathcal{\bar N}} $
% if and only if 
% $\mathbb{S}$ \textit{diff}-separates\cite{Bystrova_Arxiv_2026} $X_{t-\ell}$ from $Y_t$.
% \end{assumption}

% \begin{assumption}[Symmetry of regression invariance~\cite{Bystrova_Arxiv_2026}]
% \label{assum:regression_invariance_symmetry}

% For every pair of distinct vertices $X_{t-\ell},Y_t$ and every
% $
% \mathbb Z\subseteq\mathbb V\setminus\{X_{t-\ell},Y_t\},
% $
% it is  assumed that
% ${\beta_{X_{t-\ell},Y_t\mid\mathbb Z}}^{\mathcal N}
% =
% {\beta_{X_{t-\ell},Y_t\mid\mathbb Z}}^{\bar{\mathcal N}}
% \quad\Longleftrightarrow\quad
% {\beta_{Y_t,X_{t-\ell}\mid\mathbb Z}}^{\mathcal N}
% =
% {\beta_{Y_t,X_{t-\ell}\mid\mathbb Z}}^{\bar{\mathcal N}}.
% $
% \end{assumption}

\begin{proposition}
\label{proposition:diff-faithfulness}
    Assumptions in \cite{Bystrova_Arxiv_2026} $\implies$ diff-adjacency-faithfulness $+$ diff-orientation-faithfulness but the reverse is not true.
\end{proposition}
\begin{proof}
%The forward implications are proven in \cite{Bystrova_Arxiv_2026}. In the following we prove that the reverse is not true. Consider an underlying graph
%$$ X \rightarrow Y \rightarrow Z, $$
%and suppose that the only edge in the difference graph is
%$$ X \rightarrow Y. $$
%Then diff-adjacency-faithfulness only imposes constraints on the adjacent pair $(X,Y)$, requiring the corresponding regression coefficients to differ across environments for every admissible conditioning set. Moreover, diff-orientation-faithfulness imposes no additional constraint, since the difference graph contains no unshielded triple. However, for the non-adjacent pair $(X,Z)$, it is possible that
%$$ \beta^{\mathcal N}_{X,Z} = \beta^{\bar{\mathcal N}}_{X,Z}, $$
%even though $\emptyset$ does not diff-separate $X$ and $Z$, because the active path
%$$ X \rightarrow Y \rightarrow Z $$
%contains the changed edge $X\rightarrow Y$. Hence full diff-faithfulness fails.
% \textcolor{red}{Therefore,
% $$
% \text{diff-adjacency-faithfulness}
% +
% \text{diff-orientation-faithfulness}
% \implies
% \text{diff-faithfulness}.
% $$}
% \textcolor{blue}{Timothée : C'est peut etre un peu bizarre d'avoir ''therefore...'' alors que ce n'est pas ce qu'on vient de prouver dans le paragraphe précédent ?}

The forward implications are proven in \cite{Bystrova_Arxiv_2026}. In the following, we prove that the reverse is not true, through two examples. First consider the underlying graphs in Figure  ~\ref{fig:diamond-difference-graphs}, where $a \neq a'$.

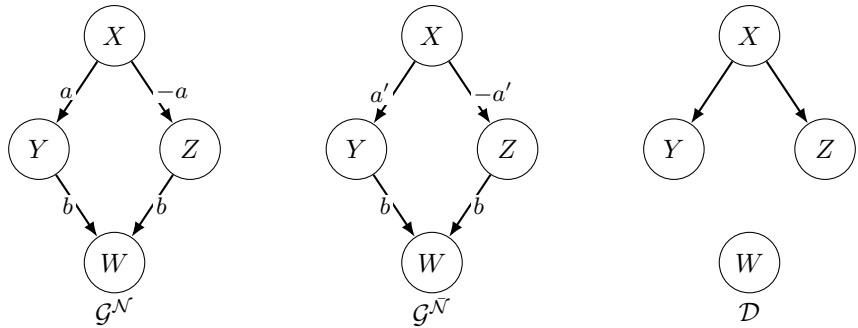
\begin{figure}[htbp]
    \centering
    \begin{tikzpicture}[
        node/.style={
            circle,
            draw,
            minimum size=8mm,
            inner sep=0pt
        },
        causal/.style={
            -{Latex[length=2mm]},
            thick
        },
        edge label/.style={
            font=\small,
            fill=white,
            inner sep=1pt
        }
    ]

    % =========================================================
    % Graph (a): G^1
    % =========================================================
    \begin{scope}[xshift=0cm]
        \node[node] (X1) at (0,3) {$X$};
        \node[node] (Y1) at (-1,1.5) {$Y$};
        \node[node] (Z1) at (1,1.5) {$Z$};
        \node[node] (W1) at (0,0) {$W$};

        \draw[causal]
            (X1) -- node[edge label, left] {$a$} (Y1);

        \draw[causal]
            (X1) -- node[edge label, right] {$-a$} (Z1);

        \draw[causal]
            (Y1) -- node[edge label, left] {$b$} (W1);

        \draw[causal]
            (Z1) -- node[edge label, right] {$b$} (W1);

        \node at (0,-0.65) {$\mathcal{G}^{\mathcal{N}}$};
    \end{scope}

    % =========================================================
    % Graph (b): G^2
    % =========================================================
    \begin{scope}[xshift=4.2cm]
        \node[node] (X2) at (0,3) {$X$};
        \node[node] (Y2) at (-1,1.5) {$Y$};
        \node[node] (Z2) at (1,1.5) {$Z$};
        \node[node] (W2) at (0,0) {$W$};

        \draw[causal]
            (X2) -- node[edge label, left] {$a'$} (Y2);

        \draw[causal]
            (X2) -- node[edge label, right] {$-a'$} (Z2);

        \draw[causal]
            (Y2) -- node[edge label, left] {$b$} (W2);

        \draw[causal]
            (Z2) -- node[edge label, right] {$b$} (W2);

        \node at (0,-0.65) {$\mathcal{G}^{\bar{\mathcal{N}}}$};
    \end{scope}

    % =========================================================
    % Graph (c): difference graph
    % =========================================================
    \begin{scope}[xshift=8.4cm]
        \node[node] (XD) at (0,3) {$X$};
        \node[node] (YD) at (-1,1.5) {$Y$};
        \node[node] (ZD) at (1,1.5) {$Z$};
        \node[node] (WD) at (0,0) {$W$};

        \draw[causal] (XD) -- (YD);
        \draw[causal] (XD) -- (ZD);
        %\draw[causal] (YD) -- (WD);
        %\draw[causal] (ZD) -- (WD);

        \node at (0,-0.65) {$\mathcal{D}$};
    \end{scope}

    \end{tikzpicture}

    \caption{The two causal graphs and the corresponding difference graph such that $b,a,a'>0$, and $a\ne a'$.}
    \label{fig:diamond-difference-graphs}
\end{figure}

Diff-adjacency faithfulness is satisfied by construction of the graph and the conditions on the coefficients. The difference graph contains an unshielded non-collider, and for both conditioning sets that do not contain $X$, meaning $\emptyset$ and $\{W\}$, the coefficients differ across regimes because $a^2 \neq (a')^2$. Thus diff-orientation-faithfulness is also satisfied. 

However, in both regimes, the two causal effects of $X$ on $W$ cancel. Consequently $\beta_{X,W}^{\mathcal{N}} = \beta_{X,W}^{\bar{\mathcal{N}}} = 0$. 
Nevertheless, the paths $X\rightarrow Y\rightarrow W$ and $X\rightarrow Z\rightarrow W$ are active given the empty set and both  contain an edge of the difference graph. Hence, the empty set does not diff-separate $X$ from $W$~\cite{Bystrova_Arxiv_2026}. Therefore, the assumption of  diff-faithfulness~\cite{Bystrova_Arxiv_2026} is violated although diff-adjacency-faithfulness and diff-orientation-faithfulness are satisfied.

Now let us consider a second example, 
$$ X \xrightarrow{a} Y \xrightarrow{b} Z \xrightarrow{c} W$$
for the normal regime,
$$ X \xrightarrow{c} Y \xrightarrow{b} Z \xrightarrow{a} W$$
for the anomalous regime and the resulting difference graph $ X \rightarrow Y$, $Z \rightarrow W$ where $a,b,c>0$ and $a\ne c$.
Diff-orientation-faithfulness is automatically satisfied, as the 2 edges of the difference graph are disjoint. Diff-adjacency-faithfulness is also satisfied, the coefficient associated to both changed edges are different across regimes for every admissible conditioning set.

But diff-faithfulness~\cite{Bystrova_Arxiv_2026} fails. As the total effect of X on W in the normal and anomalous regime is $\beta_{X,W}^{\mathcal{N}} = \beta_{X,W}^{\bar{\mathcal{N}}}  = abc$ but the empty set does not diff-separate X and W, as the path $X \rightarrow Y \rightarrow Z \rightarrow W$ is active and has two edges in the difference graph.

\end{proof}

\subsection{Proofs for tsLDiffPC}
In the following we present the proof of Theorem~\ref{theorem:correctness_tsLDiffPC}.

\begin{proof}[Theorem~\ref{theorem:correctness_tsLDiffPC}] We start by proving the skeleton construction phase then we proceed to proving the orientation phase.
\paragraph{Skeleton.}
The algorithm starts from the complete graph and removes an edge $X_{t- \ell} - Y_t$, for $\ell \geq 0$, if and only if it finds a set $\mathbb{S}$ such that 
$
\beta^{\mathcal{N}}_{X_{t- \ell},Y_t\mid \mathbb{S}}
=
\beta^{\bar{\mathcal{N}}}_{X_{t- \ell},Y_t\mid \mathbb{S}}.
$

\paragraph{(1) Every true difference edge is preserved.} Suppose that $ X_{t- \ell} - Y_t \in \mathrm{Skel}(\mathcal D).$ %By interventional faithfulness, every true mechanism change induces a detectable change in the conditional regression coefficient \cite{Bystrova_2024,Bystrova_2026}. 
If  $ X_{t- \ell} - Y_t \in \mathrm{Skel}(\mathcal D)$ then by  Assumption~\ref{assumption:diff_adjacency_faithfulness} for any subset $\mathbb{S}\subseteq \mathbb{V}\backslash\{X_{t-\ell}, Y_t\}$, $
\beta^{\mathcal{N}}_{X_{t- \ell},Y_t\mid \mathbb{S}}
\neq
\beta^{\bar{\mathcal{N}}}_{X_{t- \ell},Y_t\mid \mathbb{S}}$. Thus, given perfect conditional equality information for the pair $X_{t- \ell}$,$Y_t$  tsLDiffPC would not remove an edge $ X_{t- \ell} - Y_t$. Thus,
$
\mathrm{Skel}(\mathcal D)
\subseteq
\mathrm{Skel}(\widehat{\mathcal D}).
$

%Therefore, for every conditioning set $S$, we have
%$ \beta^{\mathcal{N}}_{X_{t- \ell},Y_t\mid S}
%\neq\beta^{\bar{\mathcal{N}}}_{X_{t- \ell},Y_t\mid S}.$So,
%$X_{t- \ell} - Y_t \in \mathrm{Skel}(\widehat{\mathcal D}).$
%Thus,
%$\mathrm{Skel}(\mathcal D)
%\subseteq
%\mathrm{Skel}%(\widehat{\mathcal D}).
%$

%by adj faithfulness
\paragraph{(2) Every non-edge is removed.}  Suppose that $X_{t- \ell} - Y_t \notin \mathrm{Skel}(\mathcal D)$, so $X_{t- \ell}$ and $Y_t$ are not adjacent in $\mathcal D$. Thus, by Assumption~\ref{assumption:diff_adjacency_faithfulness}, there exists such set $\mathbb{S}\subseteq \mathbb{V}\backslash\{X_{t-\ell}, Y_t\}$, such that $
\beta^{\mathcal{N}}_{X_{t- \ell},Y_t\mid \mathbb{S}}
=
\beta^{\bar{\mathcal{N}}}_{X_{t- \ell},Y_t\mid \mathbb{S}}$. If such set $\mathbb{S}$ does not exist, then again, under Assumption~\ref{assumption:diff_adjacency_faithfulness}, $X_{t- \ell}$ and $Y_t$ are adjacent in $\mathcal D$
which contradicts that 
$X_{t- \ell} - Y_t \notin \mathrm{Skel}(\mathcal D)$. Thus, as  $
\beta^{\mathcal{N}}_{X_{t- \ell},Y_t\mid \mathbb{S}}
=
\beta^{\bar{\mathcal{N}}}_{X_{t- \ell},Y_t\mid \mathbb{S}}$ and given perfect conditional equality information about the pair $X_{t- \ell}$,$Y_t$, tsLDiffPC would remove an edge  $X_{t- \ell} - Y_t$ in 
$\mathrm{Skel}(\widehat{\mathcal D})$. Thus, $
\mathrm{Skel}(\widehat{\mathcal D})
\subseteq
\mathrm{Skel}({\mathcal D}).
$

%\paragraph{(2) Every non-edge is removed.}  Second, suppose that $X_{t- \ell} - Y_t \notin \mathrm{Skel}(\mathcal D)$, so $X_{t- \ell}$ and $Y_t$ are not adjacent in $\mathcal D$. We have that if $X_{t- \ell}$ and $Y_t$ are not adjacent then \textcolor{orange}{adj faithfulness}
%\textcolor{red}{by causal markov condition $\beta^{\mathcal{N}}_{X_{t- \ell},Y_t\mid S} = \beta^{\bar{\mathcal{N}}}_{X_{t- \ell},Y_t\mid S}$ for $S \subseteq Parents(X_{t- \ell}, \mathcal D)$ or $S \subseteq Parents(Y_t, \mathcal D)$}.

%Since every edge in the true graph is in the output of the algorithm  (as shown in (1)), then $Parents(X_{t- \ell}, \mathcal D) \subseteq \mathrm{Adj}_{\widehat{\mathcal D}}(X_{t- \ell})$ and $Parents(Y_t, \mathcal D) \subseteq \mathrm{Adj}_{\widehat{\mathcal D}}(Y_t)$. \textcolor{red}{Hence, $\beta^{\mathcal{N}}_{X_{t- \ell},Y_t\mid S} = \beta^{\bar{\mathcal{N}}}_{X_{t- \ell},Y_t\mid S}$ such that $ S \subseteq \mathrm{Adj}_{\widehat{\mathcal D}}(X_{t- \ell})$ or $ S \subseteq \mathrm{Adj}_{\widehat{\mathcal D}}(Y_t)$}.
%\textcolor{blue}{Then, there exists $S\subseteq \mathrm{Adj}_{\widehat{\mathcal D}}(X_{t- \ell})\cup \mathrm{Adj}_{\widehat{\mathcal D}}(Y_t)$ such that $\beta^{\mathcal{N}}_{X_{t- \ell},Y_t\mid S} = \beta^{\bar{\mathcal{N}}}_{X_{t- \ell},Y_t\mid S}$.} Hence $X_{t- \ell}$ and $Y_t$ are not adjacent in  $\widehat{\mathcal D}$, i.e. $ X_{t- \ell} - Y_t \notin \mathrm{Skel}(\widehat{\mathcal D})$, and $ \mathrm{Skel}(\widehat{\mathcal D}) \subseteq \mathrm{Skel}(\mathcal D)$.

\paragraph{Orientations.}  Let us now prove the correctness of the orientations.
For the edges $X_{t - \ell} - Y_t$ where $\ell > 0$, the orientation is defined using temporal order $X_{t - \ell} \rightarrow Y_t$. 

% For the edges $X_{t - \ell} - Y_t$ where $\ell > 0$, the orientation is defined using temporal order $X_{t - \ell} \rightarrow Y_t$ . 
% Thus,  all the V-structures $X_{t - \ell} \rightarrow Y_t \leftarrow X_{t - \ell'}$, where $\ell, \ell' > 0$, are  correclty oriented in $\widehat{\mathcal D}$ using temporal order. 

Next, we prove that tsLDiffPC correctly orients all unshielded colliders. Let us consider the following triple  $X_{t - \ell} \rightarrow Z_t \leftarrow Y_{t}$ with $\ell \geq 0$ in $\mathcal{D}$, where $X_{t - \ell}$ and $Y_t$ are not adjacent. As the skeleton of $\widehat{\mathcal{D}}$ coincides with the skeleton of  ${\mathcal{D}}$, then  $X_{t - \ell}$ and $Y_t$ are not adjacent, and pairs $X_{t - \ell}$, $Z_t$ and $Y_t$, $Z_t$ are  adjacent in $\widehat{\mathcal{D}}$. As $X_{t - \ell}$ and $Y_t$ are not adjacent in  ${\mathcal{D}}$, there exists set $\mathbb{S}$ such that $\beta^{\mathcal{N}}_{X_{t- \ell},Y_t\mid \mathbb{S}} = \beta^{\bar{\mathcal{N}}}_{X_{t- \ell},Y_t\mid \mathbb{S}}$. By Assumption~\ref{assumption:diff_orient_faithfulness} the set $\mathbb{S}$ does not contain $Z_t$. Thus,  $X_{t - \ell} \rightarrow Z_t \leftarrow Y_{t}$ is oriented as an unshielded collider in $\widehat{\mathcal{D}}$ by tsLDiffPC.
 
Once all unshielded colliders and all edges that can be oriented using temporal order have been oriented, the Meek orientation rules\cite{Meek1995CausalIA} are applied, as in LDiffPC. The correctness of the application of Meek rules in $\widehat{\mathcal{D}}$  follows from Lemma 3 in \cite{Bystrova_Arxiv_2026}. 

%The correctness of the orientation of the meek rules in the temporal setting are proved in \textcolor{red}{quel citation ?}

%\textcolor{red}{Proof of orientations}
\end{proof}

\subsection{Proofs for tsDCI}

The proof follows the same structure as the proof of Theorem~4.4  in~\cite{Wang_Neurips_2018}, and relies on two lemmas which we now state and prove. The first lemma characterizes the nodes whose internal noise variance is invariant across regimes.

\begin{lemma} [temporal analogue of C.6 ~\cite{Wang_Neurips_2018}]
\label{lemma:temp_C6}
For any node $Y_t$ incident to at least one edge in $\mathcal D$,
\[
\sigma^{\mathcal N}_{Y_t}
=
\sigma^{\bar{\mathcal N}}_{Y_t}
\quad\Longleftrightarrow\quad
\exists \mathbb{S}\subseteq\mathbb V\setminus\{Y_t\}
\text{ such that }
\sigma^{\mathcal N}_{Y_t\mid \mathbb{S}}
=
\sigma^{\bar{\mathcal N}}_{Y_t\mid \mathbb{S}}.
\]
\end{lemma}

\begin{proof}
    Proving the ''$\Leftarrow$'' direction comes directly from the Var-Orientation Faithfulness assumption (Assumption~\ref{assumption:diff_var_orient_faithfulness}). 
    % Proving the ''$\Rightarrow$'' suppose $\sigma^{\mathcal N}_{Y_t} = \sigma^{\bar{\mathcal N}}_{Y_t}.$ Let $\mathbb{S} = Pa(Y_t, \mathcal{N}) \cup Pa(Y_t,\mathcal{\bar{N}})$.
    % Under the Assumption~\ref{assumption:stationarity} since  $X_{t-\ell} \in \mathbb V$ for every $0\leq\ell \leq \ell_{max}$, we have $\mathbb{S} \subseteq \mathbb V \setminus\{Y_t\}$. And so markov  property of the DT-DSCM gives :
    % $$
    % Y_t
    % =
    % \sum_{X_{t-\ell}\in Pa(Y_t,\mathcal{N})}
    % \alpha^{\mathcal{N}}_{{X_{t-\ell}},Y_t} X_{t-\ell}
    % +
    % \varepsilon^{\mathcal{N}}_{Y_t},
    % $$
    % The residuals of the regression of $Y_t$ on $X_{\mathbb{S}}$ gives $\varepsilon^{\mathcal{N}}_{Y_t}$,  as $Pa(Y_t,\mathcal{N}) \subseteq \mathbb{S}$. And so $(\sigma_{Y_t | \mathbb{S}}^{\mathcal{N}})^2 = \mathrm{Var}(\varepsilon^{\mathcal{N}}_{Y_t}) = (\sigma_{Y_t}^{\mathcal{N}})^2$. The same reasoning for $\mathcal{\bar{N}}$. So we have that $\sigma_{Y_t | \mathbb{S}}^{\mathcal{N}} = \sigma_{Y_t | \mathbb{S}}^{\mathcal{\bar{N}}}$. 

To prove the ''$\Rightarrow$'', suppose $\sigma^{\mathcal{N}}_{Y_t} = \sigma^{\bar{\mathcal{N}}}_{Y_t}$
and let $\mathbb{S} = Pa(Y_t, \mathcal{N}) \cup Pa(Y_t, \mathcal{\bar{N}})$.
By Assumption~\ref{assumption:stationarity}, $\mathbb{S} \subseteq \mathbb{V} \setminus \{Y_t\}$.
By definition of the DT-DSCM:
\begin{align*}
    Y_t^{\mathcal{N}}
    &= \sum_{X_{t-\ell}\in Pa(Y_t,\mathcal{N})} \alpha^{\mathcal{N}}_{X_{t-\ell},Y_t} X_{t-\ell}
    + \varepsilon^{\mathcal{N}}_{Y_t} \\
    &= \sum_{X_{t-\ell}\in Pa(Y_t,\mathcal{N})} \alpha^{\mathcal{N}}_{X_{t-\ell},Y_t} X_{t-\ell}
    + \sum_{X_{t-\ell}\in \mathbb{S} \setminus Pa(Y_t,\mathcal{N})}
    \underbrace{\alpha^{\mathcal{N}}_{X_{t-\ell},Y_t}}_{=0} X_{t-\ell}
    + \varepsilon^{\mathcal{N}}_{Y_t} \\
    &= \sum_{X_{t-\ell}\in \mathbb{S}} \alpha^{\mathcal{N}}_{X_{t-\ell},Y_t} X_{t-\ell}
    + \varepsilon^{\mathcal{N}}_{Y_t},
\end{align*}
so the residuals of the regression of $Y_t^{\mathcal{N}}$ on $\mathbb{S}$ are $\varepsilon^{\mathcal{N}}_{Y_t}$,
and $\sigma_{Y_t \mid \mathbb{S}}^{\mathcal{N}} = \sqrt{\mathrm{Var}(\varepsilon^{\mathcal{N}}_{Y_t})}
= \sigma_{Y_t}^{\mathcal{N}}$. The same holds for $\mathcal{\bar{N}}$:
\begin{align*}
    Y_t^{\mathcal{\bar{N}}}
    &= \sum_{X_{t-\ell}\in Pa(Y_t,\mathcal{\bar{N}})} \alpha^{\mathcal{\bar{N}}}_{X_{t-\ell},Y_t} X_{t-\ell}
    + \varepsilon^{\mathcal{\bar{N}}}_{Y_t} \\
    &= \sum_{X_{t-\ell}\in Pa(Y_t,\mathcal{\bar{N}})} \alpha^{\mathcal{\bar{N}}}_{X_{t-\ell},Y_t} X_{t-\ell}
    + \sum_{X_{t-\ell}\in \mathbb{S} \setminus Pa(Y_t,\mathcal{\bar{N}})}
    \underbrace{\alpha^{\mathcal{\bar{N}}}_{X_{t-\ell},Y_t}}_{=0} X_{t-\ell}
    + \varepsilon^{\mathcal{\bar{N}}}_{Y_t} \\
    &= \sum_{X_{t-\ell}\in \mathbb{S}} \alpha^{\mathcal{\bar{N}}}_{X_{t-\ell},Y_t} X_{t-\ell}
    + \varepsilon^{\mathcal{\bar{N}}}_{Y_t},
\end{align*}
so $\sigma_{Y_t \mid \mathbb{S}}^{\mathcal{\bar{N}}} = \sqrt{\mathrm{Var}(\varepsilon^{\mathcal{\bar{N}}}_{Y_t})}
= \sigma_{Y_t}^{\mathcal{\bar{N}}}$. Since $\sigma^{\mathcal{N}}_{Y_t} = \sigma^{\bar{\mathcal{N}}}_{Y_t}$
by assumption, we conclude $\sigma_{Y_t \mid \mathbb{S}}^{\mathcal{N}} = \sigma_{Y_t \mid \mathbb{S}}^{\mathcal{\bar{N}}}$.
\end{proof}

The second lemma shows that the conditioning set $\mathbb{S}$ witnessing the variance invariance of $Y_t$ correctly encodes the edge orientations adjacent to $Y_t$ in $\widehat{\mathcal{D}}$.

\begin{lemma}[temporal analogue of C.7 ~\cite{Wang_Neurips_2018}]
\label{lemma:temp_C7}
    For every instantaneous edge $X_t - Y_t \in \mathcal{D}$ such that
    $
    \sigma^{\mathcal N}_{Y_t}\
    =
    \sigma^{\bar{\mathcal N}}_{Y_t}.
    $ it holds that,
    \begin{itemize}
        \item if $X_t \rightarrow Y_t \in \mathcal{D}$, then $X_t \in \mathbb{S}$ for all $\mathbb{S}$ s.t. 
        $
        \sigma^{\mathcal N}_{Y_t\mid \mathbb{S}}
        =
        \sigma^{\bar{\mathcal N}}_{Y_t\mid \mathbb{S}}.
        $
        \item if $Y_t \rightarrow X_t \in \mathcal{D}$, then $X_t \notin \mathbb{S}$ for all $\mathbb{S}$ s.t. 
        $
        \sigma^{\mathcal N}_{Y_t\mid \mathbb{S}}
        =
        \sigma^{\bar{\mathcal N}}_{Y_t\mid \mathbb{S}}.
        $
    \end{itemize}
\end{lemma}
\begin{proof}
    We prove both items by contradiction. 
    \paragraph{Item 1:} Suppose $X_t \rightarrow Y_t \in \mathcal{D}$, then $\alpha^{\mathcal{N}}_{X_t,Y_t} \neq \alpha^{\mathcal{\bar{N}}}_{X_t,Y_t}$. Suppose further that there exists a $\mathbb{S}$ with $X_t \notin \mathbb{S}$ such that $
        \sigma^{\mathcal N}_{Y_t\mid \mathbb{S}}
        =
        \sigma^{\bar{\mathcal N}}_{Y_t\mid \mathbb{S}}
        $. This directly contradicts the first part of Assumption~\ref{assumption:diff_var_orient_faithfulness}. 

     \paragraph{Item 2:} Suppose $Y_t \rightarrow X_t \in \mathcal{D}$, then $\alpha^{\mathcal{N}}_{Y_t,X_t} \neq \alpha^{\mathcal{\bar{N}}}_{Y_t,X_t}$. 
     %Let us suppose that there exists a set $\mathbb{S}$ with $X_t \in \mathbb{S}$ such that $\sigma^{\mathcal N}_{Y_t\mid \mathbb{S}}=\sigma^{\bar{\mathcal N}}_{Y_t\mid \mathbb{S}}$. 
     Since $\alpha^{\mathcal{N}}_{Y_t,X_t} \neq \alpha^{\mathcal{\bar{N}}}_{Y_t,X_t}$ then under the Assumption~\ref{assumption:diff_var_orient_faithfulness} for $\mathbb{S}' = \mathbb{S} \setminus \{ X_t \}$, $\mathbb{S}' \subseteq \mathbb{V}$,  $\sigma^{\mathcal N}_{Y_t\mid \mathbb{S}' \cup \{ X_t \}} \neq \sigma^{\bar{\mathcal N}}_{Y_t\mid \mathbb{S}' \cup \{ X_t \}}$, which contradicts the assumed equality.

     %This contradicts the first part of Assumption ~\ref{assumption:diff_var_orient_faithfulness}, applied with $\mathbb{S}' = \mathbb{S} \setminus \{ X_t \}$ which gives $\sigma^{\mathcal N}_{Y_t\mid \mathbb{S}' \cup \{ X_t \}}=\sigma^{\bar{\mathcal N}}_{Y_t\mid \mathbb{S}' \cup \{ X_t \}}$
\end{proof}

\begin{proof}[Theorem~\ref{theorem:correctness_tsdci}]
    The correctness of the skeleton construction follows directly from  Theorem~\ref{theorem:correctness_tsLDiffPC}.
    The orientation step of the tsDCI proceeds in  two stages. 
    First, for all the edges $X_{t - \ell} - Y_t$ where $\ell > 0$, the orientation is defined using temporal order $X_{t - \ell} \rightarrow Y_t$.
    Once all the lagged edges have been oriented, the remaining contemporaneous edges $X_t -Y_t$ are oriented  using the variance equality criterion introduced in DCI. We prove that these orientations are correct.

    By lemma ~\ref{lemma:temp_C6}, there exists $\mathbb{S}$ such that $\sigma^{\mathcal N}_{Y_t\mid \mathbb{S}}=\sigma^{\bar{\mathcal N}}_{Y_t\mid \mathbb{S}}$ if and only if $\sigma^{\mathcal N}_{Y_t}=\sigma^{\bar{\mathcal N}}_{Y_t}$. Therefore, all the nodes where the internal noise variance is unchanged will be choosen by tsDCI.
    In addition, it also follows from lemma ~\ref{lemma:temp_C7} that for any $X_t \rightarrow Y_t \in \mathcal{D}$, $X_t \in \mathbb{S}$ and $Y_t \rightarrow X_t \in \mathcal{D}$, $X_t \notin \mathbb{S}$.
    Consequently, whenever the internal noise variance of a node $X_t$ is invariant across regimes, tsDCI correctly orients all edges adjacent to $X_t$. It remains to show that all edges oriented in the last step are correct. This easely follows from the acyclic property of the underlying graphs.
\end{proof}

\subsection{Proofs for tsLDiffPC$^2$ and tsDCIPC}
\begin{proof} (Corrolary ~\ref{cor:pc_augmented})
    By the soundness of tsLDiffPC or tsDCI, the skeleton of the estimated difference graph is correct and every previously oriented edge is correctly oriented.
    It remains to prove that the additional orientations obtained from tPC are sound. Consider an edge $X_t-Y_t$ that remains undirected in $\widehat{\mathcal D}$ and suppose that tPC orients it as $X_t\to Y_t$ in $\widehat{\mathcal G}^{\mathcal N}$. By the correctness of tPC, this implies that $X_t$ precedes $Y_t$ in the causal order of $\mathcal G^{\mathcal N}$. Since the two regimes share the same topological ordering by Assumption~\ref{assumption:order}, the reverse orientation $Y_t\to X_t$ is not admissible in either regime. Hence it is also impossible in the true difference graph $\mathcal D$.
    Moreover, since the skeleton of $\widehat{\mathcal D}$ is correct, the edge $X_t-Y_t$ corresponds to a true changed edge in $\mathcal D$. Therefore, the only admissible orientation is $X_t\to Y_t$. Thus, all additional orientations introduced by tsLDiffPC$^2$ or tsDCIPC are sound, and the result follows.
\end{proof}

\subsection{Proof for effect-defying root cause detection}

%By soundness of tsLDiffPC (resp. tsLDiffPC$^2$, tsDCI, tsDCIPC), the estimated temporal difference graph has a correct skeleton, and every oriented edge returned by the algorithm is correctly oriented. A variable $X$ is detected as a root cause only if the estimated difference graph contains a directed changed edge into $X$. So this edge also belongs to the true temporal difference graph with the same orientation. Hence $X$ is an effect-defying root cause. Therefore every detected root cause is a true root cause, and $\widehat{\mathcal C} \subseteq \mathcal C .$

%For every potential root cause $(X,Y) \in \widehat{\mathcal{C}}_P$, either $X \rightarrow Y$ or $Y \rightarrow X$. So exactly one of the two is the root cause associated to this edge. 

\begin{proof} (Corrolary ~\ref{cor:recovery})
By the correctness guarantees of $\mathcal{A}$ under the assumptions of
Table~\ref{tab:assumptions_algorithms}, we have
$
\mathrm{Skel}(\widehat{\mathcal D})=\mathrm{Skel}(\mathcal D)
\quad\text{and}\quad
\widehat{\mathbb E}\subseteq \mathbb E .
$
Hence every directed edge returned by $\mathcal A$ is a true changed edge with the correct orientation.

Let $X\in\widehat{\mathcal C}$. Then there exists $Y$ such that $Y\to X$ is in the output of $\mathcal A$. Since this edge is correctly oriented, $Y\to X$ is also in the true difference graph $\mathcal D$. Therefore $X$ is the target of a changed causal mechanism, so $X\in\mathcal C$. Thus,
$\widehat{\mathcal C}\subseteq\mathcal C .$

Now let $(X,Y)\in\widehat{\mathcal C}_P$. Then $X-Y$ is an undirected edge in
the output. Since the skeleton is correct, this edge corresponds to a true
changed edge in $\mathcal D$. Therefore, in the true difference graph, either
$X\to Y$ or $Y\to X$. Hence either $Y$ or $X$ is the target of a changed
causal mechanism, and so
$
[X\in\mathcal C]\lor[Y\in\mathcal C].
$

Finally, because the skeleton is fully recovered, every true changed edge is
either returned as a directed edge or as an undirected edge. For each
undirected edge in $\widehat{\mathcal C}_P$, choose the endpoint that is the
target of the corresponding true directed edge in $\mathcal D$. Let
$\widehat{\mathcal C}_S$ be the set of selected endpoints. Then all root causes
not already contained in $\widehat{\mathcal C}$ are selected in
$\widehat{\mathcal C}_S$, while no non-root cause is selected. Therefore,
$
\widehat{\mathcal C}\cup\widehat{\mathcal C}_S=\mathcal C .
$
\end{proof}

\section{Pseudocodes}
\label{appendix:pseudocode}

This section presents the pseudocode of tsLDiffPC (Algorithm~\ref{alg:tsLDiffPC}), tsLDiffPC$^2$ (Algorithm~\ref{alg:tsLDiffPC2}), tsDCI (Algorithm~\ref{alg:tsdci}),  tsDCIPC (Algorithm~\ref{alg:tsdci2}) and tPCUnion (Algorithm~\ref{alg:tspc-union}).

In Algorithm~\ref{alg:tspc-union}, $\operatorname{PossPa}_{\hat{\mathcal G}^{U}}(Y_t)$ denotes the set of possible parents of $Y_t$ in the union graph, including both nodes with an incoming directed edge into $Y_t$ and nodes connected to $Y_t$ by an undirected edge.

\begin{algorithm}[H]
\caption{tsLDiffPC algorithm}
\label{alg:tsLDiffPC}
\begin{algorithmic}[1]
% \Function{tsLDiffPC}{$\mathbf{P}^{\mathcal{N}},\mathbf{P}^{\bar{\mathcal{N}}},\mathbb{V}$}

\State \textbf{Input} $\mathbf{P}^{\mathcal{N}},\mathbf{P}^{\bar{\mathcal{N}}},\mathbb{V}$

\State Initialize a fully connected undirected graph $\mathcal{\hat D}$ on $\mathbb{V}$.
\State Orient every lagged edge by temporal order:
if $X_{t-\ell}-Y_t$ with $\ell>0$, orient
$X_{t-\ell}\rightarrow Y_t$.
%\Repeat
 %   \For{\textcolor{orange}{each adjacent pair $X_{t-\ell},Y_t$ in $\mathcal{\hat D}$}}
  %      \For{each $\mathbb{S}\subseteq \mathbb{V}\setminus\{X_{t-\ell},Y_t\}$ such that $\mathbb{S} \notin Descendant(Y_t)$}
   %         \If{$\beta^{\mathcal{N}}_{X_{t-\ell},Y_t\mid \mathbb{S}}=\beta^{\bar{\mathcal{N}}}_{X_{t-\ell},Y_t\mid \mathbb{S}}$}
    %            \State Remove $X_{t-\ell}-Y_t$ from $\mathcal{\hat D}$.
    %            \State By stationarity remove all $X_{t-\ell-i}-Y_{t-i}$ from $\mathcal{\hat D}$
    %            \State $\mathbb{S}=Sep(X_{t-\ell},Y_t)=Sep(Y_t,X_{t-\ell})$.
    %            \State \textbf{break}
    %        \EndIf
    %    \EndFor
    %\EndFor
%\Until

\State $s\gets 0$

\While{$s\leq |\mathbb{V}|$}
    \For{each $Y_t$ in $\mathbb{V}$}
        \For{each $X_{t-\ell}\in Adj_{\hat{\mathcal{D}}}(Y_t)$}
            \State $\mathcal{P}\gets \mathbb{V}\setminus\{X_{t-\ell},Y_t\}$
            \For{each $\mathbb{S}\subseteq\mathcal{P}$ such that $|\mathbb{S}|=s$}
                \If{$\beta^{\mathcal{N}}_{X_{t-\ell},Y_t\mid \mathbb{S}}=\beta^{\bar{\mathcal{N}}}_{X_{t-\ell},Y_t\mid \mathbb{S}}$}
                    \State Remove $X_{t-\ell}-Y_t$ from $\mathcal{\hat D}$.
                    \State By stationarity remove all $X_{t-\ell-i}-Y_{t-i}$ from $\mathcal{\hat D}$
                    \State $\mathbb{S}=Sep(X_{t-\ell},Y_t)=Sep(Y_t,X_{t-\ell})$.
                    \State \textbf{break}
                \EndIf
            \EndFor
        \EndFor
    \EndFor
    \State $s\gets s+1$
\EndWhile

\State For each triple $X_{t-\ell}-Y_t-Z_t$,
if $X_{t-\ell}$ and $Z_t$ are nonadjacent and $Y_t\notin Sep(X_{t-\ell},Z_t)$ or $Y_t\notin Sep(Z_t,X_{t-\ell})$, orient $X_{t-\ell}\rightarrow Y_t\leftarrow Z_t$.

\State By stationarity orient all $X_{t-\ell-i}\rightarrow Y_{t-i} \leftarrow Z_{t-i}$

\Repeat
    \State Apply Meek rules, only when the resulting orientation respects temporal order.
\Until{no additional edge can be oriented}

\State \Return $\mathcal{\hat D}$.
% \EndFunction
\end{algorithmic}
\end{algorithm}

\begin{algorithm}[H]
\caption{tsDCI algorithm}
\label{alg:tsdci}
\begin{algorithmic}[1]
% \Function{tsDCI}{$\mathbf{P}^{(1)},\mathbf{P}^{(2)},\mathbb{V}$}
\State \textbf{Input} $\mathbf{P}^{\mathcal{N}},\mathbf{P}^{\bar{\mathcal{N}}},\mathbb{V}$
\State Initialize a fully connected undirected graph $\mathcal{\hat D}$ on $\mathbb{V}$.
\State Orient every lagged edge by temporal order:
if $X_{t-\ell}-Y_t$ with $\ell>0$, orient
$X_{t-\ell}\rightarrow Y_t$.

\For{each adjacent pair $X_{t-\ell},Y_t$ in $\mathcal{\hat D}$}
    \For{each $\mathbb{S}\subseteq \mathbb{V}\setminus\{X_{t-\ell},Y_t\}$ such that $\mathbb{S} \notin Descendant(Y_t)$}

        \If{$\beta^{\mathcal{N}}_{X_{t-\ell},Y_t\mid \mathbb{S}}=\beta^{\bar{\mathcal{N}}}_{X_{t-\ell},Y_t\mid \mathbb{S}}$ \textbf{or} $\beta^{\mathcal{N}}_{Y_t, X_{t-\ell}\mid \mathbb{S}}=\beta^{\bar{\mathcal{N}}}_{Y_t, X_{t-\ell}\mid \mathbb{S}}$}
            \State Remove $X_{t-\ell}-Y_t$ from $\mathcal{\hat D}$.
            \State By stationarity remove all $X_{t-\ell-i}-Y_{t-i}$ from $\mathcal{\hat D}$
            \State $\mathbb{S}=Sep(X_{t-\ell},Y_t)=Sep(Y_t,X_{t-\ell})$.
            \State \textbf{break}
        \EndIf
    \EndFor
\EndFor

\For{each remaining undirected edge $X_t-Y_t$}
    \For{ $\mathbb{S}\subseteq \mathbb{V}\setminus\{Y_t\}$ such that $\mathbb{S} \notin Descendant(Y_t)$}

        \If{$\sigma^{\mathcal{N}}_{Y_t\mid \mathbb{S}}=\sigma^{\bar{\mathcal{N}}}_{Y_t\mid \mathbb{S}}$}
            \If{$X_t\in \mathbb{S}$}
                \State Orient $X_t\rightarrow Y_t$.
                \State By stationarity orient all $X_{t-i}\rightarrow Y_{t-i}$
            \Else
                \State Orient $Y_t\rightarrow X_t$.
                \State By stationarity orient all $Y_{t-i}\rightarrow X_{t-i}$
            \EndIf
            \State \textbf{break}
        \EndIf

    \EndFor
\EndFor
\State Orient remaining undirected edges by graph traversal:
orient $X_{t-\ell}-Y_t$ as $X_{t-\ell} \rightarrow Y_t$ whenever there exists a directed path
$X_{t-\ell} \rightarrow Z_1 \rightarrow \cdots \rightarrow Z_m \rightarrow Y_t$.

\State \Return $\mathcal{\hat D}$.
% \EndFunction
\end{algorithmic}
\end{algorithm}

\begin{algorithm}[H]
\caption{tsLDiffPC$^2$ algorithm}
\label{alg:tsLDiffPC2}
\begin{algorithmic}[1]

\State \textbf{Input} $\mathbf{P}^{\mathcal{N}},\mathbf{P}^{\bar{\mathcal{N}}},\mathbb{V}$

\State Run tsLDiffPC.
\State Run tPC on $\mathbf{P}^{\mathcal{N}}$.
\For{each remaining undirected edge $X_t-Y_t$ in $\mathcal{\hat D}$} 
    \If{tPC orients $X_t\rightarrow Y_t$}
        \State Orient $X_t-Y_t$ as $X_t\rightarrow Y_t$ in $\mathcal{\hat D}$.
    \EndIf
\EndFor

\State \Return $\mathcal{\hat D}$.
% \EndFunction
\end{algorithmic}
\end{algorithm}

\begin{algorithm}[H]
\caption{tsDCIPC algorithm}
\label{alg:tsdci2}
\begin{algorithmic}[1]

\State \textbf{Input} $\mathbf{P}^{\mathcal{N}},\mathbf{P}^{\bar{\mathcal{N}}},\mathbb{V}$

\State Run tsDCI.
\State Run tPC on $\mathbf{P}^{\mathcal{N}}$.
\For{each remaining undirected edge $X_t-Y_t$ in $\mathcal{\hat D}$} 
    \If{tPC orients $X_t\rightarrow Y_t$}
        \State Orient $X_t-Y_t$ as $X_t\rightarrow Y_t$ in $\mathcal{\hat D}$.
    \EndIf
\EndFor

\State \Return $\mathcal{\hat D}$.
% \EndFunction
\end{algorithmic}
\end{algorithm}

\begin{algorithm}[H]
\caption{tPCUnion algorithm}
\label{alg:tspc-union}
\begin{algorithmic}[1]

\State \textbf{Input} $\mathbf{P}^{\mathcal{N}},
\mathbf{P}^{\bar{\mathcal{N}}},\mathbb{V}$

\State Estimate separately the temporal graphs $\hat{\mathcal G}^{\mathcal N}$ and
$\hat{\mathcal G}^{\bar{\mathcal N}}$ using tPC.

\State Initialize an empty union graph $\hat{\mathcal G}^{U}$ and an empty difference graph $\hat{\mathcal D}$.

\For{each pair $X_{t-\ell},Y_t$ adjacent in $\hat{\mathcal G}^{\mathcal N}$ or
$\hat{\mathcal G}^{\bar{\mathcal N}}$}
    \If{$X_{t-\ell}$ and $Y_t$ are adjacent in only one regime}
        \If{the orientation is $X_{t-\ell}\rightarrow Y_t$}
            \State Add $X_{t-\ell}\rightarrow Y_t$ to $\hat{\mathcal G}^{U}$ and $\hat{\mathcal D}$.
        \Else
            \State Add $X_t-Y_t$ to
            $\hat{\mathcal G}^{U}$ and $\hat{\mathcal D}$.
        \EndIf
    \Else
        \If{both regimes imply opposite orientations}
            \State Record an orientation conflict between $X_t$ and $Y_t$.
        \ElsIf{at least one regime orients $X_{t-\ell}\rightarrow Y_t$}
            \State Add $X_{t-\ell}\rightarrow Y_t$ to $\hat{\mathcal G}^{U}$.
        \Else
            \State Add $X_t-Y_t$ to $\hat{\mathcal G}^{U}$.
        \EndIf
    \EndIf
\EndFor

\For{each directed edge
$X_{t-\ell}\rightarrow Y_t$
in $\hat{\mathcal G}^{U}$ not already in $\hat{\mathcal D}$}

    \If{$Y_t$ is involved in an orientation conflict}
        \State \textbf{continue}
    \EndIf

    \State
    $\mathbb{S}\gets
    PossPa_{\hat{\mathcal G}^{U}}(Y_t)
    \setminus\{X_{t-\ell}\}$.

    \If{
    $\beta^{\mathcal N}_{X_{t-\ell},Y_t\mid\mathbb{S}}
    \neq
    \beta^{\bar{\mathcal N}}_{X_{t-\ell},Y_t\mid\mathbb{S}}$
    }
        \State Add $X_{t-\ell}\rightarrow Y_t$ to $\hat{\mathcal D}$.
    \EndIf
\EndFor

\For{each undirected edge $X_t-Y_t$
in $\hat{\mathcal G}^{U}$ not already in $\hat{\mathcal D}$}

    \State
    $\mathbb{S}_Y\gets
    PossPa_{\hat{\mathcal G}^{U}}(Y_t)\setminus\{X_t\}$.

    \State
    $\mathbb{S}_X\gets
    PossPa_{\hat{\mathcal G}^{U}}(X_t)\setminus\{Y_t\}$.

    \If{
    $\beta^{\mathcal N}_{X_t,Y_t\mid\mathbb{S}_Y}
    \neq
    \beta^{\bar{\mathcal N}}_{X_t,Y_t\mid\mathbb{S}_Y}$
    \textbf{ or }
    $\beta^{\mathcal N}_{Y_t,X_t\mid\mathbb{S}_X}
    \neq
    \beta^{\bar{\mathcal N}}_{Y_t,X_t\mid\mathbb{S}_X}$
    }
        \State Add $X_t-Y_t$ to $\hat{\mathcal D}$.
    \EndIf
\EndFor

\State \Return $\hat{\mathcal D}$.

\end{algorithmic}
\end{algorithm}

\section{Sensitivity analysis}
\label{appendix:sensitivity_analysis}

To assess the sensitivity of the results to the significance level, we repeat the simulations using $\alpha=0.01$ and $\alpha=0.1$. The corresponding results are reported in Figures~\ref{fig:res_sensitivity_res_alpha_0_01} and~\ref{fig:res_sensitivity_res_alpha_0_1}, and can be compared with the main results obtained for $\alpha=0.05$ in Figure~\ref{fig:res_simulated_data}. Across $\alpha\in \{0.01,0.05,0.1\}$, the relative ranking of the proposed methods remain largely unchanged.

\input{sensitivity_analysis_plot}

\newpage
\section{Real Data Visualization}
\label{appendix:real_data_viz}

This appendix presents a visualization of the real data used for the evaluation of the algorithms. Figure~\ref{fig:monitoring_regimes} shows the IT monitoring time series and Figure~\ref{fig:mimic_regimes} shows the ICU monitoring time series. For each figure, each time series corresponds to a variable, the normal regime is shown in blue and the abnormal regime is shown in red, and the dashed orange line represents the presumed regime change point.

\begin{figure}[h]
    \centering
    \includegraphics[width=1\linewidth]{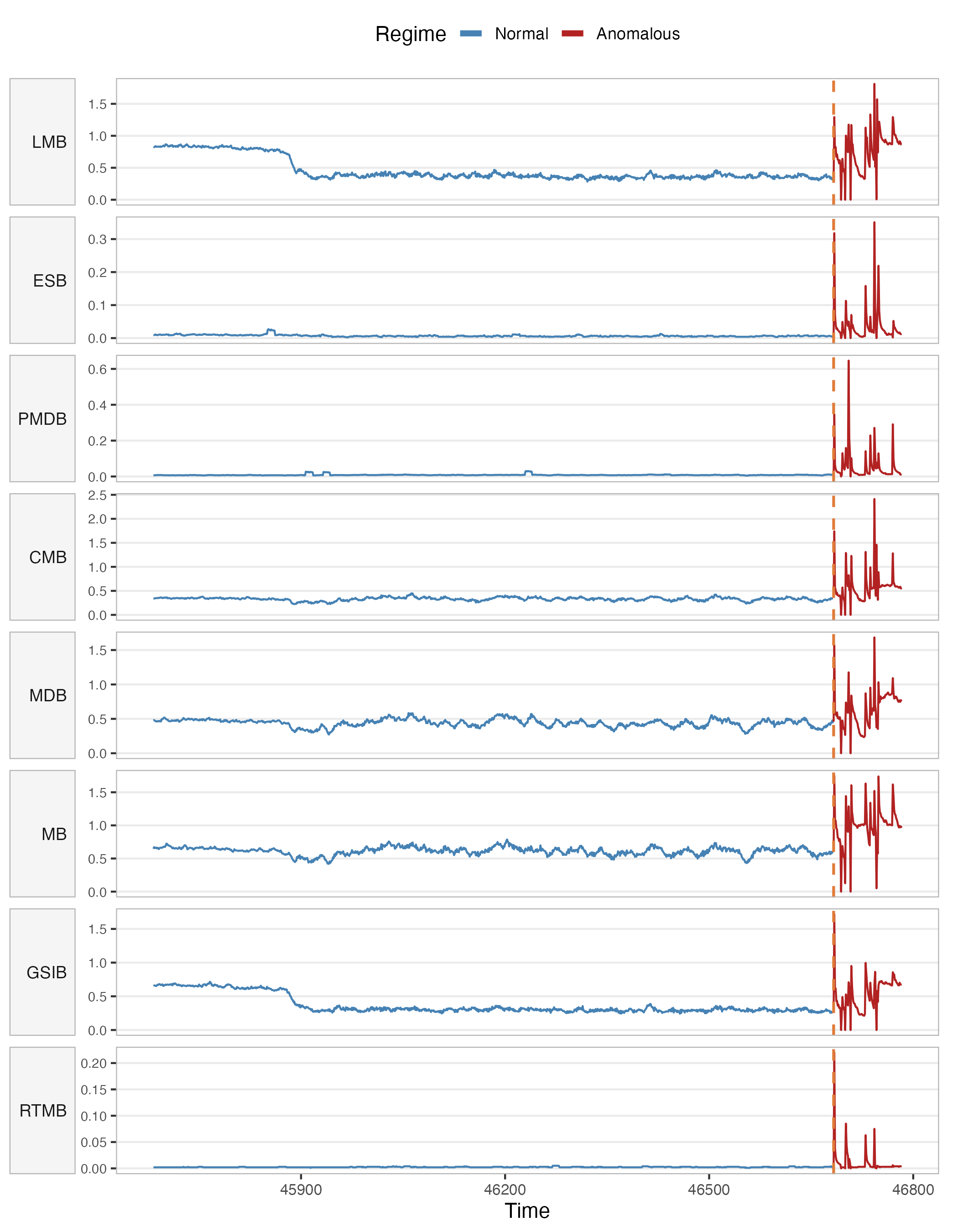}
    \caption{Time series of the eight bolt-level capacity metrics recorded in the IT monitoring system. The blue and red segments correspond to the \textit{Normal} and \textit{Anomalous} regimes respectively, as defined by the known changepoint at $t = 46683$ (dashed orange line). Each metric is displayed on its own scale.}
    \label{fig:monitoring_regimes}
\end{figure}

\label{appendix:icu}
\begin{figure}[h]
    \centering
    \includegraphics[width=1\linewidth]{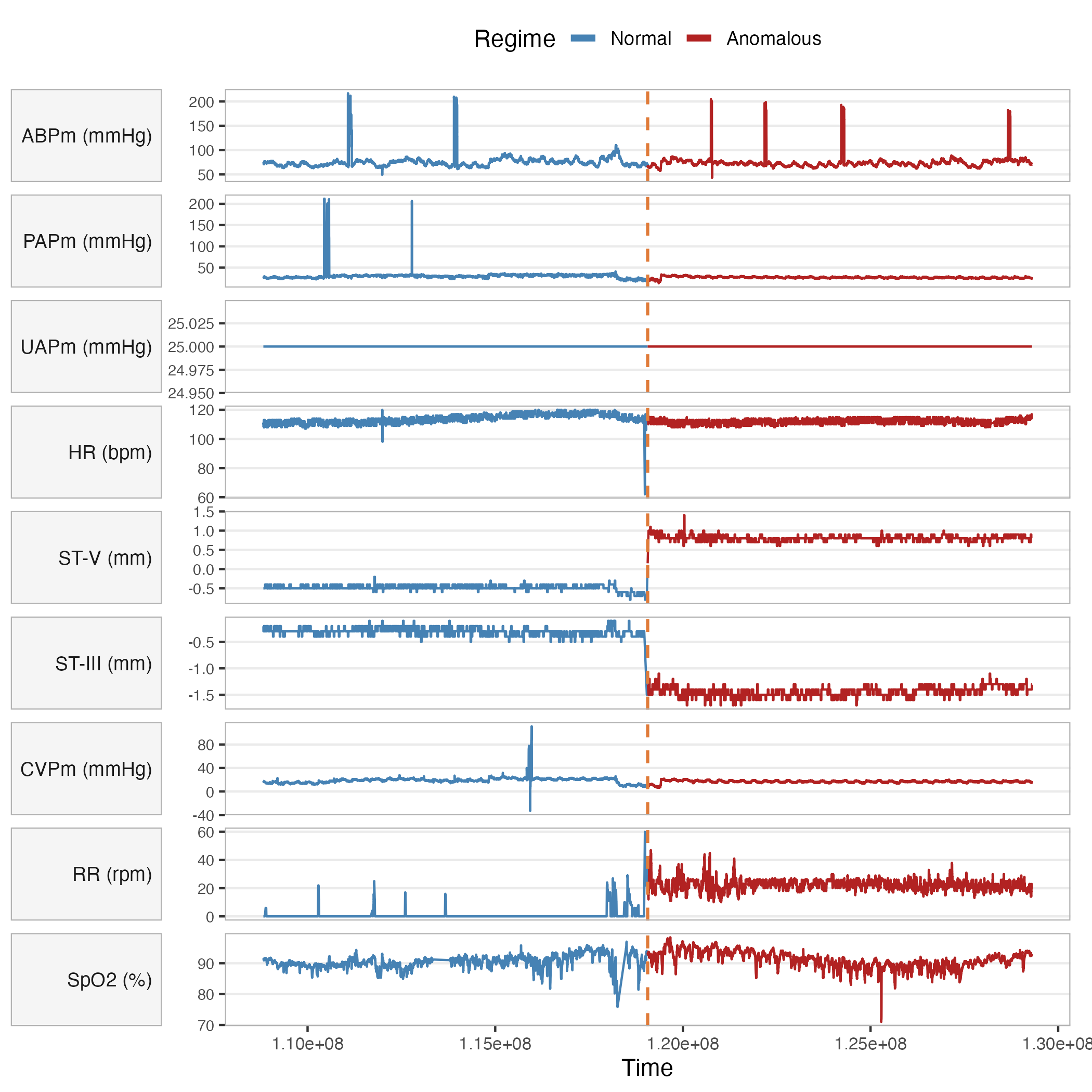}
    \caption{Time series of the nine physiological variables recorded for the selected MIMIC-IV patient. The blue and red segments correspond to the \textit{Normal} and \textit{Anomalous} regimes respectively, as defined by the changepoint detected on ST-V via PELT (dashed orange line). Each variable is displayed on its own scale.}
    \label{fig:mimic_regimes}
\end{figure}

\section{Inferred graph using real data}
\label{appendix:real_data_graphs}

We report here the FT-DFGs inferred by the different difference-graph-based methods on the two real-world applications. These graphical representations complement the results discussed in the main text and provide a direct comparison of the structural changes identified by each method. Figure~\ref{fig:it_graph} shows the graphs recovered from the IT monitoring data, while Figure~\ref{fig:mimic} reports those obtained from the MIMIC-IV patient data. Isolated vertices are omitted for readability.

\begin{figure}[t]
\centering
\begin{tikzpicture}[
    node/.style={rectangle, draw, minimum size=0.7cm, align=center, rounded corners},
    edge/.style={->, thick},
    undirected/.style={-, thick}
]

\node[node] (RRm) at (0,1) {$CMB_{t-1}$};
\node[node] (RR)  at (2,1) {$CMB_t$};

\node[node] (HRm) at (0,0) {$PMDB_{t-1}$};
\node[node] (HR)  at (2,0) {$PMDB_t$};

\node[node] (STIII) at (4,1) {$ESB_t$};
\node[node] (STV)   at (4,0) {$RTMB_t$};

\draw[edge] (RRm) -- (RR);
\draw[edge] (HRm) -- (HR);

\draw[undirected] (STIII) -- (STV);

\end{tikzpicture}
% \hfill 
% \vrule
% \hfill 
% \begin{tikzpicture}[
%     node/.style={rectangle, draw, minimum size=0.7cm, align=center, rounded corners},
%     edge/.style={->, thick},
%     undirected/.style={-, thick}
% ]

% \node[node] (RRm) at (0,1) {$CMB?$};

% \end{tikzpicture}
\caption{FT-DFG discovered by tsLDiffPC/tsLDiffPC$^2$/tsDCI/tsDCIPC on the IT monitoring data. Vertices that are not linked to anything are omitted. The  graph inferred by tsiSCAN and tPCUnion is not shown because it contains many edges, while the graph inferred by tsMBGH is not shown because it is an empty graph.}
\label{fig:it_graph}
\end{figure}

\begin{figure}[t]
\centering
\begin{tikzpicture}[
    node/.style={rectangle, draw, minimum size=0.4cm, align=center, rounded corners},
    edge/.style={->, thick},
    undirected/.style={-, thick}
]

\node[node] (CVPm)  at (0, 3) {$\mathrm{CVP}_{t-1}$};
\node[node] (RRm)   at (0, 1) {$\mathrm{RR}_{t-1}$};
\node[node] (HRm)   at (0, 0) {$\mathrm{HR}_{t-1}$};
\node[node] (ABPm)  at (0,-1) {$\mathrm{ABP}_{t-1}$};

\node[node] (CVP)  at (2, 3) {$\mathrm{CVP}_{t}$};
\node[node] (PAP)  at (2, 2) {$\mathrm{PAPm}_{t}$};
\node[node] (RR)   at (2, 1) {$\mathrm{RR}_{t}$};
\node[node] (HR)   at (2, 0) {$\mathrm{HR}_{t}$};
\node[node] (ABP)  at (2,-1) {$\mathrm{ABP}_{t}$};

\draw[edge] (CVPm) -- (CVP);
\draw[edge] (CVP) -- (PAP);
\draw[edge] (RRm)  -- (RR);
\draw[edge] (HRm)  -- (HR);
\draw[edge] (ABPm) -- (ABP);
\end{tikzpicture}
\hfill\vrule\hfill
\begin{tikzpicture}[
    node/.style={rectangle, draw, minimum size=0.4cm, align=center, rounded corners},
    edge/.style={->, thick},
    edge_2/.style={-, thick},
    undirected/.style={-, thick}
]

\node[node] (CVPm)  at (0, 3) {$\mathrm{CVP}_{t-1}$};
\node[node] (RRm)   at (0, 1) {$\mathrm{RR}_{t-1}$};
\node[node] (HRm)   at (0, 0) {$\mathrm{HR}_{t-1}$};
\node[node] (ABPm)  at (0,-1) {$\mathrm{ABP}_{t-1}$};

\node[node] (CVP)  at (2, 3) {$\mathrm{CVP}_{t}$};
\node[node] (PAP)  at (2, 2) {$\mathrm{PAPm}_{t}$};
\node[node] (RR)   at (2, 1) {$\mathrm{RR}_{t}$};
\node[node] (HR)   at (2, 0) {$\mathrm{HR}_{t}$};
\node[node] (ABP)  at (2,-1) {$\mathrm{ABP}_{t}$};

\draw[edge] (CVPm) -- (CVP);
\draw[edge_2] (CVP) -- (PAP);
\draw[edge] (RRm)  -- (RR);
\draw[edge] (HRm)  -- (HR);
\draw[edge] (ABPm) -- (ABP);
\end{tikzpicture}
\hfill\vrule\hfill
\begin{tikzpicture}[
    node/.style={rectangle, draw, minimum size=0.4cm, align=center, rounded corners},
    edge/.style={->, thick}
]
\node[node] (STVm)   at (0, 2) {$\mathrm{ST{\text{-}}V}_{t-1}$};
\node[node] (STIIIm) at (0, 0) {$\mathrm{ST{\text{-}}III}_{t-1}$};
\node[node] (STV)    at (2, 2) {$\mathrm{ST{\text{-}}V}_{t}$};
\node[node] (STIII)  at (2, 0) {$\mathrm{ST{\text{-}}III}_{t}$};
\node[node] (RR)     at (2, 1) {$\mathrm{RR}_{t}$};
\draw[edge] (STVm)   -- (RR);
\draw[edge] (STIIIm) -- (RR);
\draw[edge] (STV)    -- (RR);
\draw[edge] (STIII)  -- (RR);
\end{tikzpicture}
\caption{FT-DFGs discovered by tsLDiffPC and tsLDiffPC$^2$ (left), by tsDCI and tsDCIPC (middle) and tsISCAN (right) on the MIMIC-IV patient. Isolated vertices are omitted. The  graph inferred by tPCUnion is not shown because it contains many edges, while the graph inferred by tsMBGH is not shown because it is an empty graph.}
\label{fig:mimic}
\end{figure}

\section{Related Work}
\label{appendix:RelatedWork}

\subsection{Root cause analysis}
Several methods have been proposed for root cause analysis in dynamical systems.
Many of them require background knowledge to perform the analysis. One such method~\cite{budhathoki2021} attributes the change in the joint 
or marginal distribution across two regimes to changes in specific causal 
mechanisms (whether effect or noise defying) given 
a predefined causal graph. 
CIRCA~\cite{Li_2022} relies on a graph constructed from system architecture knowledge and performs regression-based hypothesis testing on anomalous data to identify deviations. It mainly focuses on noise-defying root causes.
Closer to our setting, EasyRCA~\cite{Assaad_AISTATS_2023} and SGRCA~\cite{Assaad_HDR_2026} provide frameworks that combine data with an abstraction of the FT-DAG, called a summary causal graph. They explicitly distinguish between changes in causal effects and changes in noise distributions. EasyRCA focuses on effect-defying root causes, whereas SGRCA addresses both effect-defying and noise-defying root causes. However, these methods require a graph as input, which may not always be available.

To relax this requirement, several approaches incorporate causal discovery, \ie, learning the FT-DAG or a partially oriented version from data~\cite{Spirtes_2000,Runge_2019_Nature,Runge_2020_UAI,Assaad_2022,Reiter_2026}.
%EasyRCA$^*$ first learns the FT-DAG of the normal regime using the PCMCI$^+$ algorithm, then constructs a summary causal graph and applies EasyRCA on the inferred structure if the inferred structure is acyclic. 
MicroCause~\cite{Meng_2020} uses PCMCI~\cite{Runge_2019_Nature} (restricted version of PCMCI$^+$ that does not allow for instantaneous relations) to learn causal relations among anomalous time series and identifies root causes via a random walk strategy based on partial correlations. However, its ability to detect the types of root causes considered here remains unclear. RCD~\cite{Ikram_2022} introduces a regime indicator variable and applies hierarchical causal discovery to identify variables whose distributions change across regimes. It does not distinguish between changes in causal effects and noise, assumes a fixed causal graph, and requires discrete data. Finally, T-RCA~\cite{Zan_CIKM_2024} defines root causes as the first variables to cross a threshold and trigger subsequent anomalies. It learns a graph from offline data using PCMCI$^+$ and combines it with anomaly timing information to identify root causes.

In a different view,
several works have addressed root cause analysis of outliers~\cite{budhathoki2022outliers,orchard2026,schkoda2026}.
These methods differ from those evaluated in this paper in several key aspects. First, the anomalies considered in these works are provoked by noise-defying root causes rather than as effect-defying root causes, as we focus in this paper. Second, all three methods are designed to identify the root causes from \emph{single} anomalous observation, whereas our approach operates on samples from two regimes (normal and anomalous) to detect changes in causal mechanisms.

\subsection{Difference graph discovery}

Since many causal discovery methods exist to recover the causal graph from observational data, a naive approach to identifying the difference graph would consist in running a causal discovery algorithm separately on each regime and comparing the outputs to detect discrepancies. However, such an approach would only capture structural changes, i.e., edges that appear or disappear across regimes, corresponding to coefficients transitioning from zero to nonzero or vice versa. Yet the difference graph should also encode edges whose coefficients change in magnitude without inducing any structural modification. A possible solution is therefore to estimate the causal graph in each regime, construct their union graph, and compare the corresponding edge coefficients across regimes. However, this strategy can be computationally costly. To address this problem, several dedicated methods have been proposed. While they rely on different assumptions and identification strategies, they all aim to estimate a graph encoding which causal mechanisms differ across environments. In this paper, we mainly focus on methods that use equality tests to infer the FT-DFG.

\begin{itemize}
    \item  LDiffPC~\cite{Bystrova_Workshop_UAI_2024,Bystrova_Arxiv_2026} is an algorithm for discovering a partially oriented graph in linear non-dynamic structural causal models. It recovers the skeleton of the graph by testing equality of regression coefficients across environments and then orients edges using collider detection and Meek \cite{Meek1995CausalIA} propagation rules.

    \item DCI~\cite{Wang_Neurips_2018} is also designed to discover a partially oriented graph in linear  non-dynamic structural causal models. It recovers the skeleton using the same procedure as LDiffPC. However, its orientation step differs: instead of relying on separation sets and orientation rules, DCI orients edges by exploiting changes in residual variances across environments.
\end{itemize}

There exists also other approaches that do not use equality tests. For instance, 
MBGH~\cite{Malik_UAI_2024} formulates difference graph discovery as a covariance-based estimation problem and recovers the graph through a recursive peeling procedure, while iSCAN~\cite{Chen_Neurips_2023} targets nonlinear additive noise models and identifies changed mechanisms through distributional invariance arguments. 

%% file: sensitivity_analysis_plot.tex
%%%% plot pour alpha 0.01

\begin{figure}[t]
\centering
\begin{subfigure}{0.31\textwidth}
\centering
\begin{tikzpicture}
\begin{axis}[
    width=\textwidth,
    height=4cm,
    xlabel={Number of time-series},
    ylabel={F1 ($\mathcal{C}$, $\widehat{\mathcal{C}}$)},
    title={One parent},
    ymin=0, ymax=1,
legend style={
    at={(2.2,1.3)},
    anchor=south,
    legend columns=6,
    font=\scriptsize
}]
% -----------------------------
%%%% f1 incoming shifted one parent
% -----------------------------
\addplot+[
    thick,
    mark=*, red, mark options={fill=red, draw=red}, solid,
    error bars/.cd,
        y dir=both,
        y explicit
] coordinates {
    (3,0.07) +- (0,0.21)
    (5,0.0) +- (0,0.0)
    (7,0.03) +- (0,0.08)
    (9,0.0) +- (0,0.0)
};
\addlegendentry{tsLDiffPC}

\addplot+[
    thick,
    mark=*, orange, mark options={fill=orange, draw=orange}, solid,
    error bars/.cd,
        y dir=both,
        y explicit
] coordinates {
    (3.1,0.53) +- (0,0.39)
    (5.1,0.47) +- (0,0.42)
    (7.1,0.56) +- (0,0.36)
    (9.1,0.43) +- (0,0.39)
};
\addlegendentry{tsLDiffPC$^2$}

\addplot+[
    thick,
    mark=*, blue, mark options={fill=blue, draw=blue}, solid,
    error bars/.cd,
        y dir=both,
        y explicit
] coordinates {
    (3.2,0.83) +- (0,0.32)
    (5.2,0.6) +- (0,0.52)
    (7.2,0.63) +- (0,0.49)
    (9.2,0.8) +- (0,0.42)
};
\addlegendentry{tsDCI}

\addplot+[
    thick,
    mark=*, cyan, mark options={fill=cyan, draw=cyan}, solid,
    error bars/.cd,
        y dir=both,
        y explicit
] coordinates {
    (3.3,0.83) +- (0,0.32)
    (5.3,0.6) +- (0,0.52)
    (7.3,0.73) +- (0,0.45)
    (9.3,0.8) +- (0,0.42)
};
\addlegendentry{tsDCIPC}

\addplot+[
    thick,
    mark=*, olive, mark options={fill=olive, draw=olive},solid, 
    error bars/.cd,
        y dir=both,
        y explicit
] coordinates {
    (2.9,0.3) +- (0,0.48)
    (4.9,0.07) +- (0,0.21)
    (6.9,0.37) +- (0,0.48)
    (8.9,0.37) +- (0,0.36)
};
\addlegendentry{tsMBGH}

\addplot+[
    thick,
    mark=*, yellow, mark options={fill=yellow, draw=yellow}, solid,
    error bars/.cd,
        y dir=both,
        y explicit
] coordinates {
    (2.8,0.33) +- (0,0.38)
    (4.8,0.37) +- (0,0.45)
    (6.8,0.04) +- (0,0.13)
    (8.8,0.03) +- (0,0.1)
};
\addlegendentry{tsiSCAN}

\addplot+[
    thick,
    mark=*, teal, mark options={fill=teal, draw=teal}, solid,
    error bars/.cd,
        y dir=both,
        y explicit
] coordinates {
    (2.6,0.47) +- (0,0.32)
    (4.6,0.2) +- (0,0.32)
    (6.6,0.27) +- (0,0.34)
    (8.6,0.13) +- (0,0.28)
};
\addlegendentry{MicroCause}

\addplot+[
    thick,
    mark=*, pink, mark options={fill=pink, draw=pink}, solid,
    error bars/.cd,
        y dir=both,
        y explicit
] coordinates {
    (2.7,0.7) +- (0,0.48)
    (4.7,0.4) +- (0,0.52)
    (6.7,0.4) +- (0,0.52)
    (8.7,0.4) +- (0,0.52)
};
\addlegendentry{RCD}

\addplot+[
    thick,
    mark=*, gray, mark options={fill=gray, draw=gray}, solid,
    error bars/.cd,
        y dir=both,
        y explicit
] coordinates {
    (3.,0.62) +- (0,0.28)
    (5,0.32) +- (0,0.25)
    (7,0.23) +- (0,0.21)
    (9,0.24) +- (0,0.18)
};
\addlegendentry{tPCUnion}

\end{axis}
\end{tikzpicture}
% \caption{Setting 1}
\end{subfigure}
\hfill
\begin{subfigure}{0.31\textwidth}
\centering
\begin{tikzpicture}
\begin{axis}[
    width=\textwidth,
    height=4cm,
    xlabel={Number of time-series},
    title={All parents},
    ymin=0, ymax=1
]
% -----------------------------
%%%% f1 incoming shifted all parents
% -----------------------------
\addplot+[
    thick,
    mark=*, red, mark options={fill=red, draw=red}, solid,
    error bars/.cd,
        y dir=both,
        y explicit
] coordinates {
    (3,0.07) +- (0,0.21)
    (5,0.2) +- (0,0.42)
    (7,0.33) +- (0,0.47)
    (9,0.5) +- (0,0.53)
}; %tsLDiffPC

\addplot+[
    thick,
    mark=*, orange, mark options={fill=orange, draw=orange}, solid,
    error bars/.cd,
        y dir=both,
        y explicit
] coordinates {
    (3.1,0.53) +- (0,0.39)
    (5.1,0.6) +- (0,0.44)
    (7.1,0.79) +- (0,0.25)
    (9.1,0.87) +- (0,0.17)
};%tsLDiffPC2

\addplot+[
    thick,
    mark=*, blue, mark options={fill=blue, draw=blue}, solid,
    error bars/.cd,
        y dir=both,
        y explicit
] coordinates {
    (3.2,0.83) +- (0,0.32)
    (5.2,0.7) +- (0,0.48)
    (7.2,0.79) +- (0,0.37)
    (9.2,0.9) +- (0,0.32)
}; %tsDCI

\addplot+[
    thick,
    mark=*, cyan, mark options={fill=cyan, draw=cyan}, solid,
    error bars/.cd,
        y dir=both,
        y explicit
] coordinates {
    (3.3,0.83) +- (0,0.32)
    (5.3,0.7) +- (0,0.48)
    (7.3,0.79) +- (0,0.37)
    (9.3,0.9) +- (0,0.32)
};%tsDCIPC

\addplot+[
    thick,
    mark=*, olive, mark options={fill=olive, draw=olive}, solid,
    error bars/.cd,
        y dir=both,
        y explicit
] coordinates {
    (2.9,0.3) +- (0,0.48)
    (4.9,0.12) +- (0,0.25)
    (6.9,0.50) +- (0,0.53)
    (8.9,0.42) +- (0,0.41)
}; %tsMalik

\addplot+[
    thick,
    mark=*, yellow, mark options={fill=yellow, draw=yellow}, solid,
    error bars/.cd,
        y dir=both,
        y explicit
] coordinates {
    (2.8,0.33) +- (0,0.38)
    (4.8,0.49) +- (0,0.41)
    (6.8,0.0) +- (0,0.0)
    (8.8,0.03) +- (0,0.08)
}; %tsiSCAN

\addplot+[
    thick,
    mark=*, teal, mark options={fill=teal, draw=teal}, solid,
    error bars/.cd,
        y dir=both,
        y explicit
] coordinates {
    (2.6,0.47) +- (0,0.32)
    (4.6,0.2) +- (0,0.32)
    (6.6,0.2) +- (0,0.32)
    (8.6,0.07) +- (0,0.21)
}; %microcause

\addplot+[
    thick,
    mark=*, pink, mark options={fill=pink, draw=pink}, solid,
    error bars/.cd,
        y dir=both,
        y explicit
] coordinates {
    (2.7,0.7) +- (0,0.48)
    (4.7,0.4) +- (0,0.52)
    (6.7,0.6) +- (0,0.52)
    (8.7,0.9) +- (0,0.32)
}; %rcd

\addplot+[
    thick,
    mark=*, gray, mark options={fill=gray, draw=gray}, solid,
    error bars/.cd,
        y dir=both,
        y explicit
] coordinates {
    (3.,0.55) +- (0,0.34)
    (5,0.40) +- (0,0.25)
    (7,0.26) +- (0,0.20)
    (9,0.27) +- (0,0.21)
};%pcunion

\end{axis}
\end{tikzpicture}
% \caption{Setting 1}
\end{subfigure}
\hfill
\begin{subfigure}{0.31\textwidth}
\centering
\begin{tikzpicture}
\begin{axis}[
    width=\textwidth,
    height=4cm,
    xlabel={Number of time-series},
    title={At least 2 parents},
    ymin=0, ymax=1,
legend style={
    at={(1.2,1.3)},
    anchor=south,
    legend columns=3,
    font=\scriptsize
}]
% -----------------------------
%%%% f1 incoming shifted, all parents with minimum 2 parents
% -----------------------------

\addplot+[
    thick,
    mark=*, red, mark options={fill=red, draw=red}, solid,
    error bars/.cd,
        y dir=both,
        y explicit
] coordinates {
    (3,0.83) +- (0,0.32) 
    (5,0.8) +- (0,0.42) 
    (7,0.73) +- (0,0.45)
    (9,0.7) +- (0,0.48)
}; %tsLDiffPC

\addplot+[
    thick,
    mark=*, orange, mark options={fill=orange, draw=orange}, solid,
    error bars/.cd,
        y dir=both,
        y explicit
] coordinates {
    (3.1,0.83) +- (0,0.32) 
    (5.1,0.87) +- (0,0.32) 
    (7.1,0.86) +- (0,0.25)
    (9.1,0.83) +- (0,0.32)
};%tsLDiffPC2

\addplot+[
    thick,
    mark=*, blue, mark options={fill=blue, draw=blue}, solid,
    error bars/.cd,
        y dir=both,
        y explicit
] coordinates {
    (3.2,0.57) +- (0,0.32)
    (5.2,0.9) +- (0,0.32)
    (7.2,0.59) +- (0,0.47)
    (9.2,0.77) +- (0,0.42)
};%tsDCI

\addplot+[
    thick,
    mark=*, cyan, mark options={fill=cyan, draw=cyan}, solid,
    error bars/.cd,
        y dir=both,
        y explicit
] coordinates {
    (3.2,0.57) +- (0,0.32)
    (5.2,0.9) +- (0,0.32)
    (7.2,0.59) +- (0,0.47)
    (9.2,0.77) +- (0,0.42)
};%ok

\addplot+[
    thick,
    mark=*, olive, mark options={fill=olive, draw=olive}, solid,
    error bars/.cd,
        y dir=both,
        y explicit
] coordinates {
    (2.9,0.33) +- (0,0.35)
    (4.9,0.44) +- (0,0.42)
    (6.9,0.36) +- (0,0.42)
    (8.9,0.46) +- (0,0.35)
}; %malik

\addplot+[
    thick,
    mark=*, yellow, mark options={fill=yellow, draw=yellow}, solid, 
    error bars/.cd,
        y dir=both,
        y explicit
] coordinates {
    (2.8,0.5) +- (0,0.29)
    (4.8,0.32) +- (0,0.37)
    (6.8,0.12) +- (0,0.18)
    (8.8,0.0) +- (0,0.0)
};% iscan ok

\addplot+[
    thick,
    mark=*, teal, mark options={fill=teal, draw=teal}, solid,
    error bars/.cd,
        y dir=both,
        y explicit
] coordinates {
    (2.6,0.6) +- (0,0.21)
    (4.6,0.33) +- (0,0.35)
    (6.6,0.07) +- (0,0.21)
    (8.6,0.07) +- (0,0.21)
}; %microcause

\addplot+[
    thick,
    mark=*, pink, mark options={fill=pink, draw=pink}, solid,
    error bars/.cd,
        y dir=both,
        y explicit
] coordinates {
    (2.7,0.4) +- (0,0.52)
    (4.7,0.5) +- (0,0.53)
    (6.7,0.7) +- (0,0.48)
    (8.7,0.7) +- (0,0.48)
}; %rcd

\addplot+[
    thick,
    mark=*, gray, mark options={fill=gray, draw=gray}, solid,
    error bars/.cd,
        y dir=both,
        y explicit
] coordinates {
    (3.,0.48) +- (0,0.29)
    (5,0.45) +- (0,0.35)
    (7,0.30) +- (0,0.12)
    (9,0.31) +- (0,0.14)
};%pcunion

\end{axis}
\end{tikzpicture}

% \caption{Setting 2}
\end{subfigure}

\vspace{0.4cm}

% -----------------------------
% -----------------------------
% NODE RElAXED
% -----------------------------
% -----------------------------

\begin{subfigure}{0.31\textwidth}
\centering
\begin{tikzpicture}
\begin{axis}[
    width=\textwidth,
    height=4cm,
    xlabel={Number of time-series},
    ylabel={F1 ($\mathcal{C}$, $\widehat{\mathcal{C}}\cup\widehat{\mathcal{C}}_P$)},    % title={Setting 3},
    ymin=0, ymax=1
]
% -----------------------------
%%% f1 node relaxed one parent
% -----------------------------
\addplot+[
    thick,
    mark=*, red, mark options={fill=red, draw=red}, solid,
    error bars/.cd,
        y dir=both,
        y explicit
] coordinates {
    (3,0.6) +- (0,0.21)
    (5,0.64) +- (0,0.08)
    (7,0.63) +- (0,0.13)
    (9,0.67) +- (0,0.0)
};%tsLDiffPC

\addplot+[
    thick,
    mark=*, orange, mark options={fill=orange, draw=orange}, solid,
    error bars/.cd,
        y dir=both,
        y explicit
] coordinates {
    (3.1,0.67) +- (0,0.27)
    (5.1,0.71) +- (0,0.18)
    (7.1,0.63) +- (0,0.30)
    (9.1,0.7) +- (0,0.11)
}; %tsLDiffPC2

\addplot+[
    thick,
    mark=*, blue, mark options={fill=blue, draw=blue}, solid,
    error bars/.cd,
        y dir=both,
        y explicit
] coordinates {
    (3.2,0.83) +- (0,0.32)
    (5.2,0.64) +- (0,0.48)
    (7.2,0.63) +- (0,0.49)
    (9.2,0.8) +- (0,0.42)
}; %tsDCI

\addplot+[
    thick,
    mark=*, cyan, mark options={fill=cyan, draw=cyan}, solid,
    error bars/.cd,
        y dir=both,
        y explicit
] coordinates {
    (3.3,0.83) +- (0,0.32)
    (5.3,0.64) +- (0,0.48)
    (7.3,0.63) +- (0,0.49)
    (9.3,0.8) +- (0,0.42)
}; %tsDCIPC

\addplot+[
    thick,
    mark=*, olive, mark options={fill=olive, draw=olive},solid, 
    error bars/.cd,
        y dir=both,
        y explicit
] coordinates {
    (2.9,0.3) +- (0,0.48)
    (4.9,0.07) +- (0,0.21)
    (6.9,0.37) +- (0,0.48)
    (8.9,0.37) +- (0,0.36)
};

\addplot+[
    thick,
    mark=*, yellow, mark options={fill=yellow, draw=yellow}, solid,
    error bars/.cd,
        y dir=both,
        y explicit
] coordinates {
    (2.8,0.33) +- (0,0.38)
    (4.8,0.37) +- (0,0.45)
    (6.8,0.04) +- (0,0.13)
    (8.8,0.03) +- (0,0.1)
};%tsiSCAN

\addplot+[
    thick,
    mark=*, teal, mark options={fill=teal, draw=teal}, solid,
    error bars/.cd,
        y dir=both,
        y explicit
] coordinates {
    (2.6,0.47) +- (0,0.32)
    (4.6,0.2) +- (0,0.32)
    (6.6,0.27) +- (0,0.34)
    (8.6,0.13) +- (0,0.28)
}; %microcause

\addplot+[
    thick,
    mark=*, pink, mark options={fill=pink, draw=pink}, solid,
    error bars/.cd,
        y dir=both,
        y explicit
] coordinates {
    (2.7,0.7) +- (0,0.48)
    (4.7,0.4) +- (0,0.52)
    (6.7,0.4) +- (0,0.52)
    (8.7,0.4) +- (0,0.52)
}; %rcd

\addplot+[
    thick,
    mark=*, gray, mark options={fill=gray, draw=gray}, solid,
    error bars/.cd,
        y dir=both,
        y explicit
] coordinates {
    (3.,0.62) +- (0,0.28)
    (5,0.32) +- (0,0.25)
    (7,0.23) +- (0,0.21)
    (9,0.24) +- (0,0.18)
};%pcunion

\end{axis}
\end{tikzpicture}
% \caption{Setting 3}
\end{subfigure}
\hfill
\begin{subfigure}{0.31\textwidth}
\centering
\begin{tikzpicture}
\begin{axis}[
    width=\textwidth,
    height=4cm,
    xlabel={Number of time-series},
    ymin=0, ymax=1
]
% -----------------------------
%%% f1 node relaxed all parents
% -----------------------------
\addplot+[
    thick,
    mark=*, red, mark options={fill=red, draw=red}, solid,
    error bars/.cd,
        y dir=both,
        y explicit
] coordinates {
    (3,0.6) +- (0,0.21)
    (5,0.71) +- (0,0.18)
    (7,0.73) +- (0,0.23)
    (9,0.83) +- (0,0.18)
}; %tsLDiffPC

\addplot+[
    thick,
    mark=*, orange, mark options={fill=orange, draw=orange}, solid,
    error bars/.cd,
        y dir=both,
        y explicit
] coordinates {
    (3.1,0.67) +- (0,0.27)
    (5.1,0.77) +- (0,0.21)
    (7.1,0.79) +- (0,0.25)
    (9.1,0.87) +- (0,0.17)
}; %tsLDiffPC2

\addplot+[
    thick,
    mark=*, blue, mark options={fill=blue, draw=blue}, solid,
    error bars/.cd,
        y dir=both,
        y explicit
] coordinates {
    (3.2,0.83) +- (0,0.32)
    (5.2,0.74) +- (0,0.43)
    (7.2,0.79) +- (0,0.37)
    (9.2,0.9) +- (0,0.32)
}; %tsDCI

\addplot+[
    thick,
    mark=*, cyan, mark options={fill=cyan, draw=cyan}, solid,
    error bars/.cd,
        y dir=both,
        y explicit
] coordinates {
    (3.3,0.83) +- (0,0.32)
    (5.3,0.74) +- (0,0.43)
    (7.3,0.79) +- (0,0.37)
    (9.3,0.9) +- (0,0.32)
}; %tsDCIPC

\addplot+[
    thick,
    mark=*, olive, mark options={fill=olive, draw=olive}, solid,
    error bars/.cd,
        y dir=both,
        y explicit
] coordinates {
    (2.9,0.3) +- (0,0.48)
    (4.9,0.12) +- (0,0.25)
    (6.9,0.50) +- (0,0.53)
    (8.9,0.42) +- (0,0.41)
}; %tsMalik

\addplot+[
    thick,
    mark=*, yellow, mark options={fill=yellow, draw=yellow}, solid, 
    error bars/.cd,
        y dir=both,
        y explicit
] coordinates {
    (2.8,0.33) +- (0,0.38)
    (4.8,0.49) +- (0,0.41)
    (6.8,0.0) +- (0,0.0)
    (8.8,0.03) +- (0,0.08)
}; %tsiSCAN

\addplot+[
    thick,
    mark=*, teal, mark options={fill=teal, draw=teal}, solid,
    error bars/.cd,
        y dir=both,
        y explicit
] coordinates {
    (2.6,0.47) +- (0,0.32)
    (4.6,0.2) +- (0,0.32)
    (6.6,0.2) +- (0,0.32)
    (8.6,0.07) +- (0,0.21)
}; %microcause

\addplot+[
    thick,
    mark=*, pink, mark options={fill=pink, draw=pink}, solid,
    error bars/.cd,
        y dir=both,
        y explicit
] coordinates {
    (2.7,0.7) +- (0,0.48)
    (4.7,0.4) +- (0,0.52)
    (6.7,0.6) +- (0,0.52)
    (8.7,0.9) +- (0,0.32)
}; %rcd

\addplot+[
    thick,
    mark=*, gray, mark options={fill=gray, draw=gray}, solid,
    error bars/.cd,
        y dir=both,
        y explicit
] coordinates {
    (3.,0.55) +- (0,0.34)
    (5,0.40) +- (0,0.25)
    (7,0.26) +- (0,0.20)
    (9,0.27) +- (0,0.21)
};%pcunion

\end{axis}
\end{tikzpicture}
% \caption{Setting 1}
\end{subfigure}
\hfill
\begin{subfigure}{0.31\textwidth}
\centering
\begin{tikzpicture}
\begin{axis}[
    width=\textwidth,
    height=4cm,
    xlabel={Number of time-series},
    ymin=0, ymax=1,
legend style={
    at={(1.2,1.3)},
    anchor=south,
    legend columns=3,
    font=\scriptsize
}]

% -----------------------------
%%%% node relaxed, all parents, with minimum 2 parents
% -----------------------------
\addplot+[
    thick,
    mark=*, red, mark options={fill=red, draw=red}, solid,
    error bars/.cd,
        y dir=both,
        y explicit
] coordinates {
    (3,0.82) +- (0,0.24)
    (5,0.9) +- (0,0.16)
    (7,0.89) +- (0,0.25)
    (9,0.83) +- (0,0.32)
}; %tslindiffpc

\addplot+[
    thick,
    mark=*, orange, mark options={fill=orange, draw=orange}, solid,
    error bars/.cd,
        y dir=both,
        y explicit
] coordinates {
    (3.1,0.83) +- (0,0.22)
    (5.1,0.88) +- (0,0.32)
    (7.1,0.89) +- (0,0.25)
    (9.1,0.83) +- (0,0.32)
}; %tslindiffpc2

\addplot+[
    thick,
    mark=*, blue, mark options={fill=blue, draw=blue}, solid,
    error bars/.cd,
        y dir=both,
        y explicit
] coordinates {
    (3.2,0.63) +- (0,0.28)
    (5.2,0.87) +- (0,0.32)
    (7.2,0.69) +- (0,0.44)
    (9.2,0.77) +- (0,0.42)
}; %tsdci 

\addplot+[
    thick,
    mark=*, cyan, mark options={fill=cyan, draw=cyan}, solid,
    error bars/.cd,
        y dir=both,
        y explicit
] coordinates {
    (3.3,0.63) +- (0,0.28)
    (5.3,0.87) +- (0,0.32)
    (7.3,0.69) +- (0,0.44)
    (9.3,0.77) +- (0,0.42)
};% tsdcipc

\addplot+[
    thick,
    mark=*, olive, mark options={fill=olive, draw=olive}, solid,
    error bars/.cd,
        y dir=both,
        y explicit
] coordinates {
    (2.9,0.33) +- (0,0.35)
    (4.9,0.44) +- (0,0.42)
    (6.9,0.36) +- (0,0.42)
    (8.9,0.46) +- (0,0.35)
}; %malik ok

\addplot+[
    thick,
    mark=*, yellow, mark options={fill=yellow, draw=yellow}, solid, 
    error bars/.cd,
        y dir=both,
        y explicit
] coordinates {
    (2.8,0.5) +- (0,0.29)
    (4.8,0.32) +- (0,0.37)
    (6.8,0.12) +- (0,0.18)
    (8.8,0.0) +- (0,0.0)
};% iscan ok

\addplot+[
    thick,
    mark=*, teal, mark options={fill=teal, draw=teal}, solid,
    error bars/.cd,
        y dir=both,
        y explicit
] coordinates {
    (2.6,0.6) +- (0,0.21)
    (4.6,0.33) +- (0,0.35)
    (6.6,0.07) +- (0,0.21)
    (8.6,0.07) +- (0,0.21)
}; %microcause

\addplot+[
    thick,
    mark=*, pink, mark options={fill=pink, draw=pink}, solid,
    error bars/.cd,
        y dir=both,
        y explicit
] coordinates {
    (2.7,0.4) +- (0,0.52)
    (4.7,0.5) +- (0,0.53)
    (6.7,0.7) +- (0,0.48)
    (8.7,0.7) +- (0,0.48)
}; %rcd

\addplot+[
    thick,
    mark=*, gray, mark options={fill=gray, draw=gray}, solid,
    error bars/.cd,
        y dir=both,
        y explicit
] coordinates {
    (3.,0.48) +- (0,0.29)
    (5,0.45) +- (0,0.35)
    (7,0.30) +- (0,0.12)
    (9,0.31) +- (0,0.14)
};%pcunion

\end{axis}
\end{tikzpicture}
% \caption{Setting 4}
\end{subfigure}

\caption{Mean F1-score for eight methods as a function of the number of vertices across three settings, for $\alpha = 0.01$.}
\label{fig:res_sensitivity_res_alpha_0_01}
\end{figure}

%%%% plot pour alpha 0.1

\begin{figure}[t]
\centering
\begin{subfigure}{0.31\textwidth}
\centering
\begin{tikzpicture}
\begin{axis}[
    width=\textwidth,
    height=4cm,
    xlabel={Number of time-series},
    ylabel={F1 ($\mathcal{C}$, $\widehat{\mathcal{C}}$)},
    title={One parent},
    ymin=0, ymax=1,
legend style={
    at={(2.2,1.3)},
    anchor=south,
    legend columns=6,
    font=\scriptsize
}]
% -----------------------------
%%%% f1 incoming shifted one parent
% -----------------------------
\addplot+[
    thick,
    mark=*, red, mark options={fill=red, draw=red}, solid,
    error bars/.cd,
        y dir=both,
        y explicit
] coordinates {
    (3,0.27) +- (0,0.44)
    (5,0.0) +- (0,0.0)
    (7,0.03) +- (0,0.08)
    (9,0.0) +- (0,0.0)
};
\addlegendentry{tsLDiffPC}

\addplot+[
    thick,
    mark=*, orange, mark options={fill=orange, draw=orange}, solid,
    error bars/.cd,
        y dir=both,
        y explicit
] coordinates {
    (3.1,0.6) +- (0,0.34)
    (5.1,0.5) +- (0,0.45)
    (7.1,0.49) +- (0,0.32)
    (9.1,0.41) +- (0,0.41)
};
\addlegendentry{tsLDiffPC$^2$}

\addplot+[
    thick,
    mark=*, blue, mark options={fill=blue, draw=blue}, solid,
    error bars/.cd,
        y dir=both,
        y explicit
] coordinates {
    (3.2,0.8) +- (0,0.32)
    (5.2,0.7) +- (0,0.48)
    (7.2,0.69) +- (0,0.44)
    (9.2,0.78) +- (0,0.44)
};
\addlegendentry{tsDCI}

\addplot+[
    thick,
    mark=*, cyan, mark options={fill=cyan, draw=cyan}, solid,
    error bars/.cd,
        y dir=both,
        y explicit
] coordinates {
    (3.2,0.8) +- (0,0.32)
    (5.2,0.7) +- (0,0.48)
    (7.2,0.69) +- (0,0.44)
    (9.2,0.78) +- (0,0.44)
};
\addlegendentry{tsDCIPC}

\addplot+[
    thick,
    mark=*, olive, mark options={fill=olive, draw=olive},solid, 
    error bars/.cd,
        y dir=both,
        y explicit
] coordinates {
    (2.9,0.3) +- (0,0.48)
    (4.9,0.07) +- (0,0.21)
    (6.9,0.37) +- (0,0.48)
    (8.9,0.32) +- (0,0.34)
};
\addlegendentry{tsMBGH}

\addplot+[
    thick,
    mark=*, yellow, mark options={fill=yellow, draw=yellow}, solid,
    error bars/.cd,
        y dir=both,
        y explicit
] coordinates {
    (2.8,0.28) +- (0,0.30)
    (4.8,0.26) +- (0,0.40)
    (6.8,0.04) +- (0,0.13)
    (8.8,0.0) +- (0,0.0)
};
\addlegendentry{tsiSCAN}

\addplot+[
    thick,
    mark=*, teal, mark options={fill=teal, draw=teal}, solid,
    error bars/.cd,
        y dir=both,
        y explicit
] coordinates {
    (2.6,0.6) +- (0,0.21)
    (4.6,0.2) +- (0,0.32)
    (6.6,0.13) +- (0,0.28)
    (8.6,0.07) +- (0,0.22)
};
\addlegendentry{MicroCause}

\addplot+[
    thick,
    mark=*, pink, mark options={fill=pink, draw=pink}, solid,
    error bars/.cd,
        y dir=both,
        y explicit
] coordinates {
    (2.7,0.5) +- (0,0.48)
    (4.7,0.4) +- (0,0.52)
    (6.7,0.4) +- (0,0.52)
    (8.7,0.4) +- (0,0.52)
};
\addlegendentry{RCD}

\addplot+[
    thick,
    mark=*, gray, mark options={fill=gray, draw=gray}, solid,
    error bars/.cd,
        y dir=both,
        y explicit
] coordinates {
    (3.,0.50) +- (0,0.30)
    (5,0.27) +- (0,0.25)
    (7,0.29) +- (0,0.17)
    (9,0.3) +- (0,0.15)
};
\addlegendentry{tPCUnion}

\end{axis}
\end{tikzpicture}
% \caption{Setting 1}
\end{subfigure}
\hfill
\begin{subfigure}{0.31\textwidth}
\centering
\begin{tikzpicture}
\begin{axis}[
    width=\textwidth,
    height=4cm,
    xlabel={Number of time-series},
    title={All parents},
    ymin=0, ymax=1
]
% -----------------------------
%%%% f1 incoming shifted all parents
% -----------------------------
\addplot+[
    thick,
    mark=*, red, mark options={fill=red, draw=red}, solid,
    error bars/.cd,
        y dir=both,
        y explicit
] coordinates {
    (3,0.27) +- (0,0.44)
    (5,0.17) +- (0,0.36)
    (7,0.33) +- (0,0.47)
    (9,0.56) +- (0,0.53)
}; %tsLDiffPC

\addplot+[
    thick,
    mark=*, orange, mark options={fill=orange, draw=orange}, solid,
    error bars/.cd,
        y dir=both,
        y explicit
] coordinates {
    (3.1,0.6) +- (0,0.34)
    (5.1,0.57) +- (0,0.42)
    (7.1,0.79) +- (0,0.25)
    (9.1,0.87) +- (0,0.17)
};%tsLDiffPC2

\addplot+[
    thick,
    mark=*, blue, mark options={fill=blue, draw=blue}, solid,
    error bars/.cd,
        y dir=both,
        y explicit
] coordinates {
    (3.2,0.8) +- (0,0.32)
    (5.2,0.76) +- (0,0.42)
    (7.2,0.73) +- (0,0.45)
    (9.2,0.89) +- (0,0.33)
}; %tsDCI

\addplot+[
    thick,
    mark=*, cyan, mark options={fill=cyan, draw=cyan}, solid,
    error bars/.cd,
        y dir=both,
        y explicit
] coordinates {
    (3.3,0.8) +- (0,0.32)
    (5.3,0.76) +- (0,0.42)
    (7.3,0.73) +- (0,0.45)
    (9.3,0.89) +- (0,0.33)
};%tsDCIPC

\addplot+[
    thick,
    mark=*, olive, mark options={fill=olive, draw=olive}, solid,
    error bars/.cd,
        y dir=both,
        y explicit
] coordinates {
    (2.9,0.3) +- (0,0.48)
    (4.9,0.17) +- (0,0.36)
    (6.9,0.37) +- (0,0.48)
    (8.9,0.43) +- (0,0.41)
}; %tsMalik

\addplot+[
    thick,
    mark=*, yellow, mark options={fill=yellow, draw=yellow}, solid,
    error bars/.cd,
        y dir=both,
        y explicit
] coordinates {
    (2.8,0.28) +- (0,0.30)
    (4.8,0.33) +- (0,0.41)
    (6.8,0.12) +- (0,0.24)
    (8.8,0.03) +- (0,0.08)
}; %tsiSCAN

\addplot+[
    thick,
    mark=*, teal, mark options={fill=teal, draw=teal}, solid,
    error bars/.cd,
        y dir=both,
        y explicit
] coordinates {
    (2.6,0.6) +- (0,0.21)
    (4.6,0.33) +- (0,0.35)
    (6.6,0.27) +- (0,0.34)
    (8.6,0.13) +- (0,0.28)
}; %microcause

\addplot+[
    thick,
    mark=*, pink, mark options={fill=pink, draw=pink}, solid,
    error bars/.cd,
        y dir=both,
        y explicit
] coordinates {
    (2.7,0.7) +- (0,0.48)
    (4.7,0.4) +- (0,0.52)
    (6.7,0.6) +- (0,0.52)
    (8.7,0.9) +- (0,0.32)
}; %rcd

\addplot+[
    thick,
    mark=*, gray, mark options={fill=gray, draw=gray}, solid,
    error bars/.cd,
        y dir=both,
        y explicit
] coordinates {
    (3.,0.5) +- (0,0.30)
    (5,0.33) +- (0,0.24)
    (7,0.24) +- (0,0.18)
    (9,0.29) +- (0,0.18)
};%pcunion

\end{axis}
\end{tikzpicture}
% \caption{Setting 1}
\end{subfigure}
\hfill
\begin{subfigure}{0.31\textwidth}
\centering
\begin{tikzpicture}
\begin{axis}[
    width=\textwidth,
    height=4cm,
    xlabel={Number of time-series},
    title={At least 2 parents},
    ymin=0, ymax=1,
legend style={
    at={(1.2,1.3)},
    anchor=south,
    legend columns=3,
    font=\scriptsize
}]
% -----------------------------
%%%% f1 incoming shifted, all parents with minimum 2 parents
% -----------------------------

\addplot+[
    thick,
    mark=*, red, mark options={fill=red, draw=red}, solid,
    error bars/.cd,
        y dir=both,
        y explicit
] coordinates {
    (3,0.72) +- (0,0.42) 
    (5,0.77) +- (0,0.42) 
    (7,0.783) +- (0,0.37)
    (9,0.67) +- (0,0.5)
}; %tsLDiffPC

\addplot+[
    thick,
    mark=*, orange, mark options={fill=orange, draw=orange}, solid,
    error bars/.cd,
        y dir=both,
        y explicit
] coordinates {
    (3.1,0.83) +- (0,0.22) 
    (5.1,0.83) +- (0,0.32) 
    (7.1,0.89) +- (0,0.25)
    (9.1,0.81) +- (0,0.34)
};%tsLDiffPC2

\addplot+[
    thick,
    mark=*, blue, mark options={fill=blue, draw=blue}, solid,
    error bars/.cd,
        y dir=both,
        y explicit
] coordinates {
    (3.2,0.68) +- (0,0.3)
    (5.2,0.87) +- (0,0.32)
    (7.2,0.64) +- (0,0.43)
    (9.2,0.74) +- (0,0.43)
};%tsDCI

\addplot+[
    thick,
    mark=*, cyan, mark options={fill=cyan, draw=cyan}, solid,
    error bars/.cd,
        y dir=both,
        y explicit
] coordinates {
    (3.3,0.68) +- (0,0.3)
    (5.3,0.87) +- (0,0.32)
    (7.3,0.64) +- (0,0.43)
    (9.3,0.74) +- (0,0.43)
};%ok

\addplot+[
    thick,
    mark=*, olive, mark options={fill=olive, draw=olive}, solid,
    error bars/.cd,
        y dir=both,
        y explicit
] coordinates {
    (2.9,0.33) +- (0,0.44)
    (4.9,0.43) +- (0,0.39)
    (6.9,0.43) +- (0,0.47)
    (8.9,0.53) +- (0,0.41)
}; %malik

\addplot+[
    thick,
    mark=*, yellow, mark options={fill=yellow, draw=yellow}, solid, 
    error bars/.cd,
        y dir=both,
        y explicit
] coordinates {
    (2.8,0.54) +- (0,0.33)
    (4.8,0.23) +- (0,0.39)
    (6.8,0.19) +- (0,0.26)
    (8.8,0.03) +- (0,0.1)
};% iscan ok

\addplot+[
    thick,
    mark=*, teal, mark options={fill=teal, draw=teal}, solid,
    error bars/.cd,
        y dir=both,
        y explicit
] coordinates {
    (2.6,0.53) +- (0,0.28)
    (4.6,0.47) +- (0,0.32)
    (6.6,0.20) +- (0,0.32)
    (8.6,0.15) +- (0,0.29)
}; %microcause

\addplot+[
    thick,
    mark=*, pink, mark options={fill=pink, draw=pink}, solid,
    error bars/.cd,
        y dir=both,
        y explicit
] coordinates {
    (2.7,0.4) +- (0,0.52)
    (4.7,0.5) +- (0,0.53)
    (6.7,0.7) +- (0,0.48)
    (8.7,0.7) +- (0,0.48)
}; %rcd

\addplot+[
    thick,
    mark=*, gray, mark options={fill=gray, draw=gray}, solid,
    error bars/.cd,
        y dir=both,
        y explicit
] coordinates {
    (3.,0.43) +- (0,0.33)
    (5,0.43) +- (0,0.34)
    (7,0.31) +- (0,0.04)
    (9,0.23) +- (0,0.14)
};%pcunion

\end{axis}
\end{tikzpicture}

% \caption{Setting 2}
\end{subfigure}

\vspace{0.4cm}

% -----------------------------
% -----------------------------
% NODE RElAXED
% -----------------------------
% -----------------------------

\begin{subfigure}{0.31\textwidth}
\centering
\begin{tikzpicture}
\begin{axis}[
    width=\textwidth,
    height=4cm,
    xlabel={Number of time-series},
    ylabel={F1 ($\mathcal{C}$, $\widehat{\mathcal{C}}\cup\widehat{\mathcal{C}}_P$)},    % title={Setting 3},
    ymin=0, ymax=1
]
% -----------------------------
%%% f1 node relaxed one parent
% -----------------------------
\addplot+[
    thick,
    mark=*, red, mark options={fill=red, draw=red}, solid,
    error bars/.cd,
        y dir=both,
        y explicit
] coordinates {
    (3,0.65) +- (0,0.28)
    (5,0.64) +- (0,0.08)
    (7,0.59) +- (0,0.14)
    (9,0.67) +- (0,0.0)
};%tsLDiffPC

\addplot+[
    thick,
    mark=*, orange, mark options={fill=orange, draw=orange}, solid,
    error bars/.cd,
        y dir=both,
        y explicit
] coordinates {
    (3.1,0.65) +- (0,0.28)
    (5.1,0.74) +- (0,0.18)
    (7.1,0.56) +- (0,0.27)
    (9.1,0.7) +- (0,0.11)
}; %tsLDiffPC2

\addplot+[
    thick,
    mark=*, blue, mark options={fill=blue, draw=blue}, solid,
    error bars/.cd,
        y dir=both,
        y explicit
] coordinates {
    (3.2,0.8) +- (0,0.32)
    (5.2,0.74) +- (0,0.43)
    (7.2,0.69) +- (0,0.44)
    (9.2,0.78) +- (0,0.44)
}; %tsDCI

\addplot+[
    thick,
    mark=*, cyan, mark options={fill=cyan, draw=cyan}, solid,
    error bars/.cd,
        y dir=both,
        y explicit
] coordinates {
    (3.3,0.8) +- (0,0.32)
    (5.3,0.74) +- (0,0.43)
    (7.3,0.69) +- (0,0.44)
    (9.3,0.78) +- (0,0.44)
}; %tsDCIPC

\addplot+[
    thick,
    mark=*, olive, mark options={fill=olive, draw=olive},solid, 
    error bars/.cd,
        y dir=both,
        y explicit
] coordinates {
    (2.9,0.3) +- (0,0.48)
    (4.9,0.07) +- (0,0.21)
    (6.9,0.37) +- (0,0.48)
    (8.9,0.32) +- (0,0.34)
};

\addplot+[
    thick,
    mark=*, yellow, mark options={fill=yellow, draw=yellow}, solid,
    error bars/.cd,
        y dir=both,
        y explicit
] coordinates {
    (2.8,0.28) +- (0,0.30)
    (4.8,0.26) +- (0,0.40)
    (6.8,0.04) +- (0,0.13)
    (8.8,0.0) +- (0,0.0)
};%tsiSCAN

\addplot+[
    thick,
    mark=*, teal, mark options={fill=teal, draw=teal}, solid,
    error bars/.cd,
        y dir=both,
        y explicit
] coordinates {
    (2.6,0.47) +- (0,0.32)
    (4.6,0.2) +- (0,0.32)
    (6.6,0.27) +- (0,0.34)
    (8.6,0.13) +- (0,0.28)
}; %microcause

\addplot+[
    thick,
    mark=*, pink, mark options={fill=pink, draw=pink}, solid,
    error bars/.cd,
        y dir=both,
        y explicit
] coordinates {
    (2.7,0.7) +- (0,0.48)
    (4.7,0.4) +- (0,0.52)
    (6.7,0.4) +- (0,0.52)
    (8.7,0.4) +- (0,0.52)
}; %rcd

\addplot+[
    thick,
    mark=*, gray, mark options={fill=gray, draw=gray}, solid,
    error bars/.cd,
        y dir=both,
        y explicit
] coordinates {
    (3.,0.50) +- (0,0.30)
    (5,0.27) +- (0,0.25)
    (7,0.29) +- (0,0.17)
    (9,0.3) +- (0,0.15)
};%pcunion

\end{axis}
\end{tikzpicture}
% \caption{Setting 3}
\end{subfigure}
\hfill
\begin{subfigure}{0.31\textwidth}
\centering
\begin{tikzpicture}
\begin{axis}[
    width=\textwidth,
    height=4cm,
    xlabel={Number of time-series},
    ymin=0, ymax=1
]
% -----------------------------
%%% f1 node relaxed all parents
% -----------------------------
\addplot+[
    thick,
    mark=*, red, mark options={fill=red, draw=red}, solid,
    error bars/.cd,
        y dir=both,
        y explicit
] coordinates {
    (3,0.65) +- (0,0.28)
    (5,0.67) +- (0,0.14)
    (7,0.73) +- (0,0.23)
    (9,0.85) +- (0,0.18)
}; %tsLDiffPC

\addplot+[
    thick,
    mark=*, orange, mark options={fill=orange, draw=orange}, solid,
    error bars/.cd,
        y dir=both,
        y explicit
] coordinates {
    (3.1,0.65) +- (0,0.28)
    (5.1,0.74) +- (0,0.20)
    (7.1,0.79) +- (0,0.25)
    (9.1,0.87) +- (0,0.17)
}; %tsLDiffPC2

\addplot+[
    thick,
    mark=*, blue, mark options={fill=blue, draw=blue}, solid,
    error bars/.cd,
        y dir=both,
        y explicit
] coordinates {
    (3.2,0.8) +- (0,0.32)
    (5.2,0.81) +- (0,0.35)
    (7.2,0.73) +- (0,0.45)
    (9.2,0.89) +- (0,0.33)
}; %tsDCI

\addplot+[
    thick,
    mark=*, cyan, mark options={fill=cyan, draw=cyan}, solid,
    error bars/.cd,
        y dir=both,
        y explicit
] coordinates {
    (3.3,0.8) +- (0,0.32)
    (5.3,0.81) +- (0,0.35)
    (7.3,0.73) +- (0,0.45)
    (9.3,0.89) +- (0,0.33)
}; %tsDCIPC

\addplot+[
    thick,
    mark=*, olive, mark options={fill=olive, draw=olive}, solid,
    error bars/.cd,
        y dir=both,
        y explicit
] coordinates {
    (2.9,0.3) +- (0,0.48)
    (4.9,0.17) +- (0,0.36)
    (6.9,0.37) +- (0,0.48)
    (8.9,0.43) +- (0,0.41)
}; %tsMalik

\addplot+[
    thick,
    mark=*, yellow, mark options={fill=yellow, draw=yellow}, solid, 
    error bars/.cd,
        y dir=both,
        y explicit
] coordinates {
    (2.8,0.28) +- (0,0.30)
    (4.8,0.33) +- (0,0.41)
    (6.8,0.12) +- (0,0.24)
    (8.8,0.03) +- (0,0.08)
}; %tsiSCAN

\addplot+[
    thick,
    mark=*, teal, mark options={fill=teal, draw=teal}, solid,
    error bars/.cd,
        y dir=both,
        y explicit
] coordinates {
    (2.6,0.6) +- (0,0.21)
    (4.6,0.33) +- (0,0.35)
    (6.6,0.27) +- (0,0.34)
    (8.6,0.13) +- (0,0.28)
}; %microcause

\addplot+[
    thick,
    mark=*, pink, mark options={fill=pink, draw=pink}, solid,
    error bars/.cd,
        y dir=both,
        y explicit
] coordinates {
    (2.7,0.7) +- (0,0.48)
    (4.7,0.4) +- (0,0.52)
    (6.7,0.6) +- (0,0.52)
    (8.7,0.9) +- (0,0.32)
}; %rcd

\addplot+[
    thick,
    mark=*, gray, mark options={fill=gray, draw=gray}, solid,
    error bars/.cd,
        y dir=both,
        y explicit
] coordinates {
    (3.,0.5) +- (0,0.30)
    (5,0.33) +- (0,0.24)
    (7,0.24) +- (0,0.18)
    (9,0.29) +- (0,0.18)
};%pcunion

\end{axis}
\end{tikzpicture}
% \caption{Setting 1}
\end{subfigure}
\hfill
\begin{subfigure}{0.31\textwidth}
\centering
\begin{tikzpicture}
\begin{axis}[
    width=\textwidth,
    height=4cm,
    xlabel={Number of time-series},
    ymin=0, ymax=1,
legend style={
    at={(1.2,1.3)},
    anchor=south,
    legend columns=3,
    font=\scriptsize
}]

% -----------------------------
%%%% node relaxed, all parents, with minimum 2 parents
% -----------------------------
\addplot+[
    thick,
    mark=*, red, mark options={fill=red, draw=red}, solid,
    error bars/.cd,
        y dir=both,
        y explicit
] coordinates {
    (3,0.82) +- (0,0.24)
    (5,0.9) +- (0,0.16)
    (7,0.89) +- (0,0.25)
    (9,0.81) +- (0,0.34)
}; %tslindiffpc

\addplot+[
    thick,
    mark=*, orange, mark options={fill=orange, draw=orange}, solid,
    error bars/.cd,
        y dir=both,
        y explicit
] coordinates {
    (3.1,0.83) +- (0,0.22)
    (5.1,0.83) +- (0,0.32)
    (7.1,0.89) +- (0,0.25)
    (9.1,0.81) +- (0,0.34)
}; %tslindiffpc2

\addplot+[
    thick,
    mark=*, blue, mark options={fill=blue, draw=blue}, solid,
    error bars/.cd,
        y dir=both,
        y explicit
] coordinates {
    (3.2,0.66) +- (0,0.3)
    (5.2,0.87) +- (0,0.32)
    (7.2,0.64) +- (0,0.43)
    (9.2,0.74) +- (0,0.43)
}; %tsdci 

\addplot+[
    thick,
    mark=*, cyan, mark options={fill=cyan, draw=cyan}, solid,
    error bars/.cd,
        y dir=both,
        y explicit
] coordinates {
    (3.3,0.66) +- (0,0.3)
    (5.3,0.87) +- (0,0.32)
    (7.3,0.64) +- (0,0.43)
    (9.3,0.74) +- (0,0.43)
};% tsdcipc

\addplot+[
    thick,
    mark=*, olive, mark options={fill=olive, draw=olive}, solid,
    error bars/.cd,
        y dir=both,
        y explicit
] coordinates {
    (2.9,0.33) +- (0,0.44)
    (4.9,0.43) +- (0,0.39)
    (6.9,0.43) +- (0,0.47)
    (8.9,0.53) +- (0,0.41)
}; %malik ok

\addplot+[
    thick,
    mark=*, yellow, mark options={fill=yellow, draw=yellow}, solid, 
    error bars/.cd,
        y dir=both,
        y explicit
] coordinates {
    (2.8,0.54) +- (0,0.33)
    (4.8,0.23) +- (0,0.39)
    (6.8,0.19) +- (0,0.26)
    (8.8,0.03) +- (0,0.1)
};% iscan ok

\addplot+[
    thick,
    mark=*, teal, mark options={fill=teal, draw=teal}, solid,
    error bars/.cd,
        y dir=both,
        y explicit
] coordinates {
    (2.6,0.53) +- (0,0.28)
    (4.6,0.47) +- (0,0.32)
    (6.6,0.20) +- (0,0.32)
    (8.6,0.15) +- (0,0.29)
}; %microcause

\addplot+[
    thick,
    mark=*, pink, mark options={fill=pink, draw=pink}, solid,
    error bars/.cd,
        y dir=both,
        y explicit
] coordinates {
    (2.7,0.4) +- (0,0.52)
    (4.7,0.5) +- (0,0.53)
    (6.7,0.7) +- (0,0.48)
    (8.7,0.7) +- (0,0.48)
}; %rcd

\addplot+[
    thick,
    mark=*, gray, mark options={fill=gray, draw=gray}, solid,
    error bars/.cd,
        y dir=both,
        y explicit
] coordinates {
    (3.,0.43) +- (0,0.33)
    (5,0.43) +- (0,0.34)
    (7,0.31) +- (0,0.04)
    (9,0.23) +- (0,0.14)
};%pcunion

\end{axis}
\end{tikzpicture}
% \caption{Setting 4}
\end{subfigure}

\caption{Mean F1-score for eight methods as a function of the number of vertices across three settings, for $\alpha = 0.1$.}
\label{fig:res_sensitivity_res_alpha_0_1}
\end{figure}